\documentclass{article}

\usepackage{iclr2027_conference,times}
\iclrfinalcopy
\usepackage{amsmath,amsfonts,bm}
\usepackage{amsthm}
\def\eqref#1{equation~\ref{#1}}
\def\1{\bm{1}}

\def\vb{{\bm{b}}}

\def\vs{{\bm{s}}}

\def\vu{{\bm{u}}}
\def\vv{{\bm{v}}}
\def\vw{{\bm{w}}}

\def\vz{{\bm{z}}}

\def\mJ{{\bm{J}}}

\def\mU{{\bm{U}}}
\def\mV{{\bm{V}}}

\DeclareMathAlphabet{\mathsfit}{\encodingdefault}{\sfdefault}{m}{sl}
\SetMathAlphabet{\mathsfit}{bold}{\encodingdefault}{\sfdefault}{bx}{n}

\newcommand{\R}{\mathbb{R}}

\newtheorem{theorem}{Theorem}
\usepackage[utf8]{inputenc}
\usepackage[T1]{fontenc}
\usepackage{hyperref}
\usepackage{url}
\usepackage{booktabs}
\usepackage{amsfonts}
\usepackage{nicefrac}
\usepackage{caption}
\usepackage{microtype}
\usepackage{xcolor}
\definecolor{class1}{HTML}{404040} 
\definecolor{class2}{HTML}{0072B2} 
\definecolor{class3}{HTML}{D55E00} 
\definecolor{class4}{HTML}{009E73}
\definecolor{class5}{HTML}{CC79A7}
\definecolor{class6}{HTML}{F0E442} 
\usepackage{algorithm}
\usepackage{algpseudocode}
\usepackage{graphicx}
\usepackage{subcaption}
\usepackage{cleveref}
\usepackage{pgfplots}
\pgfplotsset{compat=1.18}
\usepgfplotslibrary{groupplots}

\title{Improving Generative Model Self-training with Geometrically Modified Outputs}

\author{Patrick Batsell, Thomas Walker, \& Richard Baraniuk \\
Department of Electrical and Computer Engineering\\
Rice University\\
Houston, TX 77005, USA \\
\texttt{\{pb52, tw78, richb\}@rice.edu} \\}
\begin{document}

\maketitle

\begin{abstract}
  Self-training generative models -- the continued improvement of a model using its own outputs -- is becoming increasingly important as high-quality training data becomes scarce. However, naïvely finetuning on model-generated samples leads to degradation through model collapse and the model autophagy disorder. Negative-guidance self-training methods turn this degradation into a useful signal, using a model finetuned on its own outputs to guide the original model toward improved generation. Existing methods, however, take the negative signal in standard model outputs as given. We instead ask whether this signal can be explicitly strengthened. We introduce Geometrically Modified Outputs (GMOs), which reweight the singular values of the generator’s input-output Jacobian to increase the influence of its leading singular directions. This geometric modification amplifies the mode-seeking behavior and distortions of standard outputs, providing a stronger and more targeted negative signal for self-training. Across a range of one-step generative models, GMOs consistently improve the performance of negative-guidance methods, including Neon and SIMS, compared with using standard model outputs.
\end{abstract}

\begin{figure}[ht]
    \centering

    \begin{subfigure}[b]{0.32\textwidth}
        \centering
        \includegraphics[width=\textwidth]{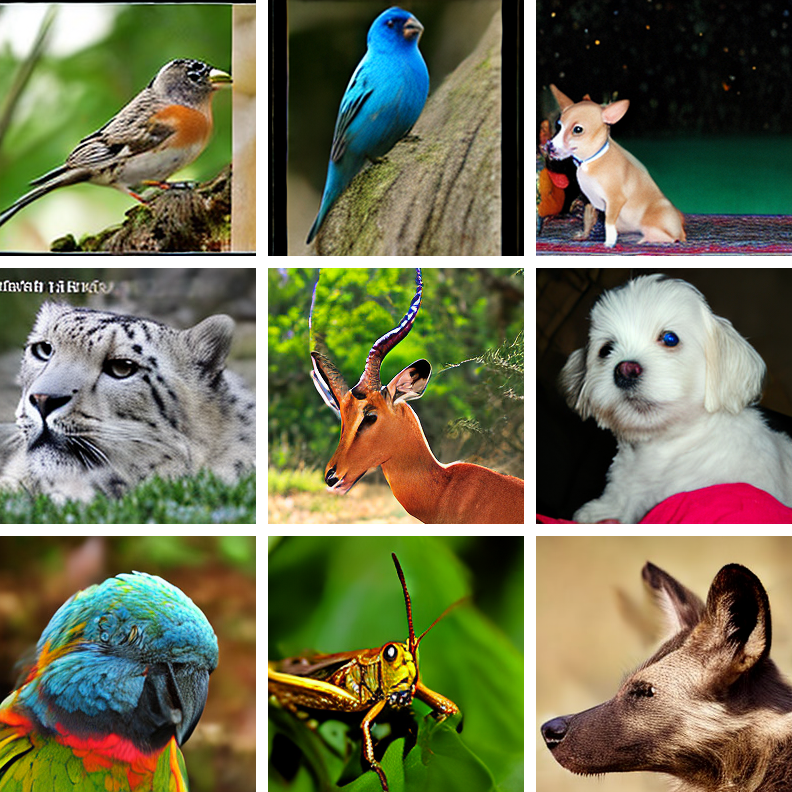}
        \caption*{Standard Outputs}
    \end{subfigure}
    \hfill
    \begin{subfigure}[b]{0.32\textwidth}
        \centering
        \includegraphics[width=\textwidth]{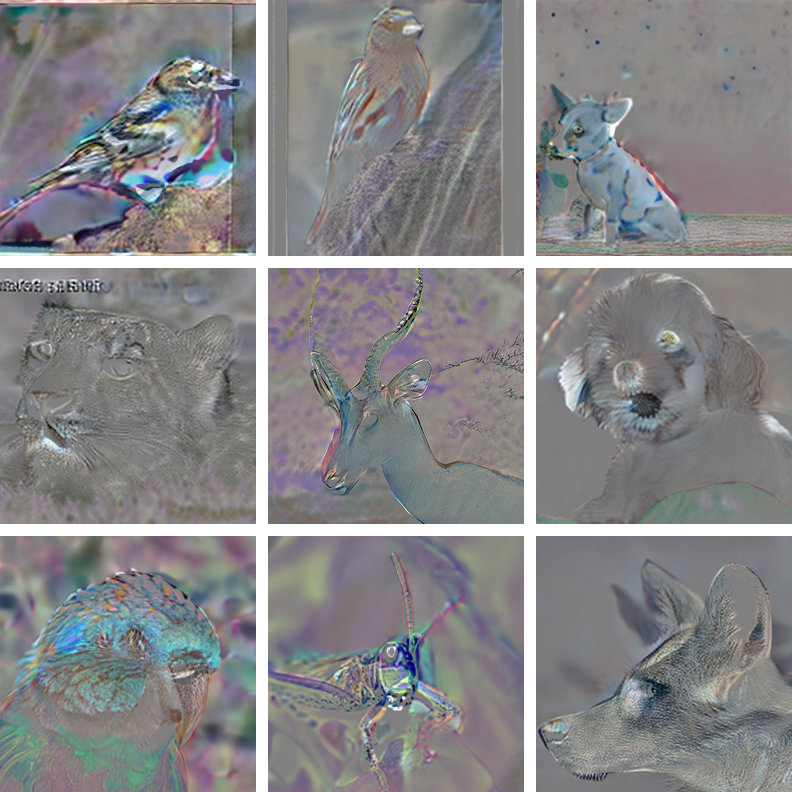}
        \caption*{Perturbation}
    \end{subfigure}
    \hfill
    \begin{subfigure}[b]{0.32\textwidth}
        \centering
        \includegraphics[width=\textwidth]{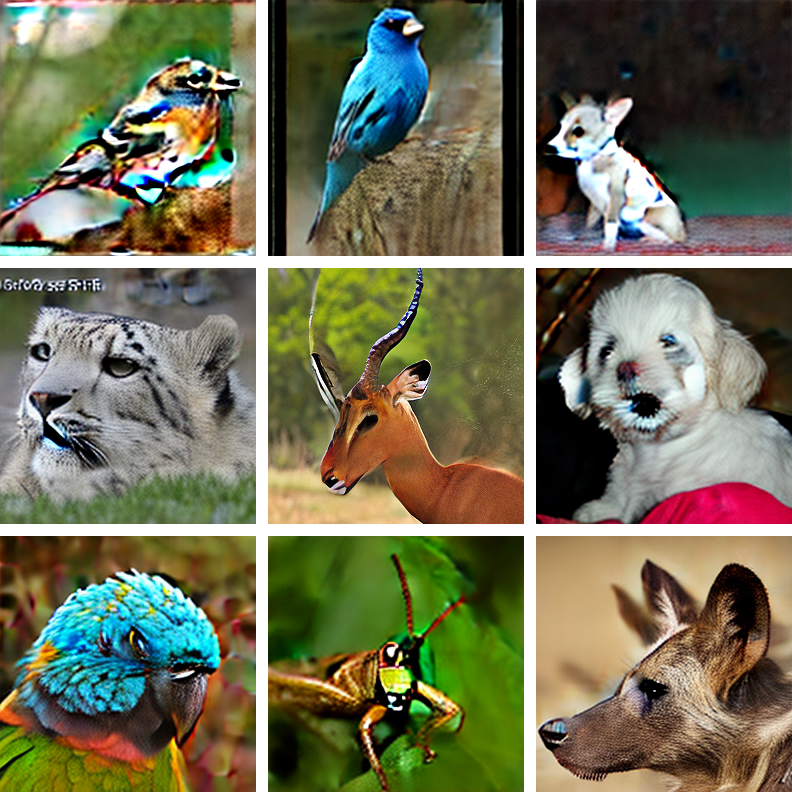}
        \caption*{Geometrically Modified Outputs}
    \end{subfigure}

    \caption{
    \textbf{Geometrically Modified Outputs (GMOs) improve the performance of one-step generative models through self-training.}
    Here we present examples of GMOs for the IMM~\citep{zhouInductiveMomentMatching2025} one-step generative model trained on Imagenet256~\citep{krizhevskyImageNetClassificationDeep2012}.
    In the first panel, we show the model's standard outputs; in the center panel, we show the perturbations of these outputs that yield the corresponding GMOs in the right panel (see \Cref{fig:model_outputs_more} for more examples). 
    Using GMOs in self-training algorithms more effectively improves generative models.
    }
    \label{fig:model_outputs}
\end{figure}

\section{Introduction}

The performance of generative models for image generation has benefited greatly from increasing amounts of compute and high-quality training data~\citep{kaplanScalingLawsNeural2020,henighan2020scalinglawsautoregressivegenerative}.
Although the path for increasing compute is relatively clear~\citep{TrendsArtificialIntelligence}, significantly expanding the quantity of high-quality training data represents a fundamental challenge~\citep{muennighoff2025scalingdataconstrainedlanguagemodels,villalobosWillWeRun2024}.
Consequently, the field of self-training--the continued development of a model using only its own outputs--is gaining significant interest~\citep{kimRefiningGenerativeProcess2023,alemohammadSelfimprovingDiffusionModels2024,yuanSelfplayFinetuningDiffusion2024,alemohammadNeonNegativeExtrapolation2026,zhengDirectDiscriminativeOptimization2025,fengModelCollapseScaling2025,karrasGuidingDiffusionModel2024}.

Although the utilization of synthetic data has been established in adjacent fields, such as for data augmentation when training classifiers~\citep{wangBetterDiffusionModels2023,heSyntheticDataGenerative2023} and for bootstrapping reinforcement learning models through self-play~\citep{silverGeneralReinforcementLearning2018}, the situation is more nuanced with generative models.
Na\"{i}ve self-training--finetuning a generative model on its own outputs--is detrimental as it leads to the ``model autophagy disorder'' (MAD)~\citep{alemohammadSelfconsumingGenerativeModels2024} and ``model collapse''~\citep{shumailovAIModelsCollapse2024}.
These phenomena describe the degradation of model quality due to the appearance of unwanted artifacts in its outputs and a reduction in the diversity of those outputs.
Common strategies implemented by self-training algorithms include, self-play~\citep{yuanSelfplayFinetuningDiffusion2024}, direct discriminative optimization guidance~\citep{zhengDirectDiscriminativeOptimization2025}, and output verification~\citep{fengModelCollapseScaling2025}, with the strongest results observed using negative guidance~\citep{alemohammadSelfimprovingDiffusionModels2024,alemohammadNeonNegativeExtrapolation2026,karrasGuidingDiffusionModel2024}.

In this paper, we ask: {\em Can the effectiveness of self-training algorithms using negative guidance be improved by augmenting model outputs to amplify the negative signal?}
% In particular, we consider the Neon self-training algorithm~\citep{alemohammadNeonNegativeExtrapolation2026} which finetunes generative models on their outputs to obtain a negative signal to improve performance.
% Unlike other self-training methods, Neon is simple in design and generally applicable across different architectures.

We answer this question affirmatively with Geometrically Modified Outputs (GMOs), by demonstrating that they can be used to improve the effectiveness of the SIMS~\citep{alemohammadSelfimprovingDiffusionModels2024} and Neon self-training algorithm~\citep{alemohammadNeonNegativeExtrapolation2026}.
% Neon finetunes generative models on their outputs to obtain a negative signal to improve performance, and has been shown to be an effective and generally applicable self-training algorithm~\citep{alemohammadNeonNegativeExtrapolation2026}.

GMOs leverage the observation of \citet{batsellGeometricPerspectiveRecursive2026}, and verified in \Cref{fig:naive_self_training}, that the effective ranks of the input-output Jacobians of generative models collapse during na\"{i}ve self-training.
Examples of how GMOs compare with standard model outputs are shown in \Cref{fig:model_outputs,fig:model_outputs_more}.
GMOs represent a geometric perspective on the self-training problem, which has not been explicitly leveraged in prior works~\citep{kimRefiningGenerativeProcess2023,alemohammadSelfimprovingDiffusionModels2024,yuanSelfplayFinetuningDiffusion2024,alemohammadNeonNegativeExtrapolation2026,zhengDirectDiscriminativeOptimization2025,fengModelCollapseScaling2025,karrasGuidingDiffusionModel2024}.

% Our contributions are the following:

% \textbf{[C1.]} We propose Geometrically Modified Outputs (GMOs), a structural augmentation of model outputs designed to amplify the negative signal derived by utilizing the generator's geometry. 
% With \Cref{alg:gmos}, we provide an efficient, scalable algorithm for computing GMOs using only Jacobian-vector products (see \Cref{sec:computational_complexity}).

% \textbf{[C2.]} \fix{We demonstrate that GMOs yield net improvements in FID while improving image quality and diversity simultaneously (see \Cref{fig:gmo_alpha_sweep} and \Cref{tab:prdc}).}
% This provides a fundamentally stronger negative signal than Gaussian noise\fix{: because the degradation induced by GMOs is more directed, optimal checkpoints are reached with less finetuning} (see Figures \ref{fig:b2_imagenet256} and \ref{fig:compute_weight_merging}).

% \textbf{[C3.]} We show that integrating GMOs into the Neon self-training algorithm consistently improves FID scores across highly optimized one-step architectures, including MeanFlow, AlphaFlow, and IMM (see Table \ref{tab:sota}). 
% Furthermore, we demonstrate that GMOs are transferable, successfully improving larger parameter models and multi-step inference regimes without requiring recomputation (see \Cref{fig:transfer}).

\section{The Geometry and Self-training of Generative Models}

In this section, we introduce relevant notation for the generative models we consider (\Cref{sec:generative_models}), we describe what we mean by the {\em geometry} of generative models (\Cref{sec:geometry}), and we introduce the {\em self-training} problem (\Cref{sec:self_training}).

\subsection{Generative Models}\label{sec:generative_models}

We consider generative models $G_{\theta}$ with parameters $\theta$ that are resultant of a training algorithm $\mathcal{A}$ being applied to a set of training data $\mathcal{D}$ drawn from a distribution $p_{\text{data}}$.
In this context, $\mathcal{D}$ represents high-quality real data.
Generative model architecture can differ in the inference routines they use for generating outputs.
Generally, we let $\mathcal{I}$ represent this inference routine and let $q_{\theta,\kappa}$ denote the induced sampling distribution when $\mathcal{I}$ is implemented with hyperparameter $\kappa$.

In this paper, we are concerned with one-step inference routines~\citep{song2023consistencymodels,fransOneStepDiffusion2025,gengMeanFlowsOnestep2025,zhouInductiveMomentMatching2025,dengGenerativeModelingDrifting2026,yue2026image,yang2026representation}.
More specifically, $\mathcal{I}$ generates an output $\vs\in\mathbb{R}^d$ based on a latent vector $\vz\in\mathbb{R}^h$, which we summarize as a map $g_{\theta}:\mathbb{R}^h\to\mathbb{R}^d$.

\subsection{The Geometry of Generative Models}\label{sec:geometry}

A standard output of a one-step generative model is generated by sampling a latent vector $\vz\in\mathbb{R}^h$ and applying the map $g_{\theta}(\vz)=\vs\in\mathbb{R}^d$.
We can decompose this output as $\vs=\mJ_{\vz}\vz+\vb_{\vz}$ where $\mJ_{\vz}\in\mathbb{R}^{d\times h}$ is the input-output Jacobian of $g_{\theta}$ at $\vz$, and $\vb_{\vz}:=\vs-\mJ_{\vz}\vz$ is the offset of $g_{\theta}$ at $\vz$.
We take the {\em local geometry} of $g_{\theta}$ at $\vz$ to be the spectral properties of $\mJ_{\vz}$.
Let $\vu^{(k)}_{\vz}\in\mathbb{R}^d$ and $\vv^{(k)}_{\vz}\in\mathbb{R}^h$ denote the $k^{\text{th}}$ left and right singular vectors of $\mJ_{\vz}$, respectively, with $\sigma^{(k)}_{\vz}\in\mathbb{R}$ being the corresponding singular value.
Let $\mU_{\vz}\in\mathbb{R}^{d\times r}$ and $\mV_{\vz}\in\mathbb{R}^{h\times r}$ be the corresponding matrices of left and right singular vectors, such that $\mJ_{\vz}=\mU_{\vz}\mathrm{diag}\left(\sigma_{\vz}^{(1)},\dots,\sigma_{\vz}^{(r)}\right)\mV_{\vz}^{\top}$.

\subsection{Self-training Generative Models}\label{sec:self_training}

Self-training is the problem of constructing a learning algorithm that uses only the model's outputs to yield a generative model $G_{\theta^\prime}$ that outperforms $G_{\theta}$.
More formally, self-training involves applying a learning algorithm $\tilde{\mathcal{A}}$ on a set of data $\mathcal{S}$ sampled from $q_{\theta,\kappa}$ to identify a set of parameters $\theta^{\prime}$ such that $G_{\theta^\prime}$ performs better than $G_{\theta}$.
Typically, self-training algorithms are constrained by a compute budget $\mathcal{B}$, which we measure as the cumulative number of images observed when using $\tilde{\mathcal{A}}$.

Many approaches exist for tackling the self-training problem.
The na\"{i}ve approach is to continue the learning algorithm $\mathcal{A}$ on $\mathcal{S}$; however, several studies observe that this leads to the degradation of the generative model~\citep{shumailovAIModelsCollapse2024,alemohammadSelfconsumingGenerativeModels2024}.
Principled methods for overcoming this, include self-play~\citep{yuanSelfplayFinetuningDiffusion2024}, direct discriminative optimization guidance~\citep{zhengDirectDiscriminativeOptimization2025}, negative guidance~\citep{alemohammadSelfimprovingDiffusionModels2024,alemohammadNeonNegativeExtrapolation2026,karrasGuidingDiffusionModel2024}, and output verification~\citep{fengModelCollapseScaling2025}.

% {\color{blue}
% The degradation that self-training induces has a geometric sign: finetuning a generator on its own outputs is mode-seeking.
% The model sharpens toward its dominant modes, and the spectrum of the generator's input-output Jacobian concentrates accordingly (see \Cref{sec:mode_seeking}). Methods that use this mode-seeking degradation as a negative signal take the signal as given, using the model's raw outputs; in this work we instead construct and amplify the negative signal itself.
% Each modified output is the sample the model would produce after a controlled amount of spectrum collapse, so a model finetuned on these outputs has moved further along the mode-seeking direction, and extrapolating away from it counteracts the collapse.
% The negative signal is therefore reliable whenever self-training degrades the model through mode-seeking, which is the regime this literature addresses and what we observe on all architectures tested.
% }

In this paper, we consider whether {\em the effectiveness of self-training algorithms using negative guidance can be improved by augmenting model outputs to amplify the negative signal.}

In particular, we focus on the Neon self-training algorithm~\citep{alemohammadNeonNegativeExtrapolation2026}, but we also consider SIMS~\citep{alemohammadSelfimprovingDiffusionModels2024}.
We focus on Neon as it has been demonstrated to efficiently and effectively improve the performance of state-of-the-art generative models more effectively than other self-training algorithms utilizing negative guidance~\citep{alemohammadNeonNegativeExtrapolation2026}.
Neon works by finetuning a base model $G_{\theta}$ on $\mathcal{S}$ to get $G_{\tilde{\theta}}$.
It then uses the direction between $\theta$ and $\tilde{\theta}$ as a negative signal to generate parameters $\theta^\prime=(1+w)\theta-w\tilde{\theta}$ for some $w\in\mathbb{R}^+$.
This process is summarized in \Cref{alg:neon}.

\begin{algorithm}[ht]
\caption{Neon}\label{alg:neon}
\begin{algorithmic}[1]
\Require Base model $G_{\theta}$, Inference routine $\mathcal{I}$ with hyperparameters $\kappa$, Weight-merging parameter $w$, Training budget $\mathcal{B}$
\State Sample $\mathcal{S}$ from $q_{\theta,\kappa}$.
\State $G_{\tilde{\theta}}\leftarrow\mathrm{FineTune}\left(G_{\theta},\mathcal{S},\mathcal{B}\right)$
\State $\theta^\prime\leftarrow(1+w)\theta-w\tilde{\theta}$
\end{algorithmic}
\end{algorithm}

Neon's guarantee that this extrapolation improves the model relies on $\mathcal{S}$ being \emph{mode-seeking}: concentrated toward the high-density regions of the model's own distribution~\citep{alemohammadNeonNegativeExtrapolation2026}.

\section{Geometrically Modified Outputs}\label{sec:gmos}

We propose Geometrically Modified Outputs (GMOs) as a strategy for amplifying the mode-seeking nature of model outputs with the intention of improving negative guidance self-training algorithms.
GMOs do this by amplifying the influence of the top singular vectors of the model's generator map.
This is motivated by the observations of \citet{batsellGeometricPerspectiveRecursive2026} which show that under na\"{i}ve self-training, the input-output Jacobians of generators become increasingly collapsed.
More specifically, their effective ranks\footnote{The effective rank of a matrix is defined as the exponential of the entropy of its singular values.} collapse and low-quality artifacts appear in the more prominent singular vectors.
In \Cref{fig:naive_self_training}, we verify the collapse of the effective rank for a MeanFlow generative model~\citep{gengMeanFlowsOnestep2025} na\"{i}vely self-training on ImageNet256~\citep{krizhevskyImageNetClassificationDeep2012}.
Moreover, we verify that GMOs exhibit amplified mode-seeking in \Cref{sec:mode_seeking}.

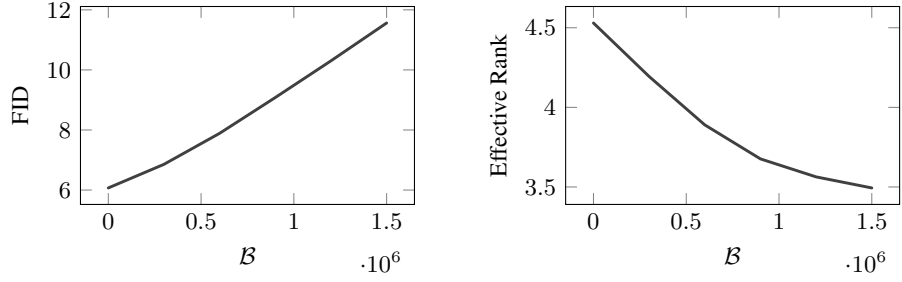
\begin{figure}[ht]
    \centering
    \centering
    \begin{tikzpicture}
    \begin{groupplot}[
        group style={
            group size=2 by 1,
            horizontal sep=2.0cm,
        },
        width=6cm,
        height=4.2cm,
        xlabel={$\mathcal{B}$},
        tick label style={font=\footnotesize},
        label style={font=\small},
        title style={font=\small},
        every axis plot/.append style={
            line width=1.1pt,
            mark=none,
            line cap=round,
            line join=round
        }
    ]

    \nextgroupplot[
        ylabel={FID}
    ]
    \addplot[color=class1] table [x=compute, y=0, col sep=comma] {data/naive_self_training_fid.csv};
    
    \nextgroupplot[
        ylabel={Effective Rank},
        legend style={
            at={(0.95,0.95)}, 
            anchor=north east, 
            font=\scriptsize,
            fill=white, 
            fill opacity=0.8, 
            draw opacity=1, 
            text opacity=1
        }
    ]
    \addplot[color=class1] table [x=compute, y=0, col sep=comma] {data/naive_self_training_rank.csv};

    \end{groupplot}
    \end{tikzpicture}
    \caption{
    \textbf{The geometry of a generative model collapses under na\"{i}ve self-training.}
    Here we finetune a MeanFlow SiT-B/2 model pre-trained on ImageNet256 for 50 epochs on $30,000$ of its own outputs.
    Throughout training, we monitor the FID~\citep{heuselGANsTrainedTwo2017} of the model (left), and its effective rank on a collection of 16 fixed latent vectors (right).
    }
    \label{fig:naive_self_training}
\end{figure}

\subsection{GMO Generation}

Consider a standard output $\vs=\mJ_{\vz}\vz+\vb_{\vz}\in\mathbb{R}^d$ generated from a latent vector $\vz\in\mathbb{R}^h$.
The corresponding $\alpha$-GMO is given by $\tilde{\vs}=\tilde{\mJ}_{\vz}\vz+\vb_{\vz}$, where $\tilde{\mJ}_{\vz}=\mU\mathrm{diag}\left(\tilde{\sigma}_{\vz}^{(1)},\dots,\tilde{\sigma}_{\vz}^{(r)}\right)\mV_{\vz}^{\top}$ with
\begin{equation*}
    \tilde{\sigma}_{\vz}^{(k)}=\begin{cases}\sqrt{(1-\alpha)\left(\sigma_{\vz}^{(1)}\right)^2+\alpha\left\Vert\mJ_{\vz}\right\Vert_F^2}&k=1\\\sqrt{1-\alpha}\sigma_{\vz}^{(k)}&k\geq2.\end{cases}
\end{equation*}
This reweighting preserves total spectral energy while transferring
an $\alpha$ fraction of the trailing energy to the leading component (Appendix~\ref{app:spectral_energy}).

Intuitively, the $\alpha$-GMO increases the influence that the top singular vectors have on the generated output.
With $\alpha$ equal to zero, the generated output remains unchanged, but with $\alpha$ equal to one, the output is entirely determined by the top singular vectors.
From the specific perspective of improving Neon, GMOs can be thought of as amplifying the ``mode-seeking'' nature of model outputs.
We explore this in \Cref{sec:mode_seeking}.

With \Cref{alg:gmos}, we provide an exact implementation of $\alpha$-GMOs that requires only the computation of the top spectral statistics of $\mJ_{\vz}$.

\begin{theorem}\label{thm:gmo_algorithm}
    For a generator $g:\mathbb{R}^h\to\mathbb{R}^d$, a vector $\vz\in\mathbb{R}^h$ and $\alpha\in[0,1]$, \Cref{alg:gmos} returns the $\alpha$-GMO of $\vs=g(\vz)$.
\end{theorem}

\textit{Proof.} See \Cref{sec:proof}. \qed

\begin{algorithm}[ht]
\caption{Geometrically Modifying Outputs}\label{alg:gmos}
\begin{algorithmic}[1]
\Require Generator $g:\R^h\to\R^d$, $\alpha\in[0,1]$.
\State Sample $\vz\in\R^h$
\State $\vs\leftarrow g(\vz)$
\State $\vb_{\vz}\leftarrow\vs-\mJ_{\vz}\vz$
\State $(\vu,\vv,\sigma)\leftarrow\mathrm{TopSigVec}(g,\vz)$
\State $E\leftarrow\left\Vert\mJ_{\vz}\right\Vert_F^2$
\State $\tilde{\sigma}\leftarrow\sqrt{(1-\alpha)\sigma^2+\alpha E}$
\State $\tilde{\vs}\leftarrow\sqrt{1-\alpha}\left(\vs-\vb_{\vz}\right)+\left(\tilde{\sigma}-\sqrt{1-\alpha}\sigma\right)\left(\vv^\top\vz\right)\vu+\vb_{\vz}$
\end{algorithmic}
\end{algorithm}

\subsection{Computational Complexity of Generating GMOs}\label{sec:computational_complexity}

Computing exact GMOs involves manipulations involving the input-output Jacobian $\mJ_{\vz}\in\mathbb{R}^{d\times h}$.
For high-dimensional generative models, where both the latent dimension $h$ and data dimension $d$ are large, this poses a severe computational bottleneck.

Here, we provide a breakdown of the computational complexity of exactly implementing \Cref{alg:gmos} and compare it with an approximate implementation using power iteration and Hutchinson's estimator.
A point to note is that, this analysis is only applicable once for a given model, as given a latent vector $\vz$, lines $6$ and $7$ of \Cref{alg:gmos} can be applied to multiple different $\alpha$ values to generate GMOs for different $\alpha$ values simultaneously. 

Let $\mathcal{O}(C_f)$ denote the time complexity of a single forward pass of the generator $g(\vz)$.

\paragraph{Exact Implementation.}

To compute GMOs exactly, one must instantiate the full Jacobian matrix $\mJ_{\vz}$. 
Using automatic differentiation, this requires either $h$ forward-mode passes (Jacobian-Vector Products, JVPs) or $d$ reverse-mode passes (Vector-Jacobian Products, VJPs). 
This scales linearly with the dimensions, resulting in a time complexity of $\mathcal{O}(\min(d, h) \cdot C_f)$. 
The memory complexity to store the dense matrix is $\mathcal{O}(d \cdot h)$.
Once $\mJ_{\vz}$ is materialized, computing the exact singular value decomposition to find $\vu$, $\vv$, and $\sigma$ requires $\mathcal{O}\left(\min\left(d^2h, dh^2\right)\right)$ operations.
Computing the exact Frobenius norm requires an additional $\mathcal{O}(d\cdot h)$ operations.

In total, the exact time complexity is on the order of $\mathcal{O}\left(\min(d,h)\cdot C_f+\min\left(d^2h, dh^2\right)\right)$, with spatial memory scaling at $\mathcal{O}(d\cdot h)$.

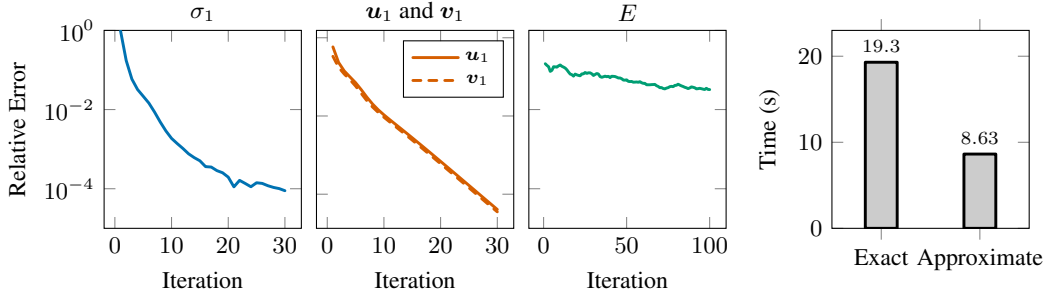
\begin{figure}[t]
    \centering
    \begin{tikzpicture}
        \begin{groupplot}[
            group style={
                group size=4 by 1,
                horizontal sep=0.2cm
            },
            width=0.3\textwidth,
            height=4.2cm,
            xlabel={Iteration},
            tick label style={font=\footnotesize},
            label style={font=\small},
            title style={font=\small, yshift=-6pt},
            every axis plot/.append style={
                line width=1.1pt,
                mark=none,
                line cap=round,
                line join=round
            }
        ]
        \nextgroupplot[
            ymax=1.0,
            ymin=0.00001,
            ymode=log,
            ylabel={Relative Error},
            title={$\sigma_1$}
        ]
        \addplot[class2] table[x=step, y=avg_rel_err_sigma1, col sep=comma] {data/avg_errors.csv};
        \nextgroupplot[
            ymode=log,
            yticklabels={},
            title={$\vu_1$ and $\vv_1$},
            legend style={
                at={(0.95,0.95)}, 
                anchor=north east, 
                font=\scriptsize,
                fill=white, 
                fill opacity=0.8,
                draw opacity=1, 
                text opacity=1
            }
        ]
        \addplot[class3] table[x=step, y expr={1-\thisrow{avg_cos_sim_left}}, col sep=comma] {data/avg_errors.csv};
        \addlegendentry{$\vu_1$}
        \addplot[class3, dashed] table[x=step, y expr={1-\thisrow{avg_cos_sim_right}}, col sep=comma] {data/avg_errors.csv};
        \addlegendentry{$\vv_1$}
        \nextgroupplot[
            ymax=1.0,
            ymin=0.00001,
            ymode=log,
            yticklabels={},
            title={$E$}
        ]
        \addplot[class4] table[x=step, y=avg_rel_err_frob, col sep=comma] {data/avg_errors.csv};
        
        \nextgroupplot[
            xshift=1.2cm,
            ybar,
            bar width=12pt,
            ymin=0,
            ymax=23,
            xlabel={},
            ylabel={Time (s)},
            enlarge x limits=0.5,
            xtick={1, 2},
            xticklabels={Exact, Approximate},
            nodes near coords,
            every node near coord/.append style={font=\scriptsize},
        ]
        \addplot[fill=gray!40, draw=black] coordinates {(1, 19.3) (2, 8.63)};
        
        \end{groupplot}
    \end{tikzpicture}
    \caption{
    Here we consider computing spectral statistics from the input-output Jacobians of a MeanFlow one-step generator~\citep{gengMeanFlowsOnestep2025} with a UNet architecture~\citep{ronnebergerUNetConvolutionalNetworks2015,song2019generative} trained on CIFAR10~\citep{krizhevskyLearningMultipleLayers2009}.
    Ten latent vectors are sampled, and $E$, $\sigma_1$, $\vu_1$, and $\vv_1$ are either computed exactly or using approximations.
    Power iteration is implemented $30$ times to obtain estimates for $\hat{\sigma}_1$, $\hat{\vu}_1$, and $\hat{\vv}_1$. Similarly, $100$ samples are used to generate the Hutchinson estimate $\hat{E}$.
    In the first panel, we record $\nicefrac{\left\vert\hat{\sigma}_1-\sigma_1\right\vert}{\sigma_1}$.
    In the second panel, we record $1-\nicefrac{\left\langle\hat{\vu}_1,\vu_1\right\rangle}{\left\Vert\hat{\vu}_1\right\Vert_2\left\Vert\vu_1\right\Vert_2}$, and similarly for $\vv_1$.
    In the third panel, we record $\nicefrac{\left\vert\hat{E}-E\right\vert}{E}$.
    In the fourth panel, we record the time taken to exactly compute the Jacobian and its singular value decomposition (Exact) and the time taken to perform $20$ power iterations and $20$ Hutchinson approximations (Approximate).}
    \label{fig:computational_approximation}
\end{figure}

\paragraph{Approximate Implementation.}

To improve the tractability of computing GMOs, we can sidestep the materialization of $\mJ_{\vz}$ entirely.
Modern deep learning frameworks support fast JVPs and VJPs in roughly $\mathcal{O}(C_f)$ time, with spatial memory comparable to that of a standard forward or backward pass.

For computing approximations for $\vu$, $\vv$ and $\sigma$, we can use power iteration. 
Power iteration works by repeatedly applying the matrix $\mJ_{\vz}^\top \mJ_{\vz}$ to a random initialization vector. 
In practice, evaluating $\mJ_{\vz}^\top (\mJ_{\vz} \vw)$ amounts to one JVP followed by one VJP. 
For $k$ iterations, the time complexity is $\mathcal{O}\left(k\cdot C_f\right)$.

For computing an approximate value of the squared Frobenius norm, we can use Hutchinson's estimator.
The squared Frobenius norm equals the trace of the Gram matrix $\left\Vert\mJ_{\vz}\right\Vert_F^2 = \mathrm{tr}(\mJ_{\vz}^\top \mJ_{\vz})$. 
Hutchinson's trace estimator approximates this via $\frac{1}{m} \sum_{i=1}^m\left\Vert\mJ_{\vz}\vw_i\right\Vert_2^2$, where $\vw_i$ are standard normal vectors. 
Calculating this requires $m$ independent JVPs, yielding a time complexity of $\mathcal{O}\left(m\cdot C_f\right)$.

Consequently, the total time complexity for computing GMOs reduces to $\mathcal{O}\left((k+m)\cdot C_f\right)$. 
Similarly, the memory complexity drops to $\mathcal{O}(d+h)$, since we only have to store the vectors involved in the JVP or VJP computations.

The computational efficiency of the approximations is highlighted in the fourth panel of \Cref{fig:computational_approximation}, which shows that the approximate implementation of \Cref{alg:gmos} is around $2$ times faster than its exact implementation.
In particular, it makes the computation of GMOs tractable in settings where memory constraints would otherwise prevent it.

Importantly, with the three panels of \Cref{fig:computational_approximation}, we ensure that we do not lose much by utilizing this approximate implementation. The method is therefore insensitive to the power-iteration count above roughly ten iterations, and we use $20$.
In \Cref{fig:approximate_singular_vectors}, we qualitatively demonstrate that these approximations are sound by visually comparing the exact and approximate left singular vectors.

\section{Experiments}\label{sec:experiments}

In this section, we examine the characteristics of GMOs in comparison to standard outputs and how Neon with GMOs differs from Neon using standard outputs.
% We demonstrate that GMOs are more effective than perturbing with Gaussian noise (\Cref{sec:ablation}), improve the FID, precision and recall of generative models (\Cref{sec:fid_precision_recall}), improve the computational efficiency of Neon (\Cref{sec:effect_on_compute}), transfer between models (\Cref{sec:transfer}), and improve one-step generative models more effectively than standard outputs (\Cref{sec:improve_sota}).
% Only in \Cref{sec:effect_on_compute} is the exact implementation of \Cref{alg:gmos} being utilized.

\subsection{GMOs Outperform Standard Outputs and Gaussian Perturbed Outputs}\label{sec:ablation}

Here, we empirically verify that applying the Neon self-training to GMOs is more effective than applying it to standard outputs. 
In particular, we isolate this benefit to the geometric structure of GMOs rather than to the mere effect of injecting noise.
We do so by applying Neon to standard outputs perturbed with Gaussian noise of an equivalent magnitude.

To construct the outputs for the equivalent Gaussian baseline, we match the GMO perturbation amplitudes on a per-sample basis.
For a given standard output and its corresponding GMO, we compute the root mean square ($\gamma$) of their difference to measure the exact perturbation amplitude. 
We then draw isotropic Gaussian noise scaled by $\gamma$ and add it to the standard output.
This ensures that the Gaussian-perturbed outputs have the exact same total perturbation energy as the GMOs, but lack the directional guidance derived from the input-output Jacobian.

\begin{figure}[t]
    \centering
    \begin{subfigure}[b]{0.58\textwidth}
        \centering
        \resizebox{\linewidth}{!}{%
        \begin{tikzpicture}
        \begin{groupplot}[
            group style={
                group size=2 by 1,
                horizontal sep=1.2cm,
            },
            width=5cm,
            height=4.8cm,
            xlabel={$\mathcal{B}$},
            tick label style={font=\footnotesize},
            label style={font=\small},
            title style={font=\small},
            every axis plot/.append style={
                line width=1.3pt,
                mark=none,
                line cap=round,
                line join=round
            },
        ]

        \nextgroupplot[
            ylabel={FID}
        ]
        \addplot[color=class1] table [x=compute, y=0, col sep=comma] {data/b2_min_fid.csv};
        \addplot[color=class2] table [x=compute, y=0.05, col sep=comma] {data/b2_min_fid.csv};
        \addplot[color=class3] table [x=compute, y=0.2, col sep=comma] {data/b2_min_fid.csv};
        \addplot[color=class4, line width=1pt, dashed] table [x=compute, y=0.05, col sep=comma] {data/b2-gaussian_min_fid.csv};

        \nextgroupplot[
            ylabel={$w$},
            legend style={
                at={(0.95,0.95)}, 
                anchor=north east, 
                font=\scriptsize,
                fill=white, 
                fill opacity=0.8, 
                draw opacity=1, 
                text opacity=1
            }
        ]
        \addplot[color=class1] table [x=compute, y=0.0, col sep=comma] {data/b2_w_at_min_fid.csv};
        \addlegendentry{$\alpha=0.0$}
        \addplot[color=class2] table [x=compute, y=0.05, col sep=comma] {data/b2_w_at_min_fid.csv};
        \addlegendentry{$\alpha=0.05$}
        \addplot[color=class3] table [x=compute, y=0.2, col sep=comma] {data/b2_w_at_min_fid.csv};
        \addlegendentry{$\alpha=0.2$}
        \addplot[color=class4] table [x=compute, y=0.05, col sep=comma, dashed] {data/b2-gaussian_w_at_min_fid.csv};
        \addlegendentry{$\gamma=0.05$}

        \end{groupplot}
        \end{tikzpicture}}
    \end{subfigure}
    \hfill
    \begin{subfigure}[b]{0.09\textwidth}
        \includegraphics[width=\textwidth]{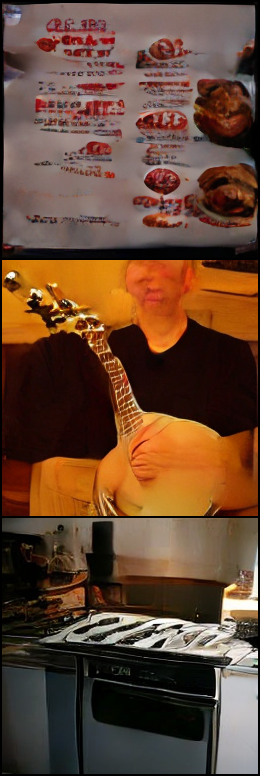}
        \caption*{\scriptsize $\alpha=0.0$}
    \end{subfigure}
    \begin{subfigure}[b]{0.09\textwidth}
        \includegraphics[width=\textwidth]{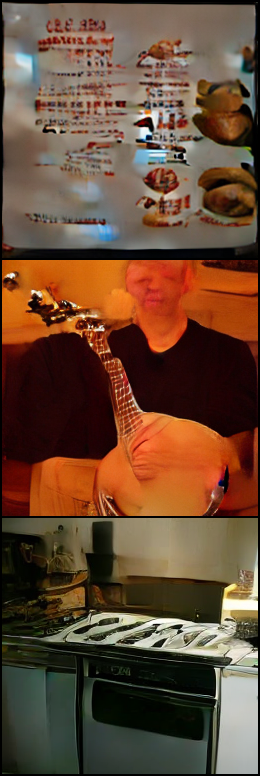}
        \caption*{\scriptsize $\alpha=0.05$}
    \end{subfigure}
    \begin{subfigure}[b]{0.09\textwidth}
        \includegraphics[width=\textwidth]{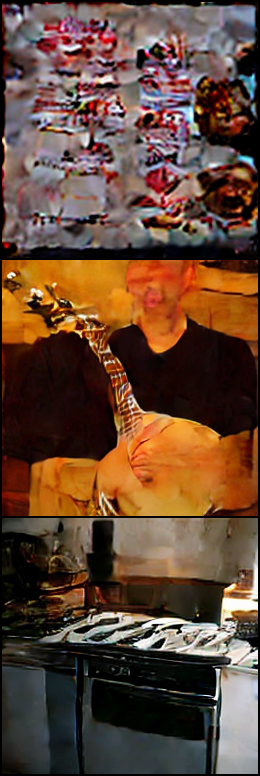}
        \caption*{\scriptsize $\gamma=0.05$}
    \end{subfigure}
    \caption{
    \textbf{GMOs provide meaningful perturbations to a model's outputs that improve the effectiveness of Neon.}
    Here, we consider applying the Neon self-training algorithm to the MeanFlow~\citep{gengMeanFlowsOnestep2025} SiT-B/2 model~\citep{maSiTExploringFlow2024} trained on ImageNet256~\citep{krizhevskyImageNetClassificationDeep2012}.
    We generated collections of $30,000$ GMOs for alpha values $0.0$ (standard outputs), $0.05$, and $0.2$.
    We then finetuned the pre-trained checkpoint with a maximum compute budget of $1.8\times10^6$.
    At regular checkpoints, we apply Neon with different weight-merging parameters $w$ in the range $[0,2]$.
    In the first and second panels, we show the minimum FID value achieved at each fine-tuning checkpoint with the corresponding weight-merging parameter $w$, respectively.
    In the third panel, we visually compare standard outputs to GMOs and the Gaussian baseline (see \Cref{fig:gaussian_examples} for more examples).
    }
    \label{fig:b2_imagenet256}
\end{figure}

In \Cref{fig:b2_imagenet256}, we see that GMOs with $\alpha$ equal to $0.05$ provide the best-performing generative model and outperform the Gaussian noise baseline.

Qualitatively, with the third panel of \Cref{fig:b2_imagenet256}, we see that the perturbations of GMOs are more similar to the standard output, whereas Gaussian noise seems to create large distortions in the output.\footnote{The data space in the MeanFlow is the latent space of a VAE tokenizer~\citep{rombach2022high}.}

\subsection{The Effect of GMOs on FID, Precision, Recall, Density and Coverage}\label{sec:fid_precision_recall}

Here, we investigate the underlying mechanics of how GMOs improve generative model performance by examining FID, precision, recall, density, and coverage~\citep{heuselGANsTrainedTwo2017,Kynkaanniemi2019improved,naeemReliableFidelityDiversity2020}.
We evaluate an IMM one-step generator~\citep{zhouInductiveMomentMatching2025} trained on ImageNet256~\citep{krizhevskyImageNetClassificationDeep2012} across a range of GMO coefficients $\alpha$.
For each evaluated curve, we report the optimum performance achieved across a fine-tuning budget $\mathcal{B}$ and merge weight $w$.
Alongside precision and recall, we report density and coverage, their outlier-robust refinements~\citep{naeemReliableFidelityDiversity2020}.
Recall is inflated by a few generated outliers with large neighborhoods, and precision by outliers in the real data.

As illustrated in the left panel of \Cref{fig:gmo_alpha_sweep}, applying GMOs yields a net improvement in FID compared to standard Neon outputs.
The FID score drops from its baseline ($\alpha$ equal to zero) to reach an optimal minimum around $\alpha$ equal to $0.2$ before slightly increasing at higher perturbation strengths.
The precision, recall, density, and coverage as a function of the GMO coefficient provides better insights into how this improvement is achieved.
The center and right panels show that as $\alpha$ increases from $0$ to $0.4$, the model's precision steadily increases while its recall steadily decreases -- a pattern that, taken at face value, would suggest GMOs trade diversity for quality.
The outlier-robust measures tell a different story: across all architectures (\Cref{tab:prdc}), density and coverage are maintained or improved from standard-outputs Neon to GMOs, and where the recall estimate dips under GMOs, most visibly on IMM ($-.007$ relative to standard-outputs Neon), coverage instead rises ($+.021$).
The FID improvements of GMOs therefore reflect higher robust fidelity together with preserved or improved robust diversity.
GMOs improve image quality and diversity simultaneously, rather than trading one for the other.

\begin{figure}[ht]
    \centering
    \resizebox{\textwidth}{!}{%
    \begin{tikzpicture}
    \begin{groupplot}[
        group style={
            group size=5 by 1,
            horizontal sep=1.4cm,
        },
        width=5cm,
        height=4.2cm,
        xmin=0,
        xmax=0.4,
        xlabel={$\alpha$},
        tick label style={font=\footnotesize},
        label style={font=\small},
        ylabel style={font=\small},
        every axis plot/.append style={
            line width=1.3pt,
            mark=none,
            line cap=round,
            line join=round
        },
    ]

    % --- FID ---
    \nextgroupplot[
        ylabel={FID},
    ]

    \addplot[color=class1, smooth]
        table [x=alpha, y=fid_mean, col sep=comma] {figures/imm_alpha_sweep.csv};

    % --- Precision ---
    \nextgroupplot[
        ylabel={Precision},
    ]

    \addplot[color=class2, smooth]
        table [x=alpha, y=p_mean, col sep=comma] {figures/imm_alpha_sweep.csv};

    % --- Recall ---
    \nextgroupplot[
        ylabel={Recall},
    ]

    \addplot[color=class3, smooth]
        table [x=alpha, y=r_mean, col sep=comma] {figures/imm_alpha_sweep.csv};

    % --- Density (outlier-robust fidelity) ---
    \nextgroupplot[
        ylabel={Density},
    ]

    \addplot[color=class4, smooth]
        table [x=alpha, y=d_mean, col sep=comma] {figures/imm_alpha_prdc.csv};

    % --- Coverage (outlier-robust diversity) ---
    \nextgroupplot[
        ylabel={Coverage},
    ]

    \addplot[color=class5, smooth]
        table [x=alpha, y=c_mean, col sep=comma] {figures/imm_alpha_prdc.csv};

    \end{groupplot}
    \end{tikzpicture}}

    \caption{
    \textbf{Neon with GMOs improves image quality and diversity.}
    For the IMM 1-step generator~\citep{zhouInductiveMomentMatching2025} trained on ImageNet256~\citep{krizhevskyImageNetClassificationDeep2012}, we plot FID, precision, recall, density, and coverage ~\citep{heuselGANsTrainedTwo2017,Kynkaanniemi2019improved} as a function of the GMO coefficient $\alpha$.
    Here, $\alpha$ equal to zero corresponds to Neon~\citep{alemohammadNeonNegativeExtrapolation2026} applied to standard outputs, while $\alpha>0$ corresponds to Neon applied to $\alpha$-GMOs.
    Each curve reports the per-$\alpha$ optimum over fine-tuning budget $\mathcal{B}$ and merge weight $w$.
    For more details, refer to Appendix \ref{sec:exp_details-sota}.}
    \label{fig:gmo_alpha_sweep}
\end{figure}
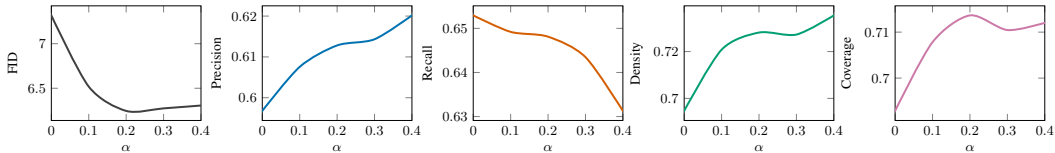

\subsection{GMOs Transfer Across Model Sizes and Inference Strategies}\label{sec:transfer}

Here, we explore whether GMOs can be transferred across model sizes and inference strategies.
In the left panel of \Cref{fig:transfer}, we explore the former with MeanFlow models~\citep{gengMeanFlowsOnestep2025} trained on ImageNet256~\citep{krizhevskyImageNetClassificationDeep2012}.
More specifically, GMOs computed from a SiT-B/2 can be successfully used to improve the performance of Neon applied to a larger SiT-L/2 model.
Although it should be noted that using GMOs from the SiT-B/2 model to improve the SiT-L/2 is more challenging, as the FID becomes more sensitive to the merging parameter $w$.

Similarly, in the right panel of \Cref{fig:transfer}, we show that GMOs extracted from a one-step IMM~\citep{zhouInductiveMomentMatching2025} generative model can be used to improve a two-step IMM on CIFAR10~\citep{krizhevskyLearningMultipleLayers2009}.
These results are significant because computing GMOs for larger generative models or those with more demanding inference strategies is more computationally expensive.

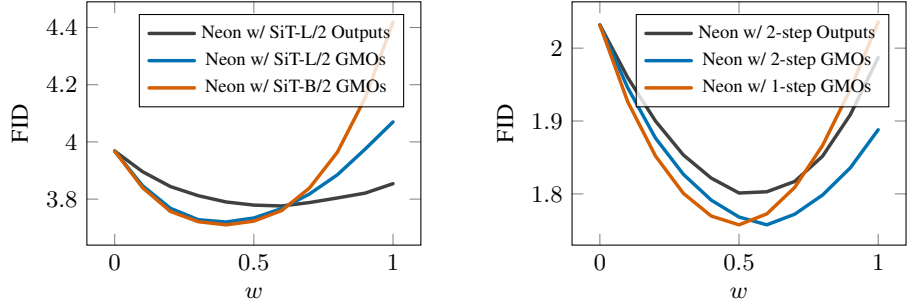
\begin{figure}[ht]
    \centering
    \begin{tikzpicture}
    \begin{groupplot}[
        group style={
            group size=2 by 1,
            horizontal sep=2cm,
        },
        width=6cm,
        height=4.8cm,
        xlabel={$w$},
        tick label style={font=\footnotesize},
        label style={font=\small},
        title style={font=\small},
        every axis plot/.append style={
            line width=1.3pt,
            mark=none,
            line cap=round,
            line join=round
        }
    ]
    \nextgroupplot[
        ylabel={FID},
        legend style={
            at={(0.95,0.95)}, 
            anchor=north east, 
            font=\scriptsize,
            fill=white, 
            fill opacity=0.8, 
            draw opacity=1, 
            text opacity=1
        }
    ]
    % Use the first column by index: its CSV header contains a UTF-8 BOM.
    \addplot[color=class1] table [x index=0, y=neon_base, col sep=comma] {data/meanflow_transfer.csv};
    \addlegendentry{Neon w/ SiT-L/2 Outputs}
    \addplot[color=class2] table [x index=0, y=neon_l2_gmo, col sep=comma] {data/meanflow_transfer.csv};
    \addlegendentry{Neon w/ SiT-L/2 GMOs}
    \addplot[color=class3] table [x index=0, y=neon_b2_gmo, col sep=comma] {data/meanflow_transfer.csv};
    \addlegendentry{Neon w/ SiT-B/2 GMOs}
    
    \nextgroupplot[
        ylabel={FID},
        legend style={
            at={(0.95,0.95)}, 
            anchor=north east, 
            font=\scriptsize,
            fill=white, 
            fill opacity=0.8, 
            draw opacity=1, 
            text opacity=1
        }
    ]

    \addplot[color=class1] table [x=w, y=vanilla_2step, col sep=comma] {data/imm_transfer_best_epoch_sweep.csv};
    \addlegendentry{Neon w/ 2-step Outputs}
    \addplot[color=class2] table [x=w, y=spectral_2step_a0.01, col sep=comma] {data/imm_transfer_best_epoch_sweep.csv};
    \addlegendentry{Neon w/ 2-step GMOs}
    \addplot[color=class3] table [x=w, y=spectral_1step_a0.01, col sep=comma] {data/imm_transfer_best_epoch_sweep.csv};
    \addlegendentry{Neon w/ 1-step GMOs}
    
    \end{groupplot}
    \end{tikzpicture}
    \caption{
    \textbf{GMOs are transferable across generative models of different sizes and inference strategies.}
    In the left panel, we consider using ImageNet256 GMOs from a SiT-B/2 MeanFlow model to improve a SiT-L/2 MeanFlow model using Neon.
    We compare this to Neon applied directly to the SiT-L/2 model's outputs and GMOs.
    A computational fine tuning budget of $1.2\times10^6$ is used in every case, and GMOs are generated with $\alpha$ equal to $0.1$.
    In the right panel, we consider using CIFAR10 GMOs from a one-step IMM model to improve a 2-step IMM model using Neon.
    We compare this to Neon applied directly to the two-step IMM model's outputs and GMOs.
    For more experimental detail on this particular experiment, refer to Appendix \ref{sec:exp_details-transfer}.
    }
    \label{fig:transfer}
\end{figure}

\subsection{GMOs Improve State-of-the-Art Generative Models}\label{sec:improve_sota}

Here, we demonstrate that Neon with GMOs can improve the performance -- measured by FID~\citep{heuselGANsTrainedTwo2017} -- more effectively than Neon with standard outputs.

As detailed in \Cref{tab:sota}, we evaluate the impact of integrating GMOs into the Neon self-training algorithm across several one-step generative models, including IMM~\citep{zhouInductiveMomentMatching2025}, MeanFlow~\citep{gengMeanFlowsOnestep2025}, and AlphaFlow~\citep{zhangAlphaFlowUnderstandingImproving2026} architectures. 
The evaluations are conducted on models pre-trained on ImageNet256~\citep{krizhevskyImageNetClassificationDeep2012}. 
The results clearly indicate that self-training with GMOs consistently yields superior generative performance compared to self-training with standard outputs.  
Across all evaluated model scales, the application of GMOs pushes the FID below both the pre-trained base model baseline and the standard Neon baseline. As shown in Table~1, Neon with GMOs achieves a lower FID than Neon with standard outputs across all five evaluated architectures. Measured relative to the FID reduction that standard Neon achieves over the base model, GMOs increase this reduction by 17--105\% across the five architectures, nearly doubling it on IMM and more than doubling it on AlphaFlow SiT-B/2 (see Table~5).

Furthermore, the strength of the negative signal is visible in the finetuning dynamics. {Models can attain their optimal FID using GMOs with a smaller fine-tuning budget ($\mathcal{B}$) than is required to reach optimal performance with standard outputs. We quantify this across a broader sweep of $\alpha$ values on a Meanflow model trained on CIFAR10; we find that increasing $\alpha$ systematically lowers both the optimal fine tuning budget $\mathcal{B}$ and the optimal merge weight $w$ required by Neon (Appendix~\ref{sec:appendix_compute}).
This confirms that the targeted signal provided by GMOs not only raises the ultimate performance ceiling but also reaches the optimal checkpoint sooner, because GMOs induce the mode-seeking degradation faster. We show that GMOs statistically significantly improve generative model self-training across all architectures tested (see \Cref{tab:significance}.)

\begin{table}[ht]
    \centering
    \caption{
    \textbf{The utilization of GMOs in the Neon self-training algorithm yields better generative models.}
    Here, we implement the Neon self-training algorithm on various one-step generative models trained on ImageNet256.
    We consider Neon on the standard output of these generators, or on the corresponding GMOs, across various compute levels (as a percentage of pre-training compute), weight-merging parameters $w$, and $\alpha$ values for GMOs (refer to Appendix \ref{sec:exp_details-sota} for more details).
    Among the best-performing configurations, we report the average and standard deviation of the FID across five random seeds and $50,000$ output samples.
    We provide a repository \href{https://github.com/PatrickBats/Geometrically-Modified-Outputs}{\textbf{here}} containing the model checkpoints obtained using GMOs.
    }
    \label{tab:sota}
    \vspace{0.5em}
    \resizebox{\textwidth}{!}{%
    \begin{tabular}{lcccccccc}
        \toprule
        Architecture & Base FID & \multicolumn{3}{c}{\textbf{Neon w/ Standard Outputs}} & \multicolumn{4}{c}{\textbf{Neon w/ GMOs}} \\
        \cmidrule(lr){3-5} \cmidrule(lr){6-9}
        & & $\mathcal{B}$ & $w$ & FID & $\mathcal{B}$ & $w$ & $\alpha$ & FID \\
        \midrule    
        IMM & $8.34$ & $2.5\times10^6(6.1\times10^{-3}\%)$ & $1.6$ & $7.32 (\pm 0.07)$ & $2.87\times 10^6(7.0\times10^{-3}\%)$ & $1.6$ & $0.2$ & $\mathbf{6.25(\pm0.03)}$ \\
        MeanFlow SiT-B/2 & $6.08$ & $7.2\times10^5 (0.25\%)$ & $0.8$ & $5.70 (\pm 0.01)$ & $4.8\times10^5$ $(0.17\%)$ & $1.2$ & $0.05$ & $\mathbf{5.60(\pm0.01)}$ \\
        MeanFlow SiT-L/2 & $3.97$ & $1.8\times10^6(0.63\%)$ & $0.5$ & $3.74(\pm0.02)$ & $1.8\times10^6 (0.63\%)$ & $0.3$ & $0.1$ & $\mathbf{3.70(\pm0.02)}$ \\
        AlphaFlow SiT-B/2 & $5.55$ & $4.8\times10^5(0.16\%)$ & $0.6$ & $5.36(\pm0.02)$ & $4.8\times10^5(0.16\%)$ & $0.8$ & $0.1$ & $\mathbf{5.16(\pm0.02)}$ \\
        AlphaFlow SiT-XL/2 & $2.93$ & $3\times10^5(0.1\%)$ & $1.2$ & $2.64(\pm0.02)$ & $3\times10^5(0.1\%)$ & $1.1$ & $0.05$ & $\mathbf{2.59(\pm0.02)}$ \\
        \bottomrule
    \end{tabular}}
\end{table}

\subsection{GMOs Improve Negative Guidance Beyond Neon}\label{sec:sims}

GMOs operate purely in data space and are agnostic to the self-training algorithm that utilizes them.
We chose Neon as the primary algorithm because it is the strongest and most efficient negative guidance method, requiring no additional model at inference and no sampling overhead.
To show that our contribution is not solely tied to Neon, we evaluate a second, structurally different algorithm: SIMS-style guidance~\citep{alemohammadSelfimprovingDiffusionModels2024} keeps the finetuned model in the sampling loop and extrapolates predictions at inference time, computing $(1+\omega)\,f_{\theta} - \omega\,f_{\tilde{\theta}}$ at every sampling step instead of merging weights once.

On IMM, guidance with the model finetuned on standard outputs improves the base FID by $0.75$ (from $8.34\pm.06$ to $7.59\pm.06$), while the same guidance with the model finetuned on GMOs improves it by $1.42$ (to $6.92\pm.05$), nearly twice the improvement.
% $\alpha=0.1$ dose is again the optimum: $\alpha=0.2$ also beats standard outputs but loses to $0.1$ and requires a weaker $\omega$, the same dose-strength trade-off observed in weight space. 
GMOs therefore strengthen both consumers of negative signals that we test, weight-space extrapolation and output-space guidance.

\section{Discussion}\label{sec:discussion}

\paragraph{Summary.}

In this paper, we introduced Geometrically Modified Outputs (GMOs) and demonstrated that they improve the performance of the Neon self-training algorithm for one-step generative models.
GMOs are an augmentation technique that modifies model outputs by reweighting the singular values of the generator's input-output Jacobian. 
This structural augmentation provides a stronger and more targeted negative signal for self-training algorithms, such as Neon and SIMS. Through empirical evaluations, we demonstrated that integrating GMOs into multiple self-training frameworks consistently yields superior generative performance. 

\paragraph{Relation to other self-improvement methods.}
GMOs should be viewed as an enhancement to negative-signal
construction, rather than a competing standalone self-training
framework. They modify the synthetic outputs used to train
the negative model while retaining the downstream fine-tuning
and guidance procedures. We demonstrate their utility with
both Neon's weight-space extrapolation and SIMS-style
inference-time guidance.

The computational requirements of the downstream method remain
important. SIMS~\citep{alemohammadSelfimprovingDiffusionModels2024}
and Autoguidance~\citep{karrasGuidingDiffusionModel2024}
require auxiliary-model predictions during sampling, while
DDO~\citep{zhengDirectDiscriminativeOptimization2025}
uses multiple self-play rounds and reports a training budget
of approximately $12\%$ of pretraining compute. We primarily evaluate GMOs with
Neon~\citep{alemohammadNeonNegativeExtrapolation2026},
whose merged model requires no additional inference-time
evaluations. In this setting, GMOs add offline sample-construction
cost while preserving Neon's sampling efficiency.
Their effectiveness under both Neon and SIMS motivates
investigating whether geometry-based output augmentation can
also strengthen other negative-signal methods, including
DDO and Autoguidance.

\paragraph{Limitations and Future Directions.} 

Our results suggest that Geometrically Modified Outputs (GMOs) can improve self-training by leveraging the generator's geometric structure, although several directions remain open for future work.

First, GMOs require specifying $\alpha$. Although outputs for multiple $\alpha$ values can be generated in parallel, applying Neon to each is costly. In our experiments, a fixed $\alpha=0.1$ without per-model tuning outperforms standard-outputs Neon on four of five ImageNet-trained architectures and is within seed noise on the fifth. Generating GMOs at additional $\alpha$ values is nearly free because they reuse the same probe computation (Algorithm 2); only evaluation and finetuning scale with the number of values tested. However, a search-free heuristic for estimating the optimal $\alpha$ for a given model and dataset would be valuable. Second, we explore GMOs only in the image domain; extending them to other modalities and architectures remains future work.  

% GMOS are most natural for one-step generative models, where the relevant input-output Jacobian is easier to access. 
% Although the field of one-step generative models is yielding increasingly more capable models and we show evidence that GMOs can work for few-step generators (see \Cref{fig:transfer}), exploring GMOs on few-step generators is a promising direction of future work.

\newpage
\paragraph{Acknowledgements}
This work was supported by ONR grant N00014-23-1-2714, DOE grant DE-SC0020345, DOI grant 140D0423C0076, and a Google Cloud Computing Award.

\newpage

\bibliographystyle{iclr2027_conference}
\bibliography{references}

\begin{thebibliography}{37}
\providecommand{\natexlab}[1]{#1}
\providecommand{\url}[1]{\texttt{#1}}
\expandafter\ifx\csname urlstyle\endcsname\relax
  \providecommand{\doi}[1]{doi: #1}\else
  \providecommand{\doi}{doi: \begingroup \urlstyle{rm}\Url}\fi

\bibitem[AI(2026)]{TrendsArtificialIntelligence}
Epoch AI.
\newblock {Trends in Artificial Intelligence}, February 2026.
\newblock URL \url{https://epoch.ai/trends}.

\bibitem[Alemohammad et~al.(2024{\natexlab{a}})Alemohammad, {Casco-Rodriguez}, Luzi, Humayun, Babaei, LeJeune, Siahkoohi, and Baraniuk]{alemohammadSelfconsumingGenerativeModels2024}
Sina Alemohammad, Josue {Casco-Rodriguez}, Lorenzo Luzi, Ahmed~Imtiaz Humayun, Hossein Babaei, Daniel LeJeune, Ali Siahkoohi, and Richard Baraniuk.
\newblock {Self-Consuming Generative Models Go MAD}.
\newblock In \emph{{International Conference on Learning Representations}}, 2024{\natexlab{a}}.

\bibitem[Alemohammad et~al.(2024{\natexlab{b}})Alemohammad, Humayun, Agarwal, Collomosse, and Baraniuk]{alemohammadSelfimprovingDiffusionModels2024}
Sina Alemohammad, Ahmed~Imtiaz Humayun, Shruti Agarwal, John Collomosse, and Richard Baraniuk.
\newblock {Self-Improving Diffusion Models With Synthetic Data}.
\newblock \emph{arXiv:2408.16333}, 2024{\natexlab{b}}.

\bibitem[Alemohammad et~al.(2026)Alemohammad, Wang, and Baraniuk]{alemohammadNeonNegativeExtrapolation2026}
Sina Alemohammad, Zhangyang Wang, and Richard Baraniuk.
\newblock Neon: {{Negative Extrapolation From Self-training Improves Image Generation}}.
\newblock In \emph{{{International Conference}} on {{Learning Representations}}}, 2026.

\bibitem[Batsell et~al.(2026)Batsell, Walker, and Baraniuk]{batsellGeometricPerspectiveRecursive2026}
Patrick Batsell, Thomas Walker, and Richard Baraniuk.
\newblock {A Geometric Perspective on Recursive Synthetic Training}.
\newblock In \emph{{ICLR Workshop on Deep Generative Model in Machine Learning: Theory, Principle and Efficacy}}, 2026.

\bibitem[Deng et~al.(2026)Deng, Li, Li, Du, and He]{dengGenerativeModelingDrifting2026}
Mingyang Deng, He~Li, Tianhong Li, Yilun Du, and Kaiming He.
\newblock {Generative Modeling Via Drifting}.
\newblock \emph{arXiv:2602.04770}, 2026.

\bibitem[Feng et~al.(2025)Feng, Dohmatob, Yang, Charton, and Kempe]{fengModelCollapseScaling2025}
Yunzhen Feng, Elvis Dohmatob, Pu~Yang, Francois Charton, and Julia Kempe.
\newblock {Beyond Model Collapse: Scaling up With Synthesized Data Requires Verification}.
\newblock In \emph{{International Conference on Learning Representations}}, 2025.

\bibitem[Frans et~al.(2025)Frans, Hafner, Levine, and Abbeel]{fransOneStepDiffusion2025}
Kevin Frans, Danijar Hafner, Sergey Levine, and Pieter Abbeel.
\newblock {One Step Diffusion Via Shortcut Models}.
\newblock In \emph{{International Conference on Learning Representations}}, 2025.

\bibitem[Geng et~al.(2025)Geng, Deng, Bai, Kolter, and He]{gengMeanFlowsOnestep2025}
Zhengyang Geng, Mingyang Deng, Xingjian Bai, J~Zico Kolter, and Kaiming He.
\newblock {Mean Flows For One-step Generative Modeling}.
\newblock In \emph{{Advances in Neural Information Processing Systems}}, 2025.

\bibitem[He et~al.(2023)He, Sun, Yu, Xue, Zhang, Torr, Bai, and Qi]{heSyntheticDataGenerative2023}
Ruifei He, Shuyang Sun, Xin Yu, Chuhui Xue, Wenqing Zhang, Philip Torr, Song Bai, and Xiaojuan Qi.
\newblock {Is Synthetic Data From Generative Models Ready For Image Recognition?}
\newblock In \emph{{International Conference on Learning Representations}}, 2023.

\bibitem[Henighan et~al.(2020)Henighan, Kaplan, Katz, Chen, Hesse, Jackson, Jun, Brown, Dhariwal, Gray, Hallacy, Mann, Radford, Ramesh, Ryder, Ziegler, Schulman, Amodei, and McCandlish]{henighan2020scalinglawsautoregressivegenerative}
Tom Henighan, Jared Kaplan, Mor Katz, Mark Chen, Christopher Hesse, Jacob Jackson, Heewoo Jun, Tom~B. Brown, Prafulla Dhariwal, Scott Gray, Chris Hallacy, Benjamin Mann, Alec Radford, Aditya Ramesh, Nick Ryder, Daniel~M. Ziegler, John Schulman, Dario Amodei, and Sam McCandlish.
\newblock {Scaling Laws for Autoregressive Generative Modeling}.
\newblock \emph{arXiv:2010.14701}, 2020.

\bibitem[Heusel et~al.(2017)Heusel, Ramsauer, Unterthiner, Nessler, and Hochreiter]{heuselGANsTrainedTwo2017}
Martin Heusel, Hubert Ramsauer, Thomas Unterthiner, Bernhard Nessler, and Sepp Hochreiter.
\newblock {GANs Trained by a Two Time-scale Update Rule Converge to a Local Nash Equilibrium}.
\newblock In I.~Guyon, U.~Von Luxburg, S.~Bengio, H.~Wallach, R.~Fergus, S.~Vishwanathan, and R.~Garnett (eds.), \emph{{Advances in Neural Information Processing Systems}}, 2017.

\bibitem[Kaplan et~al.(2020)Kaplan, McCandlish, Henighan, Brown, Chess, Child, Gray, Radford, Wu, and Amodei]{kaplanScalingLawsNeural2020}
Jared Kaplan, Sam McCandlish, Tom Henighan, Tom~B. Brown, Benjamin Chess, Rewon Child, Scott Gray, Alec Radford, Jeffrey Wu, and Dario Amodei.
\newblock {Scaling Laws for Neural Language Models}.
\newblock \emph{arXiv:2001.08361}, 2020.

\bibitem[Karras et~al.(2024)Karras, Aittala, Kynk{\"a}{\"a}nniemi, Lehtinen, Aila, and Laine]{karrasGuidingDiffusionModel2024}
Tero Karras, Miika Aittala, Tuomas Kynk{\"a}{\"a}nniemi, Jaakko Lehtinen, Timo Aila, and Samuli Laine.
\newblock {Guiding a Diffusion Model With a Bad Version of Itself}.
\newblock In \emph{{Advances in Neural Information Processing Systems}}, 2024.

\bibitem[Kim et~al.(2023)Kim, Kim, Kwon, Kang, and Moon]{kimRefiningGenerativeProcess2023}
Dongjun Kim, Yeongmin Kim, Se~Jung Kwon, Wanmo Kang, and Il-Chul Moon.
\newblock {Refining Generative Process With Discriminator Guidance in Score-based Diffusion Models}.
\newblock \emph{arXiv:2211.17091}, 2023.

\bibitem[Krizhevsky \& Hinton(2009)Krizhevsky and Hinton]{krizhevskyLearningMultipleLayers2009}
Alex Krizhevsky and Geoffrey Hinton.
\newblock {Learning Multiple Layers of Features from Tiny Images}.
\newblock Technical report, University of Toronto, 2009.

\bibitem[Krizhevsky et~al.(2012)Krizhevsky, Sutskever, and Hinton]{krizhevskyImageNetClassificationDeep2012}
Alex Krizhevsky, Ilya Sutskever, and Geoffrey~E Hinton.
\newblock {ImageNet Classification with Deep Convolutional Neural Networks}.
\newblock In \emph{{Advances in Neural Information Processing Systems}}, 2012.

\bibitem[Kynk\"{a}\"{a}nniemi et~al.(2019)Kynk\"{a}\"{a}nniemi, Karras, Laine, Lehtinen, and Aila]{Kynkaanniemi2019improved}
Tuomas Kynk\"{a}\"{a}nniemi, Tero Karras, Samuli Laine, Jaakko Lehtinen, and Timo Aila.
\newblock {Improved Precision and Recall Metric for Assessing Generative Models}.
\newblock In \emph{Advances in Neural Information Processing Systems}, volume~32. Curran Associates, Inc., 2019.

\bibitem[Ma et~al.(2024)Ma, Goldstein, Albergo, Boffi, {Vanden-Eijnden}, and Xie]{maSiTExploringFlow2024}
Nanye Ma, Mark Goldstein, Michael~S Albergo, Nicholas~M Boffi, Eric {Vanden-Eijnden}, and Saining Xie.
\newblock {SiT: Exploring Flow and Diffusion-based Generative Models With Scalable Interpolant Transformers}.
\newblock In \emph{{European Conference on Computer Vision}}. Springer, 2024.

\bibitem[Muennighoff et~al.(2025)Muennighoff, Rush, Barak, Scao, Piktus, Tazi, Pyysalo, Wolf, and Raffel]{muennighoff2025scalingdataconstrainedlanguagemodels}
Niklas Muennighoff, Alexander~M. Rush, Boaz Barak, Teven~Le Scao, Aleksandra Piktus, Nouamane Tazi, Sampo Pyysalo, Thomas Wolf, and Colin Raffel.
\newblock {Scaling Data-Constrained Language Models}.
\newblock \emph{arXiv:2305.16264}, 2025.

\bibitem[Naeem et~al.(2020)Naeem, Oh, Uh, Choi, and Yoo]{naeemReliableFidelityDiversity2020}
Muhammad~Ferjad Naeem, Seong~Joon Oh, Youngjung Uh, Yunjey Choi, and Jaejun Yoo.
\newblock Reliable fidelity and diversity metrics for generative models.
\newblock In \emph{International Conference on Machine Learning}, 2020.

\bibitem[Radford et~al.(2021)Radford, Kim, Hallacy, Ramesh, Goh, Agarwal, Sastry, Askell, Mishkin, Clark, et~al.]{radford2021learning}
Alec Radford, Jong~Wook Kim, Chris Hallacy, Aditya Ramesh, Gabriel Goh, Sandhini Agarwal, Girish Sastry, Amanda Askell, Pamela Mishkin, Jack Clark, et~al.
\newblock {Learning Transferable Visual Models From Natural Language Supervision}.
\newblock In \emph{{International Conference on Machine Learning}}, 2021.

\bibitem[Rombach et~al.(2022)Rombach, Blattmann, Lorenz, Esser, and Ommer]{rombach2022high}
Robin Rombach, Andreas Blattmann, Dominik Lorenz, Patrick Esser, and Bj{\"o}rn Ommer.
\newblock {High-resolution Image Synthesis With Latent Diffusion Models}.
\newblock In \emph{{IEEE/CVF Conference on Computer Vision and Pattern Recognition}}, 2022.

\bibitem[Ronneberger et~al.(2015)Ronneberger, Fischer, and Brox]{ronnebergerUNetConvolutionalNetworks2015}
Olaf Ronneberger, Philipp Fischer, and Thomas Brox.
\newblock {U-Net: Convolutional Networks for Biomedical Image Segmentation}.
\newblock In \emph{{International Conference on Medical Image Computing and Computer-assisted Intervention}}. Springer, 2015.

\bibitem[Shumailov et~al.(2024)Shumailov, Shumaylov, Zhao, Papernot, Anderson, and Gal]{shumailovAIModelsCollapse2024}
Ilia Shumailov, Zakhar Shumaylov, Yiren Zhao, Nicolas Papernot, Ross Anderson, and Yarin Gal.
\newblock {AI Models Collapse When Trained on Recursively Generated Data}.
\newblock \emph{Nature}, 631\penalty0 (8022), 2024.

\bibitem[Silver et~al.(2018)Silver, Hubert, Schrittwieser, Antonoglou, Lai, Guez, Lanctot, Sifre, Kumaran, Graepel, Lillicrap, Simonyan, and Hassabis]{silverGeneralReinforcementLearning2018}
David Silver, Thomas Hubert, Julian Schrittwieser, Ioannis Antonoglou, Matthew Lai, Arthur Guez, Marc Lanctot, Laurent Sifre, Dharshan Kumaran, Thore Graepel, Timothy Lillicrap, Karen Simonyan, and Demis Hassabis.
\newblock {A General Reinforcement Learning Algorithm That Masters Chess, Shogi, and Go Through Self-play}.
\newblock \emph{Science}, 362\penalty0 (6419), 2018.

\bibitem[Song \& Ermon(2019)Song and Ermon]{song2019generative}
Yang Song and Stefano Ermon.
\newblock {Generative Modeling by Estimating Gradients of the Data Distribution}.
\newblock 2019.

\bibitem[Song et~al.(2023)Song, Dhariwal, Chen, and Sutskever]{song2023consistencymodels}
Yang Song, Prafulla Dhariwal, Mark Chen, and Ilya Sutskever.
\newblock {Consistency Models}.
\newblock \emph{arXiv:2303.01469}, 2023.

\bibitem[van~der Maaten \& Hinton(2008)van~der Maaten and Hinton]{maatenVisualizingDataUsing2008}
Laurens van~der Maaten and Geoffrey Hinton.
\newblock {Visualizing Data Using t-SNE}.
\newblock \emph{{Journal of Machine Learning Research}}, 9\penalty0 (86), 2008.

\bibitem[Villalobos et~al.(2024)Villalobos, Ho, Sevilla, Besiroglu, Heim, and Hobbhahn]{villalobosWillWeRun2024}
Pablo Villalobos, Anson Ho, Jaime Sevilla, Tamay Besiroglu, Lennart Heim, and Marius Hobbhahn.
\newblock {Will We Run Out of Data? Limits of LLM Scaling Based on Human-generated Data}, 2024.

\bibitem[Wang et~al.(2023)Wang, Pang, Du, Lin, Liu, and Yan]{wangBetterDiffusionModels2023}
Zekai Wang, Tianyu Pang, Chao Du, Min Lin, Weiwei Liu, and Shuicheng Yan.
\newblock {Better Diffusion Models Further Improve Adversarial Training}.
\newblock In \emph{{International Conference on Machine Learning}}. PMLR, 2023.

\bibitem[Yang et~al.(2026)Yang, Geng, Ju, Tian, and Wang]{yang2026representation}
Jiawei Yang, Zhengyang Geng, Xuan Ju, Yonglong Tian, and Yue Wang.
\newblock {Representation Fr\'echet Loss for Visual Generation}.
\newblock \emph{arXiv:2604.28190}, 2026.

\bibitem[Yuan et~al.(2024)Yuan, Chen, Ji, and Gu]{yuanSelfplayFinetuningDiffusion2024}
Huizhuo Yuan, Zixiang Chen, Kaixuan Ji, and Quanquan Gu.
\newblock {Self-play Fine-tuning of Diffusion Models For Text-to-image Generation}.
\newblock In \emph{{Advances in Neural Information Processing Systems}}. Curran Associates Inc., 2024.

\bibitem[Yue et~al.(2026)Yue, Jia, Hou, and Goldstein]{yue2026image}
Kaiyu Yue, Menglin Jia, Ji~Hou, and Tom Goldstein.
\newblock {Image Generation With a Sphere Encoder}.
\newblock \emph{arXiv:2602.15030}, 2026.

\bibitem[Zhang et~al.(2026)Zhang, Siarohin, Menapace, Vasilkovsky, Tulyakov, Qu, and Skorokhodov]{zhangAlphaFlowUnderstandingImproving2026}
Huijie Zhang, Aliaksandr Siarohin, Willi Menapace, Michael Vasilkovsky, Sergey Tulyakov, Qing Qu, and Ivan Skorokhodov.
\newblock {AlphaFlow: Understanding and Improving MeanFlow Models}.
\newblock In \emph{{International Conference on Learning Representations}}, 2026.

\bibitem[Zheng et~al.(2025)Zheng, Chen, Chen, He, Liu, Zhu, and Zhang]{zhengDirectDiscriminativeOptimization2025}
Kaiwen Zheng, Yongxin Chen, Huayu Chen, Guande He, Ming-Yu Liu, Jun Zhu, and Qinsheng Zhang.
\newblock {Direct Discriminative Optimization: Your Likelihood-based Visual Generative Model Is Secretly a GAN Discriminator}.
\newblock In \emph{{International Conference on Machine Learning}}, 2025.

\bibitem[Zhou et~al.(2025)Zhou, Ermon, and Song]{zhouInductiveMomentMatching2025}
Linqi Zhou, Stefano Ermon, and Jiaming Song.
\newblock {Inductive Moment Matching}.
\newblock In \emph{{International Conference on Machine Learning}}, 2025.

\end{thebibliography}

\newpage
\appendix

\section{\texorpdfstring{Proof of \Cref{thm:gmo_algorithm}}{Proof of Theorem 1}}\label{sec:proof}

The standard output of the generator can be decomposed as $\vs=\mJ_{\vz}\vz+\vb_{\vz}$.
Using a singular value decomposition, the Jacobian-vector product $\mJ_{\vz}\vz$ can be expanded as a sum of its singular components. 
Let $r$ be the rank of $\mJ_{\vz}$, then
\begin{equation*}
    \mJ_{\vz}\vz=\sum_{k=1}^{r} \sigma^{(k)}_{\vz}\left(\vv^{(k)\top}_{\vz}\vz\right)\vu^{(k)}_{\vz}
\end{equation*}
Then, by construction, we have
\begin{align*}
    \tilde{\mJ}_{\vz}\vz&=\tilde{\sigma}^{(1)}_{\vz}\left(\vv^{(1)\top}_{\vz}\vz\right)\vu^{(1)}_{\vz}+\sum_{k=2}^{r} \tilde{\sigma}^{(k)}_{\vz}\left(\vv^{(k)\top}_{\vz}\vz\right)\vu^{(k)}_{\vz}\\&=\tilde{\sigma}^{(1)}_{\vz}\left(\vv^{(1)\top}_{\vz}\vz\right)\vu^{(1)}_{\vz}+\sum_{k=2}^{r} \sqrt{1-\alpha}\sigma^{(k)}_{\vz}\left(\vv^{(k)\top}_{\vz}\vz\right)\vu^{(k)}_{\vz}
\end{align*}
Then, since
\begin{equation*}
    \sum_{k=2}^{r} \sigma^{(k)}_{\vz}\left(\vv^{(k)\top}_{\vz} \vz\right)\vu^{(k)}_{\vz}=\mJ_{\vz}\vz - \sigma^{(1)}_{\vz}\left(\vv^{(1)\top}_{\vz}\vz\right)\vu^{(1)}_{\vz},
\end{equation*}
we can write
\begin{equation*}
    \sum_{k=2}^{r}\sqrt{1-\alpha}\sigma^{(k)}_{\vz}\left(\vv^{(k)\top}_{\vz}\vz\right)\vu^{(k)}_{\vz}=\sqrt{1-\alpha}\left(\vs-\vb_{\vz}\right)-\sqrt{1-\alpha}\sigma^{(1)}_{\vz}\left(\vv^{(1)\top}_{\vz}\vz\right)\vu^{(1)}_{\vz}.
\end{equation*}
Meaning,
\begin{align*}
    \tilde{\mJ}_{\vz}\vz&=\tilde{\sigma}^{(1)}_{\vz}\left(\vv^{(1)\top}_{\vz}\vz\right)\vu^{(1)}_{\vz}+\sqrt{1-\alpha}\left(\vs-\vb_{\vz}\right)-\sqrt{1-\alpha}\sigma^{(1)}_{\vz}\left(\vv^{(1)\top}_{\vz} \vz\right)\vu^{(1)}_{\vz}\\&=\sqrt{1-\alpha}\left(\vs-\vb_{\vz}\right)+\left(\tilde{\sigma}^{(1)}_{\vz}-\sqrt{1-\alpha}\sigma^{(1)}_{\vz}\right)\left(\vv^{(1)\top}_{\vz}\vz\right)\vu^{(1)}_{\vz}
\end{align*}
Therefore,
\begin{equation*}
    \tilde{\vs}=\tilde{\mJ}_{\vz}\vz+\vb_{\vz}=\sqrt{1-\alpha}\left(\vs-\vb_{\vz}\right)+\left(\tilde{\sigma}^{(1)}_{\vz}-\sqrt{1-\alpha}\sigma^{(1)}_{\vz}\right)\left(\vv^{(1)\top}_{\vz}\vz\right)\vu^{(1)}_{\vz}+\vb_{\vz}.
\end{equation*}
This is precisely the output generated by \Cref{alg:gmos}, thus the proof is complete. \qed

\subsection{Spectral Energy Redistribution}
\label{app:spectral_energy}

The reweighting in Section 3.1 admits a direct interpretation
as a redistribution of spectral energy. For a fixed latent
vector $\vz$, write $\sigma_k=\sigma_{\vz}^{(k)}$ and
$\tilde{\sigma}_k=\tilde{\sigma}_{\vz}^{(k)}$, and let
$E=\|\mJ_{\vz}\|_F^2=\sum_{k=1}^{r}\sigma_k^2$.
For $\alpha\in[0,1]$, the prescribed transformation satisfies
\[
\tilde{\sigma}_1^2
=
\sigma_1^2+\alpha\sum_{k=2}^{r}\sigma_k^2,
\qquad
\tilde{\sigma}_k^2=(1-\alpha)\sigma_k^2
\quad (k\geq 2).
\]
Thus, the leading component receives exactly the spectral
energy removed from the trailing components. Consequently,
\[
\|\tilde{\mJ}_{\vz}\|_F^2
=
\tilde{\sigma}_1^2+\sum_{k=2}^{r}\tilde{\sigma}_k^2
=
\|\mJ_{\vz}\|_F^2.
\]

When $E>0$, define the normalized spectral energies
$p_k=\sigma_k^2/E$ and $\tilde{p}_k=\tilde{\sigma}_k^2/E$.
The transformation can then be written as
\[
\tilde{\mathbf p}
=
(1-\alpha)\mathbf p+\alpha\mathbf e_1,
\]
where $\mathbf e_1=(1,0,\ldots,0)^\top$.
Hence, $\alpha$ controls a linear interpolation in normalized
squared singular values between the original spectrum and
one with all spectral energy in the leading component.

The singular directions and affine offset $\vb_{\vz}$ are
retained. At $\alpha=0$, the original output is recovered.
At $\alpha=1$, the modified linear term has rank at most one,
while the offset remains unchanged. The preserved quantity
is the squared Frobenius norm of the constructed matrix,
not necessarily the norm of the generated output.

This provides a design rationale for the GMO reweighting:
it introduces a controlled concentration of the local spectrum,
motivated by the spectral concentration observed during
self-training, while preserving total spectral energy. 

\section{GMOs are Mode Seeking}\label{sec:mode_seeking}

In \citet{alemohammadNeonNegativeExtrapolation2026}, Neon is shown to theoretically work when finetuning against model outputs that are ``mode-seeking''.
Informally, this means that the model outputs are concentrated in high-density regions of the learned data distribution.
With \Cref{fig:mode_seeking}, we show that GMOs amplify this effect.
In the top row of \Cref{fig:mode_seeking}, we visualize the CLIP embeddings~\citep{radford2021learning} of GMOs across a range of $\alpha$ values using t-SNE.
It can be seen that the embeddings become increasingly concentrated for larger values of $\alpha$, indicating that the images are focusing on the same regions of the data distribution.

With the bottom row of \Cref{fig:mode_seeking}, we visualize GMOs for each corresponding value of $\alpha$.
Qualitatively, we observe that the perturbations concentrate on distortions in the output.
We provide more examples in \Cref{fig:mode_seeking2}.
In particular, the bottom row of \Cref{fig:mode_seeking2} shows that when the standard output has no clear distortions, the corresponding GMO looks relatively unchanged.
This supports the idea that GMOs amplify the negative signal in model outputs.

\begin{figure}[ht]
    \centering
    \begin{subfigure}[b]{0.18\textwidth}
        \centering
        \includegraphics[width=\textwidth]{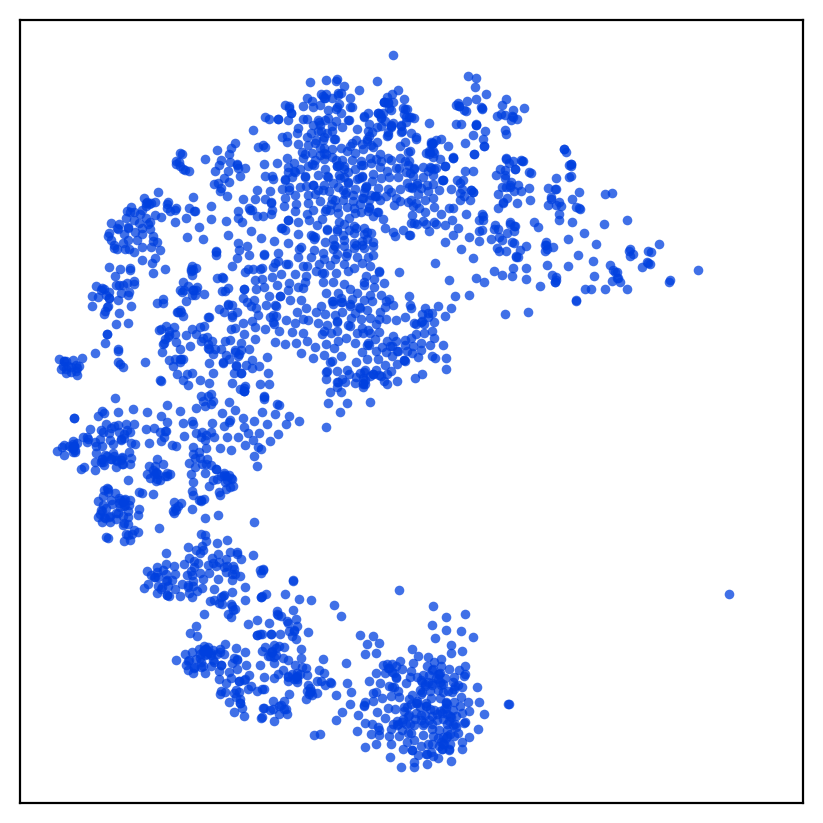}
    \end{subfigure}
    \hfill
    \begin{subfigure}[b]{0.18\textwidth}
        \centering
        \includegraphics[width=\textwidth]{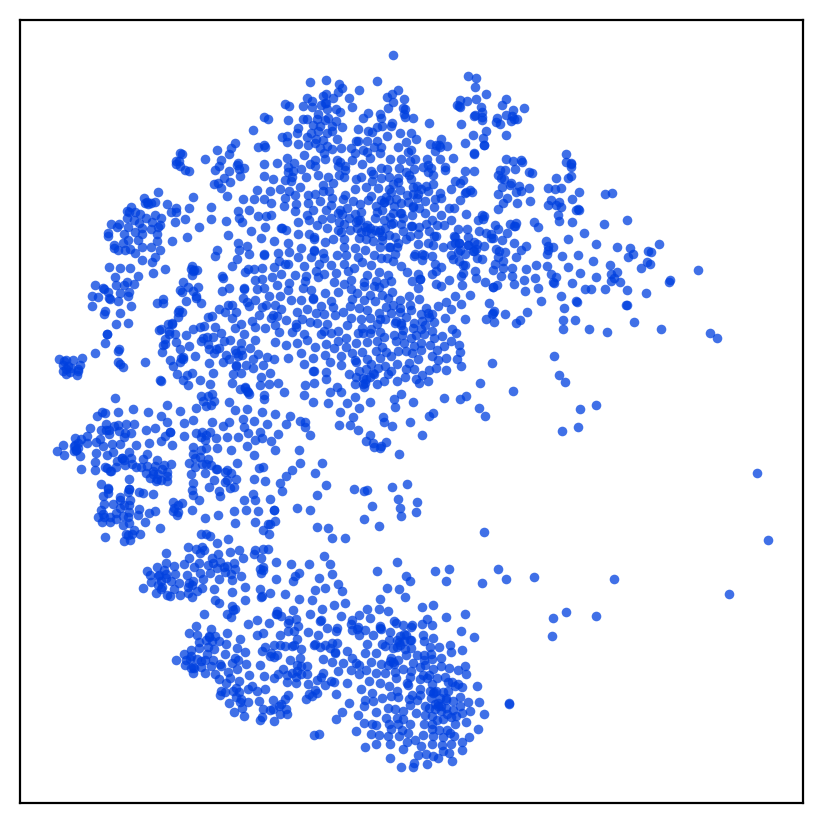}
    \end{subfigure}
    \hfill
    \begin{subfigure}[b]{0.18\textwidth}
        \centering
        \includegraphics[width=\textwidth]{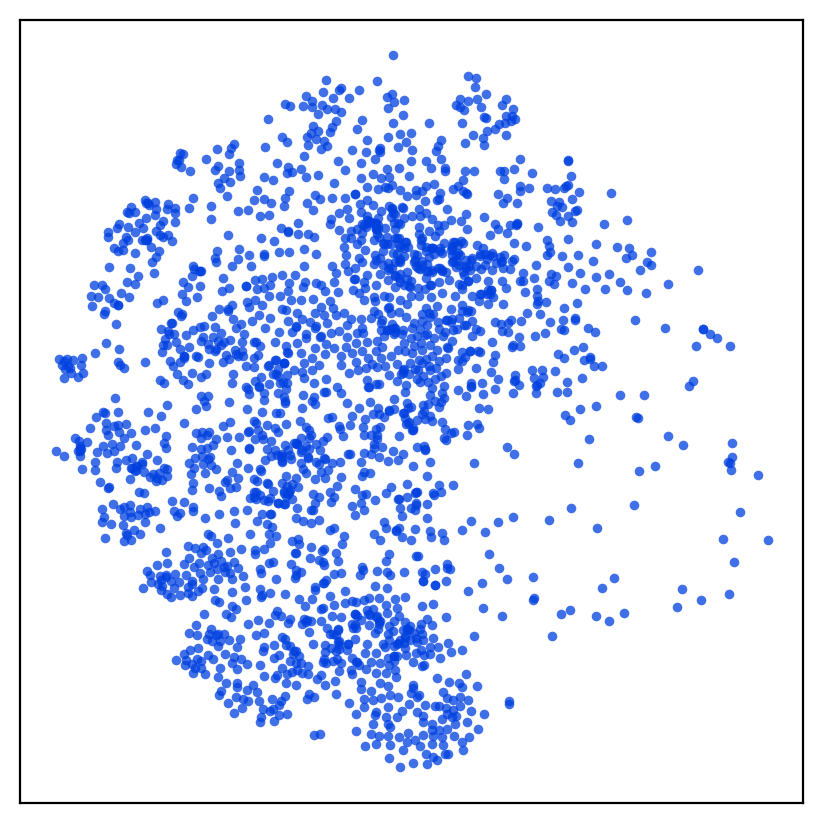}
    \end{subfigure}
    \hfill
    \begin{subfigure}[b]{0.18\textwidth}
        \centering
        \includegraphics[width=\textwidth]{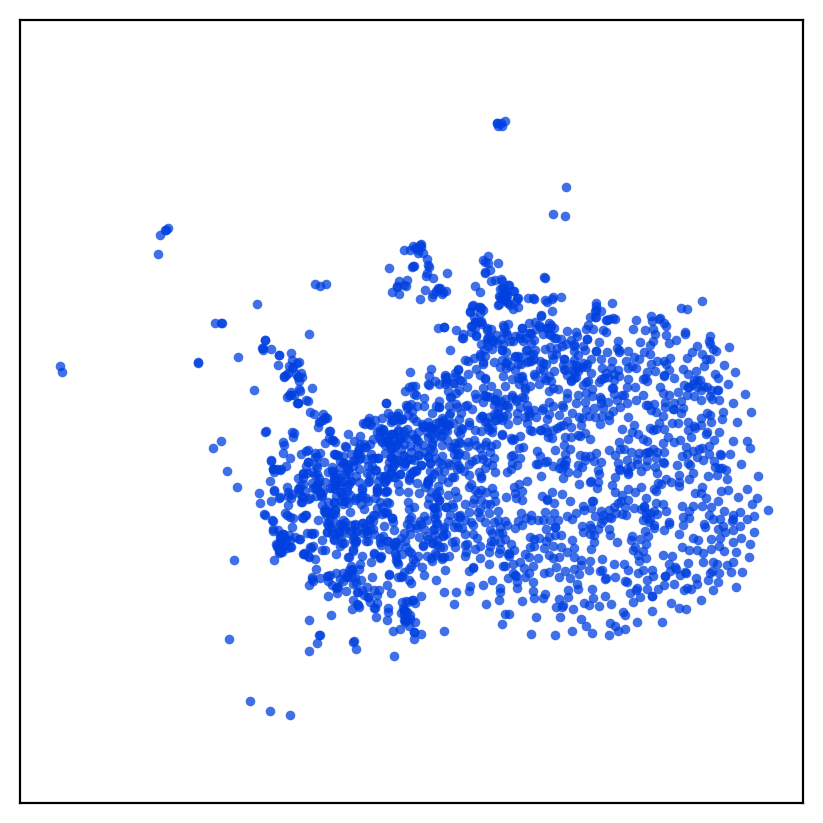}
    \end{subfigure}
    \hfill
    \begin{subfigure}[b]{0.18\textwidth}
        \centering
        \includegraphics[width=\textwidth]{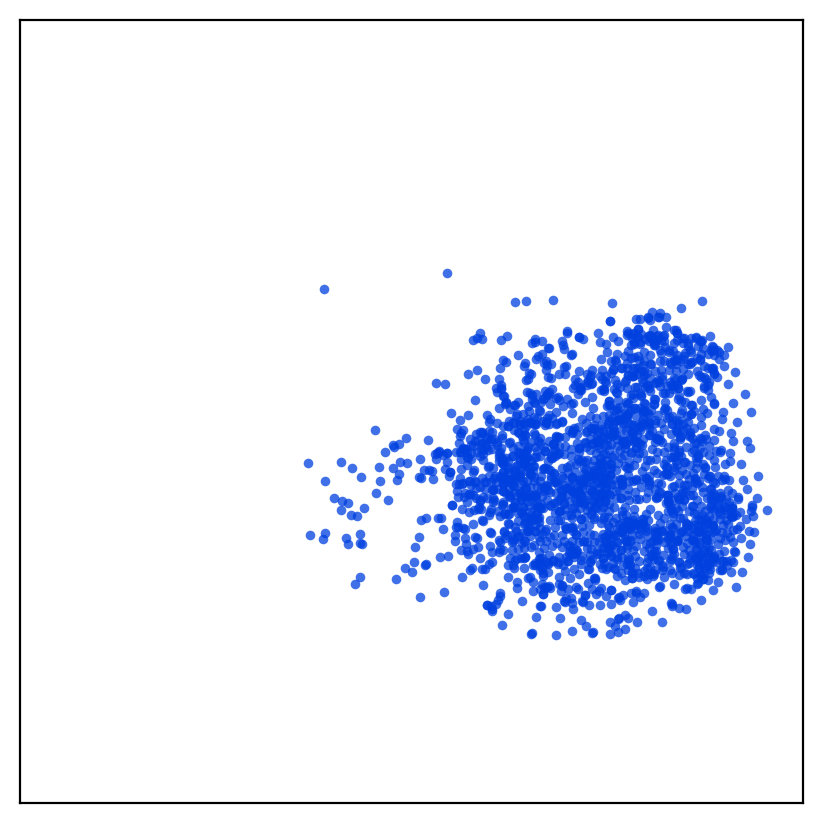}
    \end{subfigure}

    \begin{subfigure}[b]{0.18\textwidth}
        \centering
        \includegraphics[width=\textwidth]{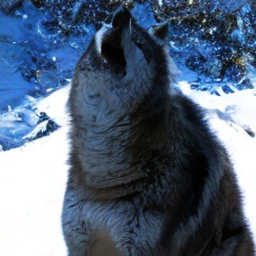}
        \caption*{$\alpha=0.0$}
    \end{subfigure}
    \hfill
    \begin{subfigure}[b]{0.18\textwidth}
        \centering
        \includegraphics[width=\textwidth]{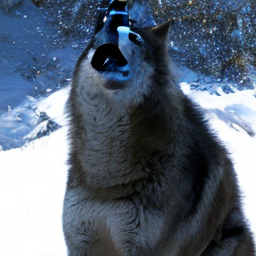}
        \caption*{$\alpha=0.1$}
    \end{subfigure}
    \hfill
    \begin{subfigure}[b]{0.18\textwidth}
        \centering
        \includegraphics[width=\textwidth]{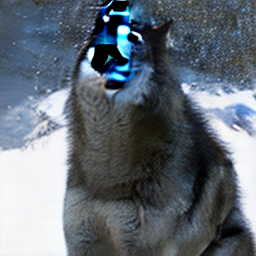}
        \caption*{$\alpha=0.2$}
    \end{subfigure}
    \hfill
    \begin{subfigure}[b]{0.18\textwidth}
        \centering
        \includegraphics[width=\textwidth]{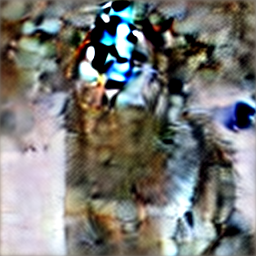}
        \caption*{$\alpha=0.6$}
    \end{subfigure}
    \hfill
    \begin{subfigure}[b]{0.18\textwidth}
        \centering
        \includegraphics[width=\textwidth]{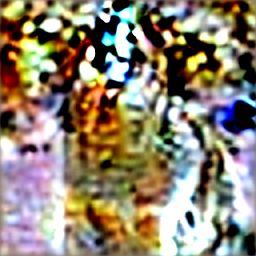}
        \caption*{$\alpha=1.0$}
    \end{subfigure}
    \caption{
    \textbf{GMOs amplify the ``mode-seeking'' nature of model outputs and their errors.}
    Here we consider a collection of $2048$ GMOs for $\alpha$ values $0.0$, $0.1$, $0.2$, $0.6$, and $1.0$ generated using AlphaFlow~\citep{zhangAlphaFlowUnderstandingImproving2026} SiT-XL/2~\citep{maSiTExploringFlow2024} model trained on ImageNet256~\citep{krizhevskyImageNetClassificationDeep2012}.
    In the top row, we visualize a t-SNE projection~\citep{maatenVisualizingDataUsing2008} of the CLIP embeddings~\citep{radford2021learning} of the GMOs.
    In the bottom row, we visualize the GMOs for a specific latent vector.
    }
    \label{fig:mode_seeking}
\end{figure}

\section{The Effect of GMOs on the Compute and Weight-merging of Neon}\label{sec:appendix_compute}

The performance of the Neon self-training algorithm is dependent on the amount of compute used for finetuning $\mathcal{B}$ and the weight-merging parameter $w$.
In \Cref{fig:compute_weight_merging}, we explore how GMOs affect these hyperparameters for a UNet generative model~\citep{ronnebergerUNetConvolutionalNetworks2015} training on CIFAR10~\citep{krizhevskyLearningMultipleLayers2009} using the MeanFlow framework~\citep{gengMeanFlowsOnestep2025}.
We observe that as the value of $\alpha$ increases, the optimal values for $\mathcal{B}$ and $w$ decrease.
This supports the claim that GMOs improve Neon's effectiveness: the more directed degradation induced by GMOs means the optimal checkpoint is reached with less finetuning.

To quantify the complete cost of the pipeline, we follow Neon's analysis~\citep{alemohammadNeonNegativeExtrapolation2026} and report the total cost as a percentage of the model's pretraining compute. Unlike the finetuning budgets in \Cref{tab:sota}, the totals in \Cref{tab:compute} also include the upfront cost of GMO generation.
Across all models, the complete pipeline costs between $0.01\%$ and $1.1\%$ of pretraining compute.

\begin{table}[ht]
    \centering
    \caption{\textbf{Total pipeline cost as a percentage of pretraining compute, including the upfront cost of GMO generation.}}
    \label{tab:compute}
    \vspace{0.5em}
    \begin{tabular}{lcc}
        \toprule
        Model & Standard Neon (\% of pretraining) & Neon+GMOs (\% of pretraining) \\
        \midrule
        IMM               & $0.007\%$ & $0.01\%$ \\
        AlphaFlow-XL/2    & $0.10\%$  & $0.51\%$ \\
        AlphaFlow-B/2     & $0.16\%$  & $0.57\%$ \\
        MeanFlow-B/2      & $0.26\%$  & $0.61\%$ \\
        MeanFlow-L/2      & $0.63\%$  & $1.06\%$ \\
        \bottomrule
    \end{tabular}
\end{table}

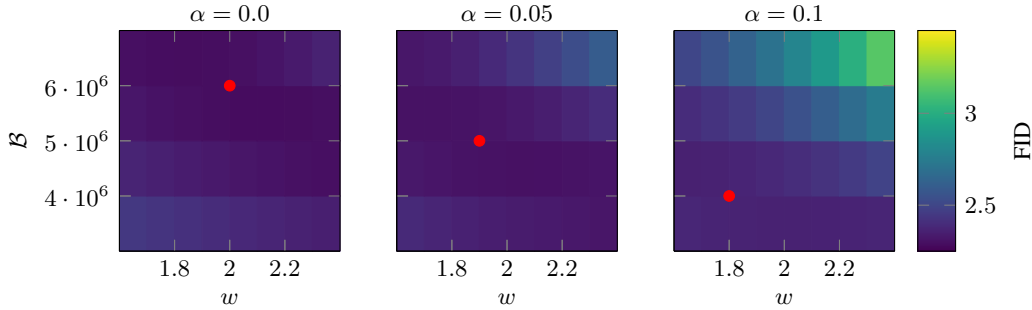
\begin{figure}[ht]
    \centering
    \begin{tikzpicture}
        \begin{groupplot}[
            group style={
                group size=3 by 1,
                horizontal sep=0.75cm,
            },
            width=4.5cm,
            height=4.5cm,
            xlabel={$w$},
            title style={align=center},
            colormap/viridis,
            view={0}{90},
            enlargelimits=false,
            scaled y ticks=false,
            y tick label style={/pgf/number format/sci, /pgf/number format/precision=1},
            zmin=2.25,
            zmax=3.45, 
            point meta min=2.25,
            point meta max=3.45,
            ytick={4000000,5000000,6000000},
            xtick={1.8,2.0,2.2},
            tick label style={font=\footnotesize},
            label style={font=\small},
            title style={font=\footnotesize},
        ]
        \nextgroupplot[
            ylabel={$\mathcal{B}$},
            title={$\alpha=0.0$},
            title style={yshift=-6pt}
        ]
        \addplot3[
            surf,
            shader=flat,
            mesh/cols=9,
        ] table [x=w, y=compute, z=fid, col sep=comma] {data/alpha_0p0.csv};
        \addplot3[only marks, mark=*, red, mark size=2pt] coordinates {(2.0, 6000000, 3)};
        
        \nextgroupplot[
            yticklabels={},
            title={$\alpha=0.05$},
            title style={yshift=-6pt}
        ]
        \addplot3[
            surf,
            shader=flat,
            mesh/cols=9,
        ] table [x=w, y=compute, z=fid, col sep=comma] {data/alpha_0p05.csv};
        \addplot3[only marks, mark=*, red, mark size=2pt] coordinates {(1.9, 5000000, 3)};
        
        \nextgroupplot[
            colorbar,
            colorbar style={
                ylabel={FID},
                yticklabel style={/pgf/number format/fixed, /pgf/number format/precision=1}
            },
            yticklabels={},
            title={$\alpha=0.1$},
            title style={yshift=-6pt}
        ]
        \addplot3[
            surf,
            shader=flat,
            mesh/cols=9,
        ] table [x=w, y=compute, z=fid, col sep=comma] {data/alpha_0p1.csv};
        \addplot3[only marks, mark=*, red, mark size=2pt] coordinates {(1.8, 4000000, 3)};

        \end{groupplot}
    \end{tikzpicture}
    \caption{
    \textbf{GMOs can reduce the finetuning required to achieve optimal results using Neon.}
    Here we apply Neon to a UNet generative model~\citep{ronnebergerUNetConvolutionalNetworks2015} training on CIFAR10~\citep{krizhevskyLearningMultipleLayers2009} using the MeanFlow framework~\citep{gengMeanFlowsOnestep2025}.
    We monitor the value of FID for a range of compute levels $\mathcal{B}$ and weight-merging parameters $w$.
    With red markers, we indicate which combination of these hyperparameters yields the model with the lowest FID score.
    }
    \label{fig:compute_weight_merging}
\end{figure}

\section{Experimental Details}\label{sec:experimental_details}

\subsection{\texorpdfstring{\Cref{fig:transfer}}{Transfer experiments}}\label{sec:exp_details-transfer}

To obtain the result of the right panel of Figure \ref{fig:transfer}, we generate $10,000$ GMOs from 1-step and 2-step IMM~\citep{zhouInductiveMomentMatching2025} models trained on CIFAR10~\citep{krizhevskyLearningMultipleLayers2009} at an $\alpha$ value of $0.1$.
This process takes approximately 6 hours using 8 NVIDIA A100-SXM4-80GB GPUs.
The 2-step base model is then finetuned using a computational budget of $2.56\times10^5$.
For reference, a computational budget $16,000$ times larger was used to pre-train the base checkpoint~\citep{zhouInductiveMomentMatching2025}.

Throughout the finetuning we take 8 checkpoints and evaluate each using weight-merging parameters in the range $[0,1]$ along a $0.1$-step interval.
In the right panel of Figure \ref{fig:transfer}, for each finetuning dataset, we visualize the evaluation curves that yielded the lowest FID scores.
When using standard outputs from the 2-step model, the optimal FID was seen at a computational budget of $1.6\times10^5$.
When using GMOs from the 2-step model, the optimal FID was seen at a computational budget of $9.6\times10^4$.
When using GMOs from the 1-step model, the optimal FID was seen at a computational budget of $1.28\times10^5$.
For completeness, in Figure \ref{fig:transfer_complete}, we provide the evaluation curves for finetuning checkpoint.

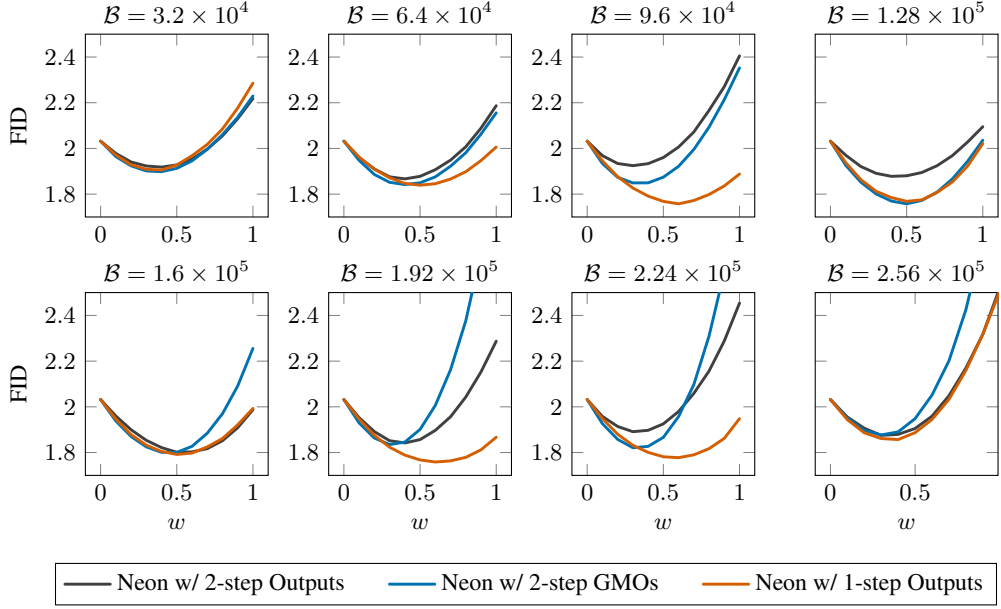
\begin{figure}[ht]
    \centering
    \begin{tikzpicture}
    \begin{groupplot}[
        group style={
            group size=4 by 2,
            horizontal sep=0.8cm,
        },
        width=4cm,
        height=4cm,
        tick label style={font=\footnotesize},
        label style={font=\small},
        title style={font=\footnotesize},
        ymin=1.7,
        ymax=2.5,
        every axis plot/.append style={
            line width=1.1pt,
            mark=none,
            line cap=round,
            line join=round
        }
    ]
    \nextgroupplot[
        ylabel={FID},
        title={\footnotesize $\mathcal{B}=3.2\times10^4$},
        title style={yshift=-6pt},
        legend to name=sharedlegend,
        legend columns=3,
        legend style={
            font=\footnotesize,
            /tikz/every even column/.append style={column sep=0.5cm},
        }
    ]
    \addplot[color=class1] table [x=w, y=vanilla_2step, col sep=comma] {data/imm_transfer_results_epoch_250.csv};
    \addlegendentry{Neon w/ 2-step Outputs}
    \addplot[color=class2] table [x=w, y=spectral_1step_a0.01, col sep=comma] {data/imm_transfer_results_epoch_250.csv};
    \addlegendentry{Neon w/ 2-step GMOs}
    \addplot[color=class3] table [x=w, y=spectral_2step_a0.01, col sep=comma] {data/imm_transfer_results_epoch_250.csv};
    \addlegendentry{Neon w/ 1-step Outputs}
    
    \nextgroupplot[
        title={$\mathcal{B}=6.4\times10^4$},
        title style={yshift=-6pt}
    ]
    \addplot[color=class1] table [x=w, y=vanilla_2step, col sep=comma] {data/imm_transfer_results_epoch_500.csv};
    \addplot[color=class2] table [x=w, y=spectral_1step_a0.01, col sep=comma] {data/imm_transfer_results_epoch_500.csv};
    \addplot[color=class3] table [x=w, y=spectral_2step_a0.01, col sep=comma] {data/imm_transfer_results_epoch_500.csv};

    \nextgroupplot[
        title={$\mathcal{B}=9.6\times10^4$},
        title style={yshift=-6pt}
    ]
    \addplot[color=class1] table [x=w, y=vanilla_2step, col sep=comma] {data/imm_transfer_results_epoch_750.csv};
    \addplot[color=class2] table [x=w, y=spectral_1step_a0.01, col sep=comma] {data/imm_transfer_results_epoch_750.csv};
    \addplot[color=class3] table [x=w, y=spectral_2step_a0.01, col sep=comma] {data/imm_transfer_results_epoch_750.csv};

    \nextgroupplot[
        title={$\mathcal{B}=1.28\times10^5$},
        title style={yshift=-6pt},
        ylabel={\phantom{FID}},
        yticklabel pos=right,
        yticklabels={\phantom{2}}
    ]
    \addplot[color=class1] table [x=w, y=vanilla_2step, col sep=comma] {data/imm_transfer_results_epoch_1000.csv};
    \addplot[color=class2] table [x=w, y=spectral_1step_a0.01, col sep=comma] {data/imm_transfer_results_epoch_1000.csv};
    \addplot[color=class3] table [x=w, y=spectral_2step_a0.01, col sep=comma] {data/imm_transfer_results_epoch_1000.csv};

    \nextgroupplot[
        ylabel={FID},
        xlabel={$w$},
        title={$\mathcal{B}=1.6\times10^5$},
        title style={yshift=-6pt}
    ]
    \addplot[color=class1] table [x=w, y=vanilla_2step, col sep=comma] {data/imm_transfer_results_epoch_1250.csv};
    \addplot[color=class2] table [x=w, y=spectral_1step_a0.01, col sep=comma] {data/imm_transfer_results_epoch_1250.csv};
    \addplot[color=class3] table [x=w, y=spectral_2step_a0.01, col sep=comma] {data/imm_transfer_results_epoch_1250.csv};
    
    \nextgroupplot[
        xlabel={$w$},
        title={$\mathcal{B}=1.92\times10^5$},
        title style={yshift=-6pt}
    ]
    \addplot[color=class1] table [x=w, y=vanilla_2step, col sep=comma] {data/imm_transfer_results_epoch_1500.csv};
    \addplot[color=class2] table [x=w, y=spectral_1step_a0.01, col sep=comma] {data/imm_transfer_results_epoch_1500.csv};
    \addplot[color=class3] table [x=w, y=spectral_2step_a0.01, col sep=comma] {data/imm_transfer_results_epoch_1500.csv};

    \nextgroupplot[
        xlabel={$w$},
        title={$\mathcal{B}=2.24\times10^5$},
        title style={yshift=-6pt}
    ]
    \addplot[color=class1] table [x=w, y=vanilla_2step, col sep=comma] {data/imm_transfer_results_epoch_1750.csv};
    \addplot[color=class2] table [x=w, y=spectral_1step_a0.01, col sep=comma] {data/imm_transfer_results_epoch_1750.csv};
    \addplot[color=class3] table [x=w, y=spectral_2step_a0.01, col sep=comma] {data/imm_transfer_results_epoch_1750.csv};

    \nextgroupplot[
        xlabel={$w$},
        title={$\mathcal{B}=2.56\times10^5$},
        title style={yshift=-6pt},
        ylabel={\phantom{FID}},
        yticklabel pos=right,
        yticklabels={\phantom{2}}
    ]
    \addplot[color=class1] table [x=w, y=vanilla_2step, col sep=comma] {data/imm_transfer_results_epoch_2000.csv};
    \addplot[color=class2] table [x=w, y=spectral_1step_a0.01, col sep=comma] {data/imm_transfer_results_epoch_2000.csv};
    \addplot[color=class3] table [x=w, y=spectral_2step_a0.01, col sep=comma] {data/imm_transfer_results_epoch_2000.csv};
    
    \end{groupplot}
    \end{tikzpicture}

    \vspace{0.3cm}
    
    \ref{sharedlegend}
    
    \caption{
    The complete set of evaluation curves for the experiment shown in the right panel of Figure \ref{fig:transfer}, and described in Appendix \ref{sec:exp_details-transfer}.
    }
    \label{fig:transfer_complete}
\end{figure}

\subsection{\texorpdfstring{\Cref{tab:sota}}{State-of-the-art experiments}}\label{sec:exp_details-sota}

Here we detail the hyperparameter sweeps that were used to obtain the results of \Cref{tab:sota}.
The following computation times are based on 8 NVIDIA RTX A6000 GPUs.

\begin{itemize}
    \item \textbf{MeanFlow SiT-B/2:} We generate $30,000$ standard outputs and GMOs for $\alpha$ values $0.05$ and $0.1$.
    This takes approximately 4 hours.
    The base model is finetuned for 48 epochs on each of these data sets, and we evaluate 6 equally spaced checkpoints. 
    Each round of finetuning takes 1 hour.
    For each checkpoint, we evaluate the FID using $50,000$ samples and $w$ values ranging from $0.0$ to $1.5$ at $ 0.1$-step increments.
    Each FID computation takes approximately 2 minutes.
    \item \textbf{MeanFlow SiT-L/2:}
    We generate $30,000$ standard outputs and GMOs for $\alpha$ values $0.1$ and $0.2$.
    This takes approximately 7 hours.
    The base model is finetuned for 60 epochs on each of these data sets, and we evaluate 6 equally spaced checkpoints.
    Each round of finetuning takes 3 hours.
    For each checkpoint, we evaluate the FID using $50,000$ samples and $w$ values ranging from $0.0$ to $1.5$ at $0.1$-step increments.
    Each FID computation takes approximately 3 minutes.
    \item \textbf{AlphaFlow SiT-B/2:} We generate $30,000$ standard outputs and GMOs for $\alpha$ values $0.1$ and $0.2$.
    This takes approximately 4 hours.
    The base model is finetuned for 32 epochs on each of these data sets, and we take checkpoints every 4 epochs.
    Each round of finetuning takes less than 1 hour.
    For each checkpoint, we evaluate the FID using $50,000$ samples and $w$ values ranging from $0.0$ to $1.5$ at $ 0.1$-step increments.
    Each FID computation takes approximately 2 minutes.
    \item \textbf{AlphaFlow SiT-XL/2:}
    We generate $30,000$ standard outputs and GMOs for $\alpha$ values $0.05$ and $0.1$.
    This takes approximately 10 hours.
    The base model is finetuned for 15 epochs on each of these data sets, and we take checkpoints every 5 epochs.
    Each round of finetuning takes less than 1 hour.
    For each checkpoint, we evaluate the FID using $50,000$ samples and $w$ values ranging from $0.0$ to $1.6$ at $ 0.1$-step increments.
    Each FID computation takes approximately 4 minutes.
    \item \textbf{IMM (DiT-XL/2):}
    We generate $30{,}000$ standard outputs and GMOs for $\alpha$ values of $0.1$ and $0.2$.
    This takes approximately 5 hours.
    The base model is finetuned for $109$ epochs on each dataset, and we evaluate 8 equally spaced checkpoints. Each round of finetuning takes approximately 5 hours. For each checkpoint, we evaluate FID using $50{,}000$ samples and $w$ values ranging from $0.0$ to $1.8$ in increments of $0.2$. Each FID computation takes approximately 10 minutes.
    
\end{itemize}

In \Cref{fig:hyperparameter_sweeps}, we consider each individual model as a row, and show in the left panel what the minimum FID value is for the sweep over $w$ at each compute level.
In the right panel, we show the value of $w$ for which the minimum FID was obtained at each compute level.

\subsection{SIMS-style guidance}
\label{sec:exp_details-sims}

Here we detail the experiments in \Cref{sec:sims}.
The following computation times are based on
8 NVIDIA RTX A6000 GPUs.

\begin{itemize}
    \item \textbf{IMM (DiT-XL/2):}
    We use synthetic datasets containing $50{,}000$
    standard outputs or GMOs with $\alpha$ values $0.1$
    and $0.2$.
    For the main sweep, we evaluate 5 checkpoints for
    standard outputs and 8 checkpoints for GMOs with
    $\alpha=0.1$, using guidance strengths
    $\omega\in\{0.5,1.0,1.6\}$.
    The auxiliary models used for the reported comparison
    are finetuned for approximately $31$ epochs on their
    respective datasets, requiring approximately
    $2.5$ hours per model.
    Each FID evaluation uses $50{,}000$ samples,
    one-step sampling, and classifier-free guidance
    scale $1.5$, and takes approximately $6$ minutes.
    We additionally evaluate $\alpha=0.2$ using
    $\omega\in\{0.5,1.0,1.6,2.4\}$; its lowest observed
    FID is with GMOs at $7.535$ 
\end{itemize}

\begin{table}[ht]
    \centering
    \caption{SIMS-style guidance on IMM ImageNet256.
    FID is reported as mean $\pm$ sample standard deviation
    across five evaluations of fixed checkpoints,
    using $50{,}000$ samples per evaluation.
    Both guided variants use $\omega=1.6$.}
    \label{tab:sims_confirmation}
    \begin{tabular}{lc}
        \toprule
        Method & FID $\downarrow$ \\
        \midrule
        Base IMM & $8.342 \pm 0.058$ \\
        SIMS, standard outputs & $7.591 \pm 0.057$ \\
        SIMS, GMOs ($\alpha=0.1$)
            & $\mathbf{6.920 \pm 0.048}$ \\
        \bottomrule
    \end{tabular}
\end{table}

\Cref{tab:prdc} reports precision, recall, density, and coverage for all architectures of \Cref{tab:sota} (five seeds, mean $\pm$ std).
Density and coverage, the outlier-robust fidelity and diversity measures~\citep{naeemReliableFidelityDiversity2020}, are maintained or improved from standard-outputs Neon to GMOs on every architecture, showing that GMOs improve image quality and diversity simultaneously rather than trading one for the other (see \Cref{sec:fid_precision_recall}).
\Cref{tab:significance} reports the corresponding significance analysis for the FID results of \Cref{tab:sota}.

\begin{table}[ht]
    \centering
    \caption{\textbf{Precision, recall, density, and coverage (5 seeds, mean $\pm$ std).}
    Density and coverage are maintained or improved from standard-outputs Neon to GMOs on every architecture.}
    \label{tab:prdc}
    \vspace{0.5em}
    \setlength{\tabcolsep}{5.5pt}
    \begin{tabular}{llcccc}
        \toprule
        Architecture & Method & Precision & Recall & Density & Coverage \\
        \midrule
        IMM (DiT-XL/2)     & Base     & $.585\pm.002$ & $.647\pm.001$ & $.652\pm.001$ & $.662\pm.003$ \\
                           & Neon-std & $.598\pm.002$ & $.654\pm.002$ & $.695\pm.001$ & $.693\pm.002$ \\
                           & GMOs     & $.615\pm.003$ & $.647\pm.002$ & $\mathbf{.728\pm.003}$ & $\mathbf{.714\pm.002}$ \\
        \midrule
        MeanFlow SiT-B/2   & Base     & $.717\pm.002$ & $.452\pm.002$ & $1.081\pm.006$ & $.749\pm.001$ \\
                           & Neon-std & $.720\pm.001$ & $.452\pm.003$ & $1.110\pm.006$ & $.755\pm.001$ \\
                           & GMOs     & $.725\pm.001$ & $.451\pm.003$ & $\mathbf{1.124\pm.007}$ & $\mathbf{.759\pm.002}$ \\
        \midrule
        MeanFlow SiT-L/2   & Base     & $.752\pm.001$ & $.494\pm.002$ & $1.214\pm.002$ & $.828\pm.001$ \\
                           & Neon-std & $.764\pm.000$ & $.486\pm.002$ & $1.264\pm.002$ & $.837\pm.001$ \\
                           & GMOs     & $.763\pm.001$ & $.483\pm.002$ & $\mathbf{1.267\pm.002}$ & $\mathbf{.839\pm.001}$ \\
        \midrule
        AlphaFlow SiT-B/2  & Base     & $.748\pm.001$ & $.450\pm.001$ & $1.210\pm.004$ & $.794\pm.001$ \\
                           & Neon-std & $.743\pm.002$ & $.454\pm.003$ & $1.197\pm.003$ & $.791\pm.002$ \\
                           & GMOs     & $.740\pm.001$ & $.455\pm.003$ & $1.197\pm.003$ & $.791\pm.002$ \\
        \midrule
        AlphaFlow SiT-XL/2 & Base     & $.695\pm.001$ & $.596\pm.001$ & $1.005\pm.002$ & $.811\pm.001$ \\
                           & Neon-std & $.701\pm.001$ & $.593\pm.003$ & $1.025\pm.003$ & $.820\pm.002$ \\
                           & GMOs     & $.705\pm.000$ & $.589\pm.002$ & $\mathbf{1.034\pm.003}$ & $\mathbf{.821\pm.001}$ \\
        \bottomrule
    \end{tabular}
\end{table}

\begin{table}[ht]
    \centering
    \caption{\textbf{Statistical significance of the GMO improvement over standard-outputs Neon} (5 seeds, mean $\pm$ std).
    The final column reports the additional FID reduction of GMOs as a percentage of the reduction Neon itself achieves over the base model.}
    \label{tab:significance}
    \vspace{0.5em}
    \resizebox{\textwidth}{!}{%
    \begin{tabular}{lccccc}
        \toprule
        Architecture & Base FID & Neon-std FID & GMOs FID & Welch $p$ & GMOs' reduction, \% of Neon's \\
        \midrule
        IMM (DiT-XL/2)      & $8.34$ & $7.32\pm.07$ & $\mathbf{6.25\pm.03}$ & ${\sim}10^{-9}$    & $96\%$  \\
        MeanFlow SiT-B/2    & $6.08$ & $5.70\pm.01$ & $\mathbf{5.60\pm.01}$ & $2.6\times10^{-7}$ & $26\%$  \\
        MeanFlow SiT-L/2    & $3.97$ & $3.74\pm.02$ & $\mathbf{3.70\pm.02}$ & $0.013$            & $17\%$  \\
        AlphaFlow SiT-B/2   & $5.55$ & $5.36\pm.02$ & $\mathbf{5.16\pm.02}$ & $2.6\times10^{-7}$ & $105\%$ \\
        AlphaFlow SiT-XL/2  & $2.93$ & $2.64\pm.02$ & $\mathbf{2.59\pm.02}$ & $0.004$            & $17\%$  \\
        \bottomrule
    \end{tabular}}
\end{table}
% Removed an unmatched closing brace present in edit.tex.

\section{Model Licenses}\label{sec:licenses}

IMM model weights are obtained through the imm GitHub repository\footnote{\url{https://github.com/lumalabs/imm}} under CC BY-NC-SA 4.0 License.
MeanFlow model weights are obtained through the MeanFlow GitHub repository\footnote{\url{https://github.com/zhuyu-cs/MeanFlow}} under the MIT license.
AlphaFlow model weights are obtained through the AlphaFlow GitHub repository\footnote{\url{https://github.com/snap-research/alphaflow}} under the Snap Inc. Non-Commercial license.

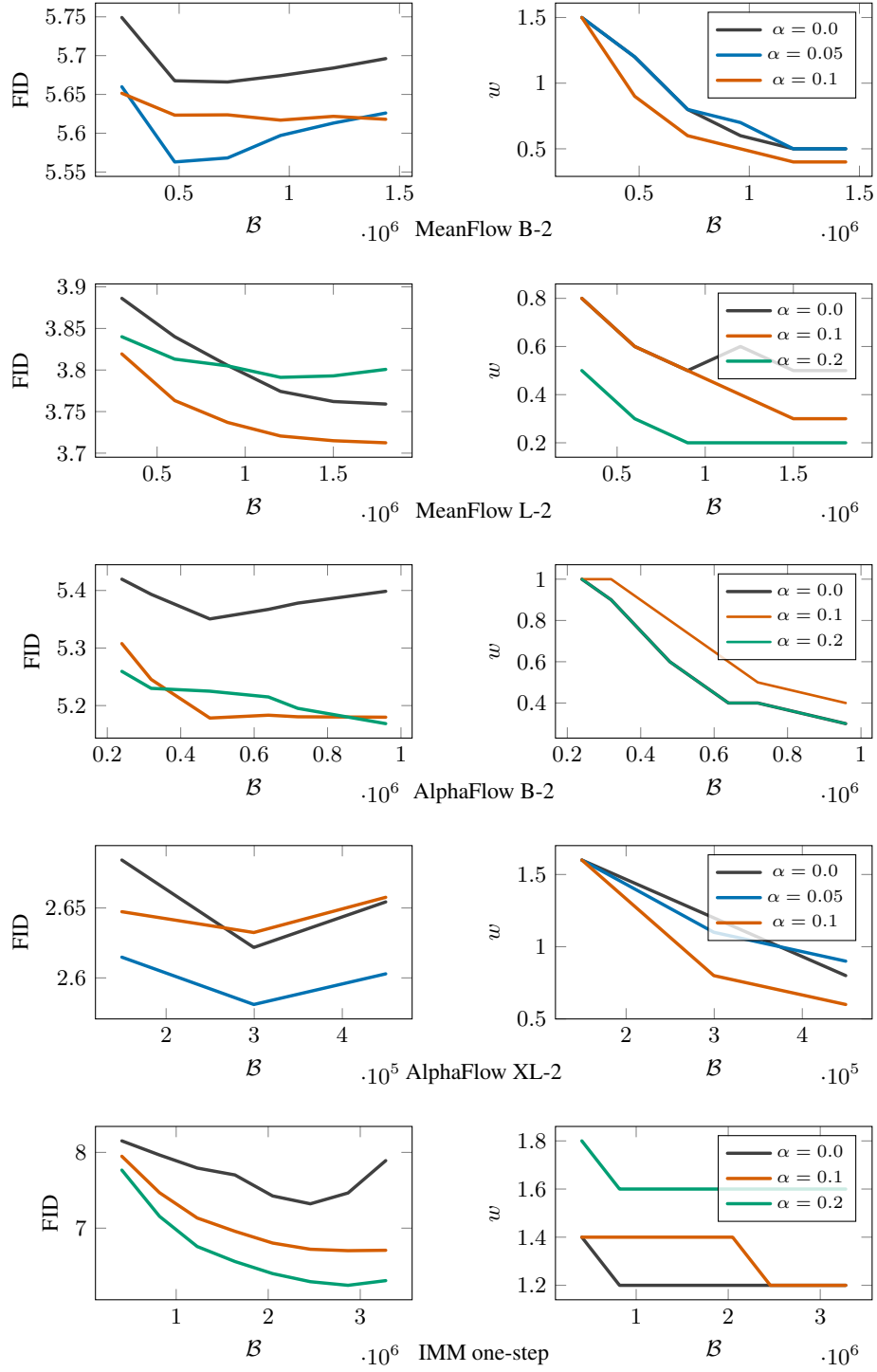
\begin{figure}[ht]
    \centering
    \begin{tikzpicture}
    \begin{groupplot}[
        group style={
            group size=2 by 5,
            horizontal sep=2cm,
            vertical sep=1.5cm,
            group name=plots
        },
        width=6cm,
        height=4.0cm,
        xlabel={$\mathcal{B}$},
        tick label style={font=\footnotesize},
        label style={font=\small},
        title style={font=\small},
        every axis plot/.append style={
            line width=1.3pt,
            mark=none,
            line cap=round,
            line join=round
        }
    ]

    % ================= ROW 1: MeanFlow B-2 =================
    \nextgroupplot[
        ylabel={FID},
    ]
    \addplot[color=class1] table [x=compute, y=0.0, col sep=comma]
        {data/meanflow-b2-sweep-fid.csv};
    \addplot[color=class2] table [x=compute, y=0.05, col sep=comma]
        {data/meanflow-b2-sweep-fid.csv};
    \addplot[color=class3] table [x=compute, y=0.1, col sep=comma]
        {data/meanflow-b2-sweep-fid.csv};

    \nextgroupplot[
        ylabel={$w$},
        legend style={
            at={(0.95,0.95)},
            anchor=north east,
            font=\scriptsize,
            fill=white,
            fill opacity=0.8,
            draw opacity=1,
            text opacity=1
        }
    ]
    \addplot[color=class1] table [x=compute, y=0.0, col sep=comma]
        {data/meanflow-b2-sweep-w.csv};
    \addlegendentry{$\alpha=0.0$}

    \addplot[color=class2] table [x=compute, y=0.05, col sep=comma]
        {data/meanflow-b2-sweep-w.csv};
    \addlegendentry{$\alpha=0.05$}

    \addplot[color=class3] table [x=compute, y=0.1, col sep=comma]
        {data/meanflow-b2-sweep-w.csv};
    \addlegendentry{$\alpha=0.1$}

    % ================= ROW 2: MeanFlow L-2 =================
    \nextgroupplot[
        ylabel={FID},
    ]
    \addplot[color=class1] table [x=compute, y=0.0, col sep=comma]
        {data/meanflow-l2-sweep-fid.csv};
    \addplot[color=class3] table [x=compute, y=0.1, col sep=comma]
        {data/meanflow-l2-sweep-fid.csv};
    \addplot[color=class4] table [x=compute, y=0.2, col sep=comma]
        {data/meanflow-l2-sweep-fid.csv};

    \nextgroupplot[
        ylabel={$w$},
        legend style={
            at={(0.95,0.95)},
            anchor=north east,
            font=\scriptsize,
            fill=white,
            fill opacity=0.8,
            draw opacity=1,
            text opacity=1
        }
    ]
    \addplot[color=class1] table [x=compute, y=0.0, col sep=comma]
        {data/meanflow-l2-sweep-w.csv};
    \addlegendentry{$\alpha=0.0$}

    \addplot[color=class3] table [x=compute, y=0.1, col sep=comma]
        {data/meanflow-l2-sweep-w.csv};
    \addlegendentry{$\alpha=0.1$}

    \addplot[color=class4] table [x=compute, y=0.2, col sep=comma]
        {data/meanflow-l2-sweep-w.csv};
    \addlegendentry{$\alpha=0.2$}

    % ================= ROW 3: AlphaFlow B-2 =================
    \nextgroupplot[
        ylabel={FID},
    ]
    \addplot[color=class1] table [x=compute, y=0.0, col sep=comma]
        {data/alphaflow-b2-sweep-fid.csv};
    \addplot[color=class3] table [x=compute, y=0.1, col sep=comma]
        {data/alphaflow-b2-sweep-fid.csv};
    \addplot[color=class4] table [x=compute, y=0.2, col sep=comma]
        {data/alphaflow-b2-sweep-fid.csv};

    \nextgroupplot[
        ylabel={$w$},
        legend style={
            at={(0.95,0.95)},
            anchor=north east,
            font=\scriptsize,
            fill=white,
            fill opacity=0.8,
            draw opacity=1,
            text opacity=1
        }
    ]
    \addplot[color=class1] table [x=compute, y=0.0, col sep=comma]
        {data/alphaflow-b2-sweep-w.csv};
    \addlegendentry{$\alpha=0.0$}

    \addplot[color=class3, line width=1pt] table [x=compute, y=0.1, col sep=comma]
        {data/alphaflow-b2-sweep-w.csv};
    \addlegendentry{$\alpha=0.1$}

    \addplot[color=class4, line width=1pt] table [x=compute, y=0.2, col sep=comma]
        {data/alphaflow-b2-sweep-w.csv};
    \addlegendentry{$\alpha=0.2$}

    % ================= ROW 4: AlphaFlow XL-2 =================
    \nextgroupplot[
        ylabel={FID},
    ]
    \addplot[color=class1] table [x=compute, y=0.0, col sep=comma]
        {data/alphaflow-xl2-sweep-fid.csv};
    \addplot[color=class2] table [x=compute, y=0.05, col sep=comma]
        {data/alphaflow-xl2-sweep-fid.csv};
    \addplot[color=class3] table [x=compute, y=0.1, col sep=comma]
        {data/alphaflow-xl2-sweep-fid.csv};

    \nextgroupplot[
        ylabel={$w$},
        legend style={
            at={(0.95,0.95)},
            anchor=north east,
            font=\scriptsize,
            fill=white,
            fill opacity=0.8,
            draw opacity=1,
            text opacity=1
        }
    ]
    \addplot[color=class1] table [x=compute, y=0.0, col sep=comma]
        {data/alphaflow-xl2-sweep-w.csv};
    \addlegendentry{$\alpha=0.0$}

    \addplot[color=class2] table [x=compute, y=0.05, col sep=comma]
        {data/alphaflow-xl2-sweep-w.csv};
    \addlegendentry{$\alpha=0.05$}

    \addplot[color=class3] table [x=compute, y=0.1, col sep=comma]
        {data/alphaflow-xl2-sweep-w.csv};
    \addlegendentry{$\alpha=0.1$}

    % ================= ROW 5: IMM 1-step =================
    \nextgroupplot[
        ylabel={FID},
    ]
    \addplot[color=class1] table [x=compute, y=0.0, col sep=comma]
        {data/imm-1step-sweep-fid.csv};
    \addplot[color=class3] table [x=compute, y=0.1, col sep=comma]
        {data/imm-1step-sweep-fid.csv};
    \addplot[color=class4] table [x=compute, y=0.2, col sep=comma]
        {data/imm-1step-sweep-fid.csv};

    \nextgroupplot[
        ylabel={$w$},
        legend style={
            at={(0.95,0.95)},
            anchor=north east,
            font=\scriptsize,
            fill=white,
            fill opacity=0.8,
            draw opacity=1,
            text opacity=1
        }
    ]
    \addplot[color=class1] table [x=compute, y=0.0, col sep=comma]
        {data/imm-1step-sweep-w.csv};
    \addlegendentry{$\alpha=0.0$}

    \addplot[color=class3] table [x=compute, y=0.1, col sep=comma]
        {data/imm-1step-sweep-w.csv};
    \addlegendentry{$\alpha=0.1$}

    \addplot[color=class4] table [x=compute, y=0.2, col sep=comma]
        {data/imm-1step-sweep-w.csv};
    \addlegendentry{$\alpha=0.2$}

    \end{groupplot}

    \path (plots c1r1.south west) -- (plots c2r1.south east)
        node[midway, below=0.5cm, font=\small] {MeanFlow B-2};
    \path (plots c1r2.south west) -- (plots c2r2.south east)
        node[midway, below=0.5cm, font=\small] {MeanFlow L-2};
    \path (plots c1r3.south west) -- (plots c2r3.south east)
        node[midway, below=0.5cm, font=\small] {AlphaFlow B-2};
    \path (plots c1r4.south west) -- (plots c2r4.south east)
        node[midway, below=0.5cm, font=\small] {AlphaFlow XL-2};
    \path (plots c1r5.south west) -- (plots c2r5.south east)
        node[midway, below=0.5cm, font=\small] {IMM one-step};
    
    \end{tikzpicture}

    \caption{
        Hyperparameter sweeps described in Appendix~\ref{sec:exp_details-sota}
        that yielded the results in \Cref{tab:sota}.
    }
    \label{fig:hyperparameter_sweeps}
\end{figure}

% =========== % More examples
\begin{figure}[ht]
    \centering

    \makebox[0.32\textwidth]{\textbf{Standard Outputs}}\hfill
    \makebox[0.32\textwidth]{\textbf{Perturbation}}\hfill
    \makebox[0.32\textwidth]{\textbf{Geometrically Modified Outputs}}

    \vspace{0.5em}

    \begin{subfigure}[b]{\textwidth}
        \centering
        \includegraphics[width=0.32\textwidth]{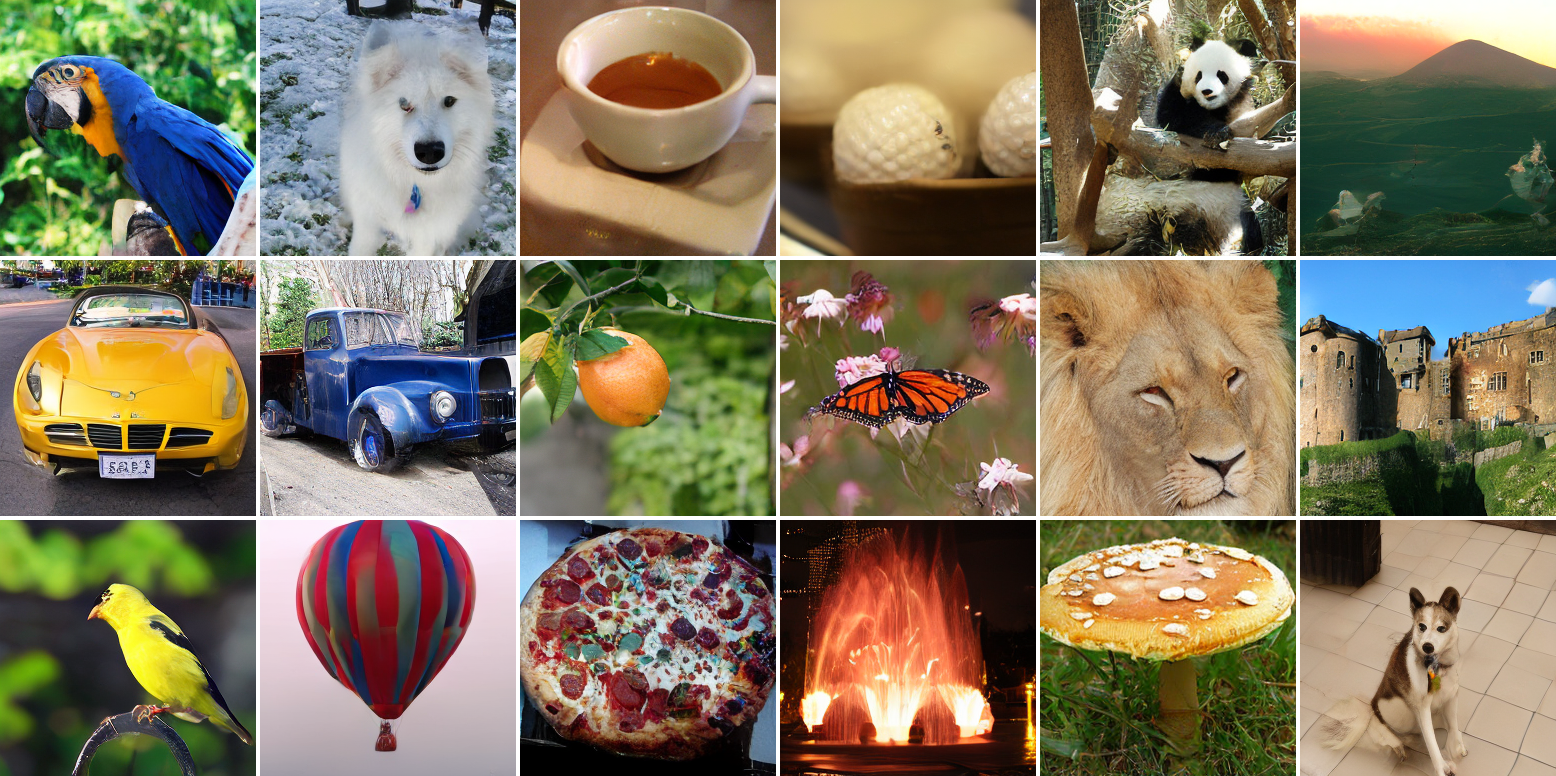}\hfill
        \includegraphics[width=0.32\textwidth]{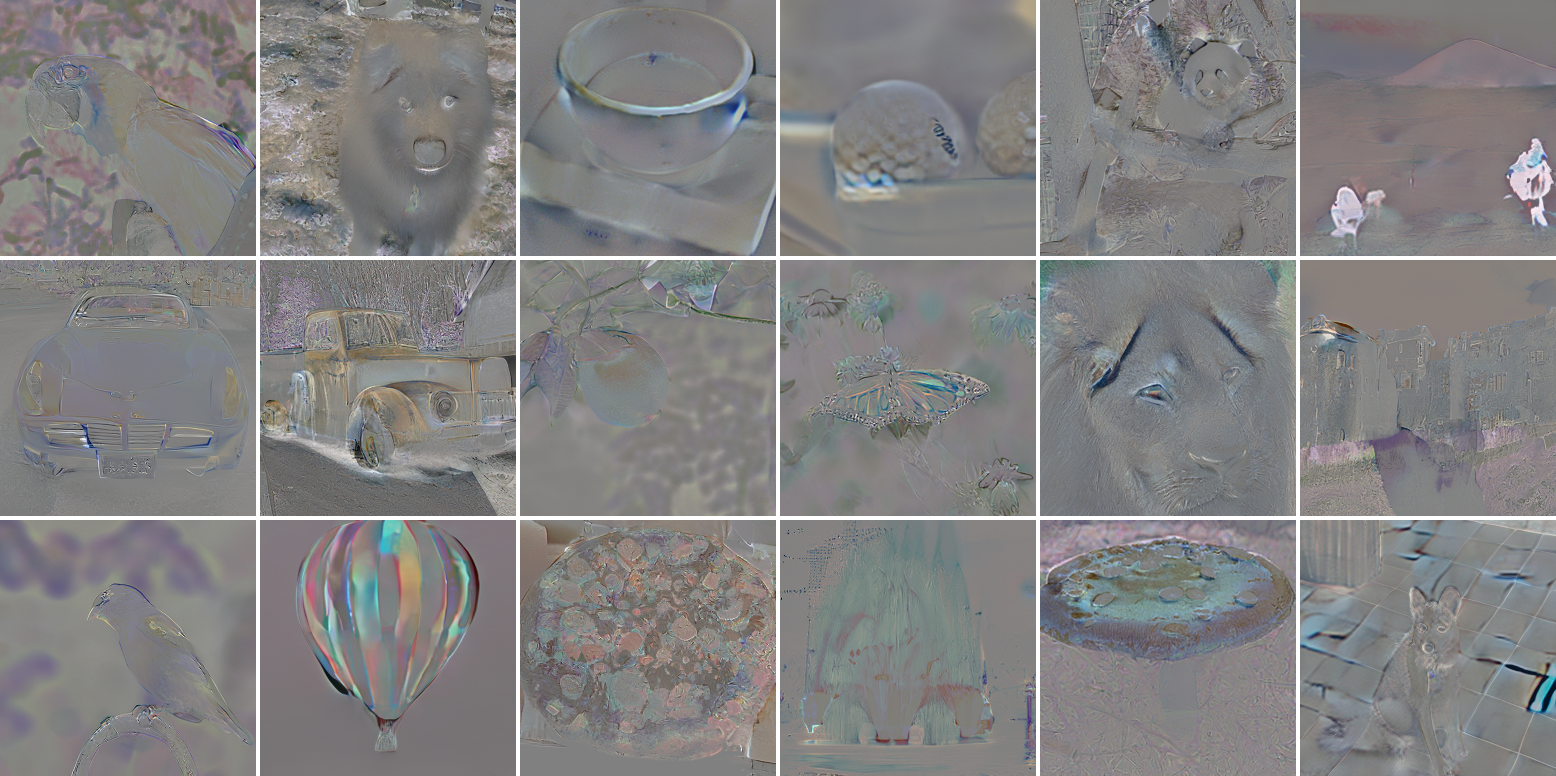}\hfill
        \includegraphics[width=0.32\textwidth]{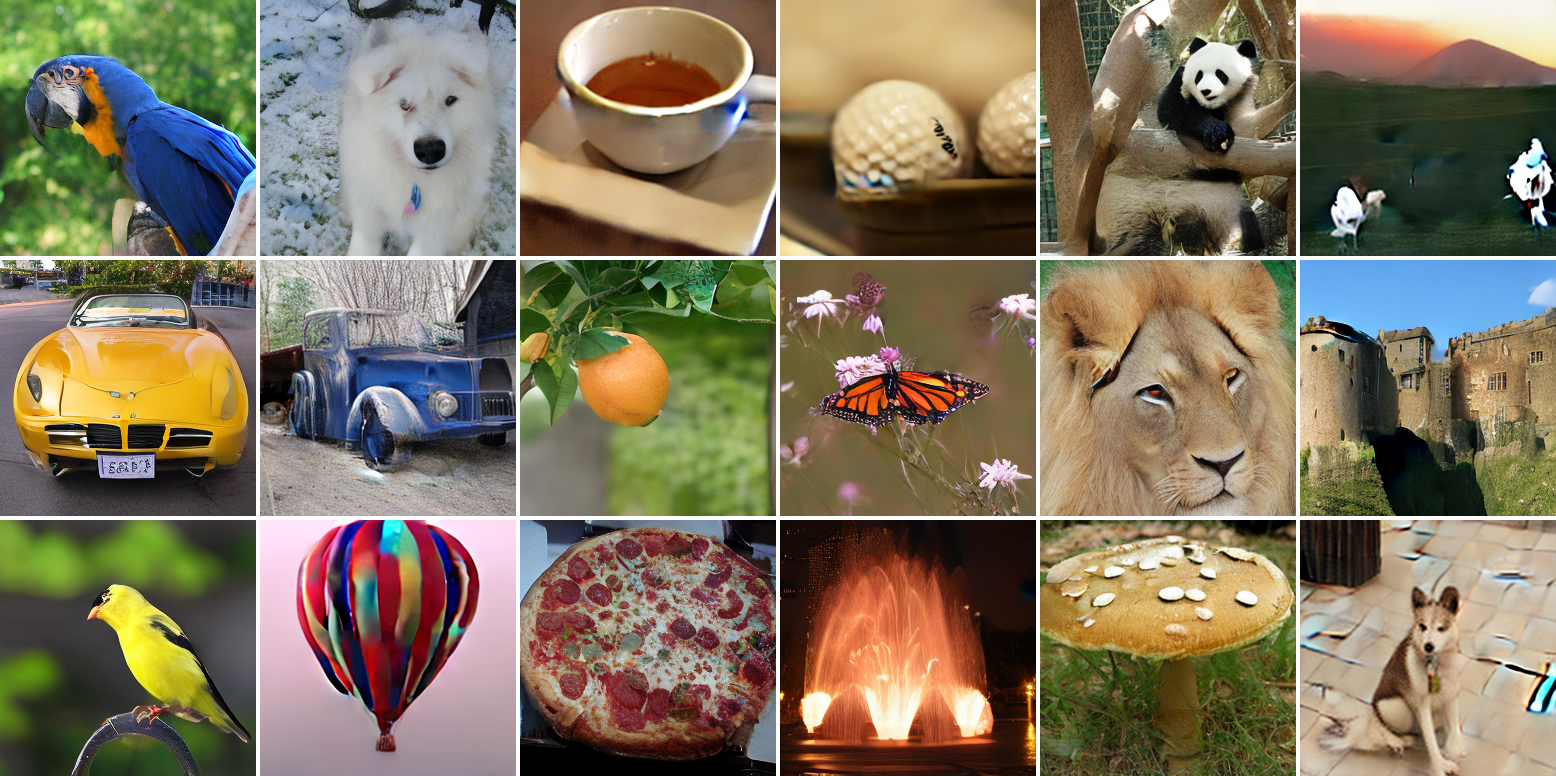}
        \caption*{$\alpha = 0.1$}
    \end{subfigure}

    \vspace{0.3em}

    \begin{subfigure}[b]{\textwidth}
        \centering
        \includegraphics[width=0.32\textwidth]{figures/per_alpha_grids/alpha_0.1/grid_standard.png}\hfill
        \includegraphics[width=0.32\textwidth]{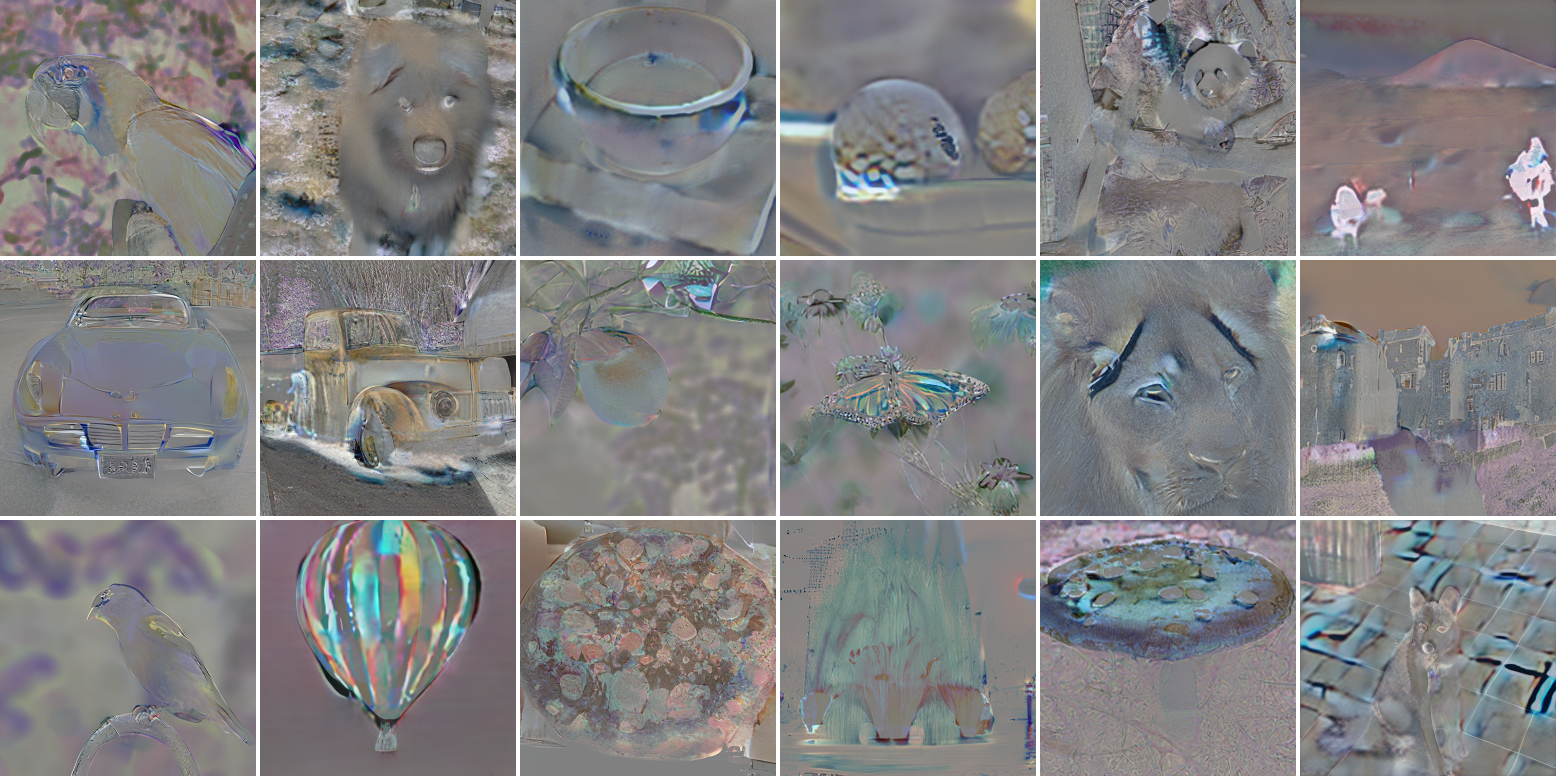}\hfill
        \includegraphics[width=0.32\textwidth]{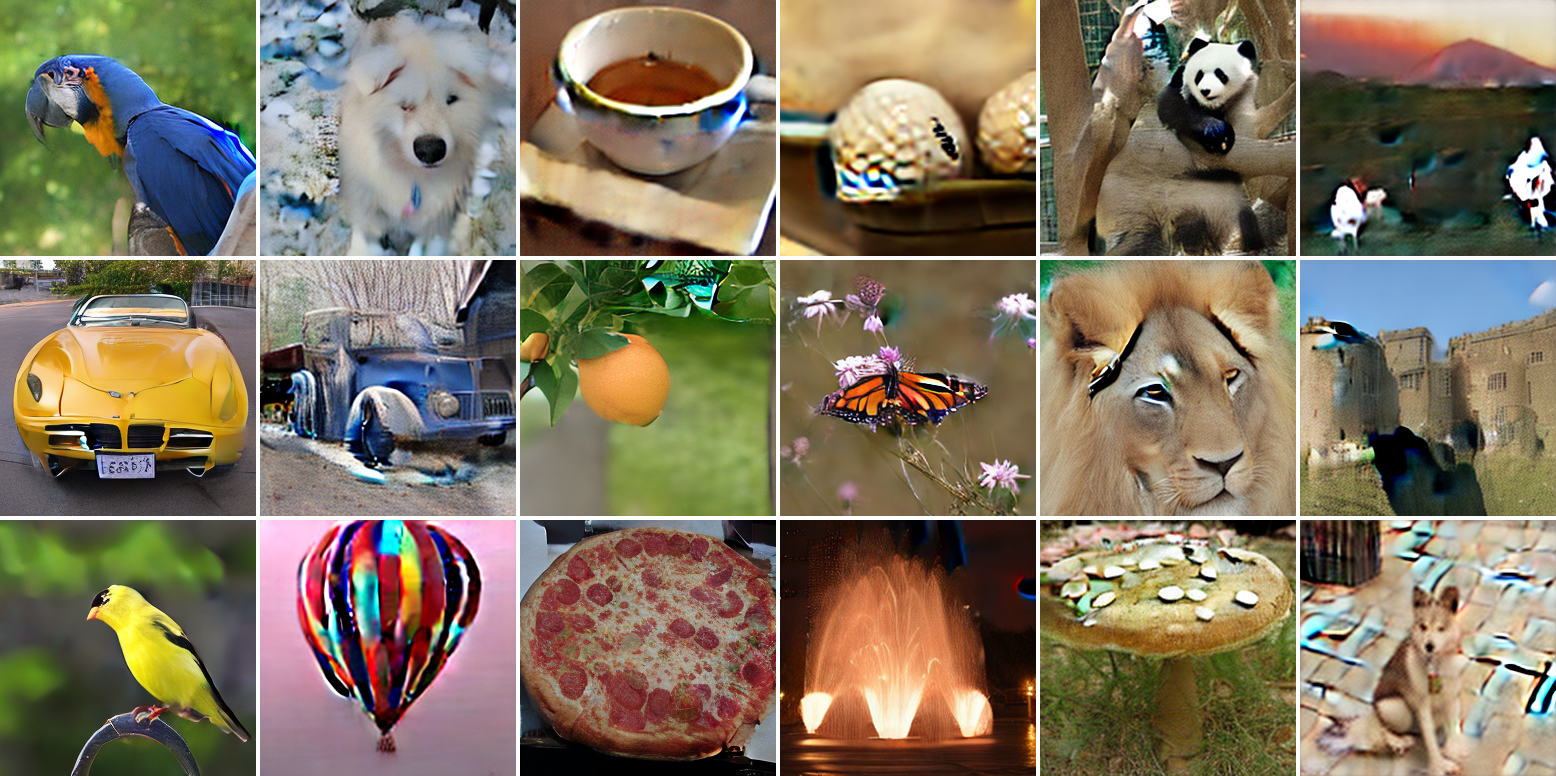}
        \caption*{$\alpha = 0.2$}
    \end{subfigure}

    \vspace{0.3em}

    \begin{subfigure}[b]{\textwidth}
        \centering
        \includegraphics[width=0.32\textwidth]{figures/per_alpha_grids/alpha_0.1/grid_standard.png}\hfill
        \includegraphics[width=0.32\textwidth]{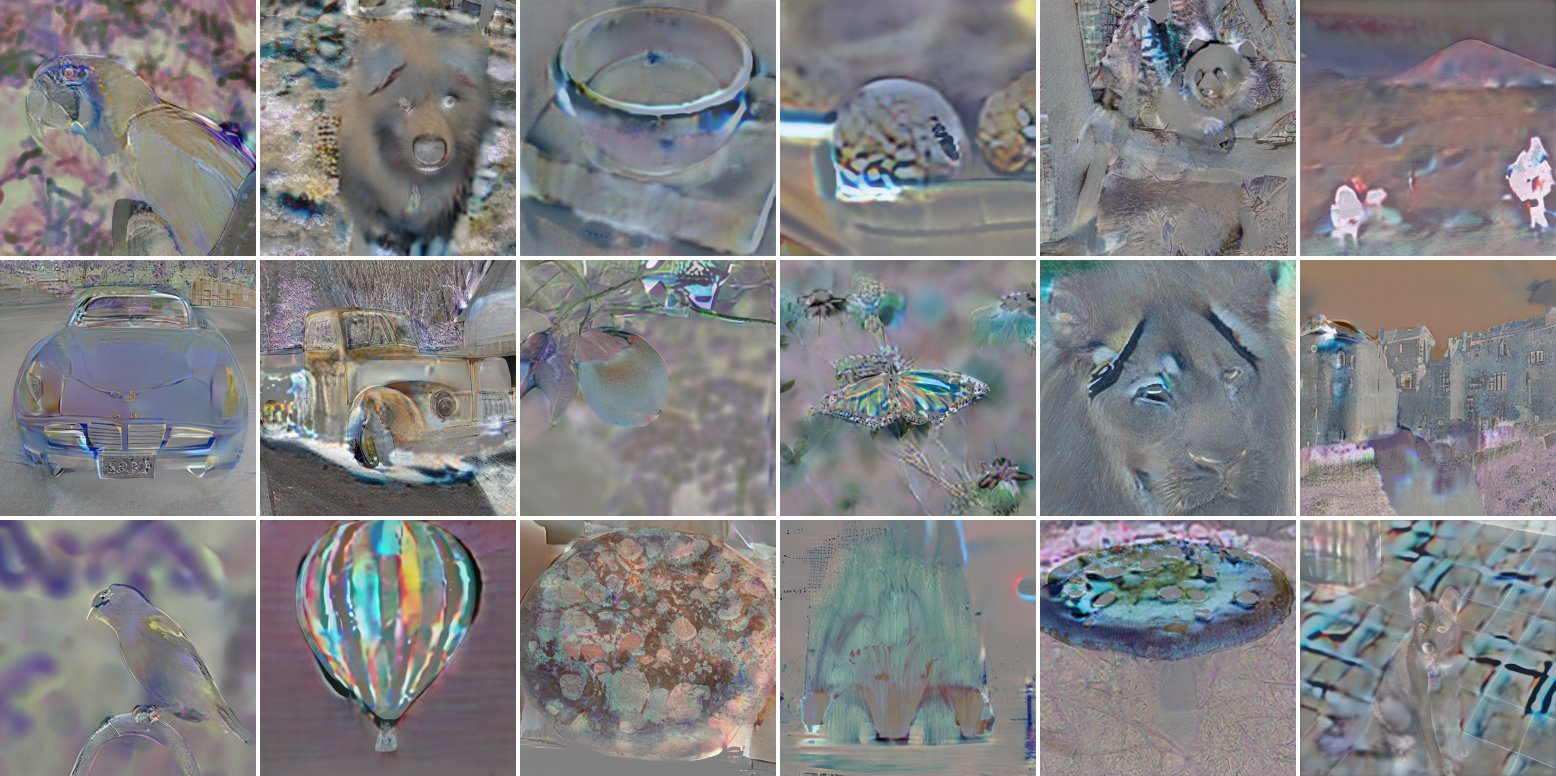}\hfill
        \includegraphics[width=0.32\textwidth]{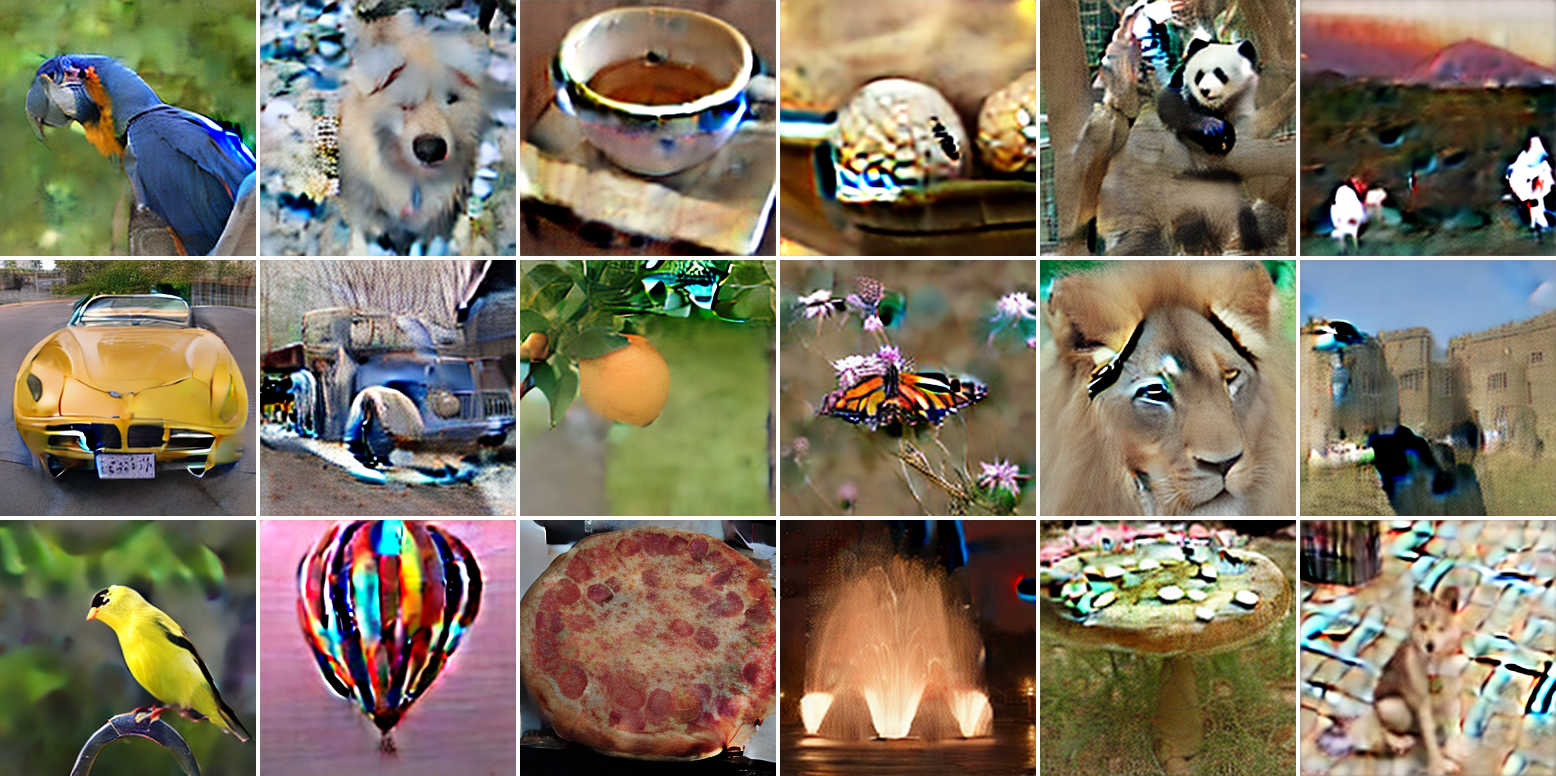}
        \caption*{$\alpha = 0.3$}
    \end{subfigure}

    \vspace{0.3em}

    \begin{subfigure}[b]{\textwidth}
        \centering
        \includegraphics[width=0.32\textwidth]{figures/per_alpha_grids/alpha_0.1/grid_standard.png}\hfill
        \includegraphics[width=0.32\textwidth]{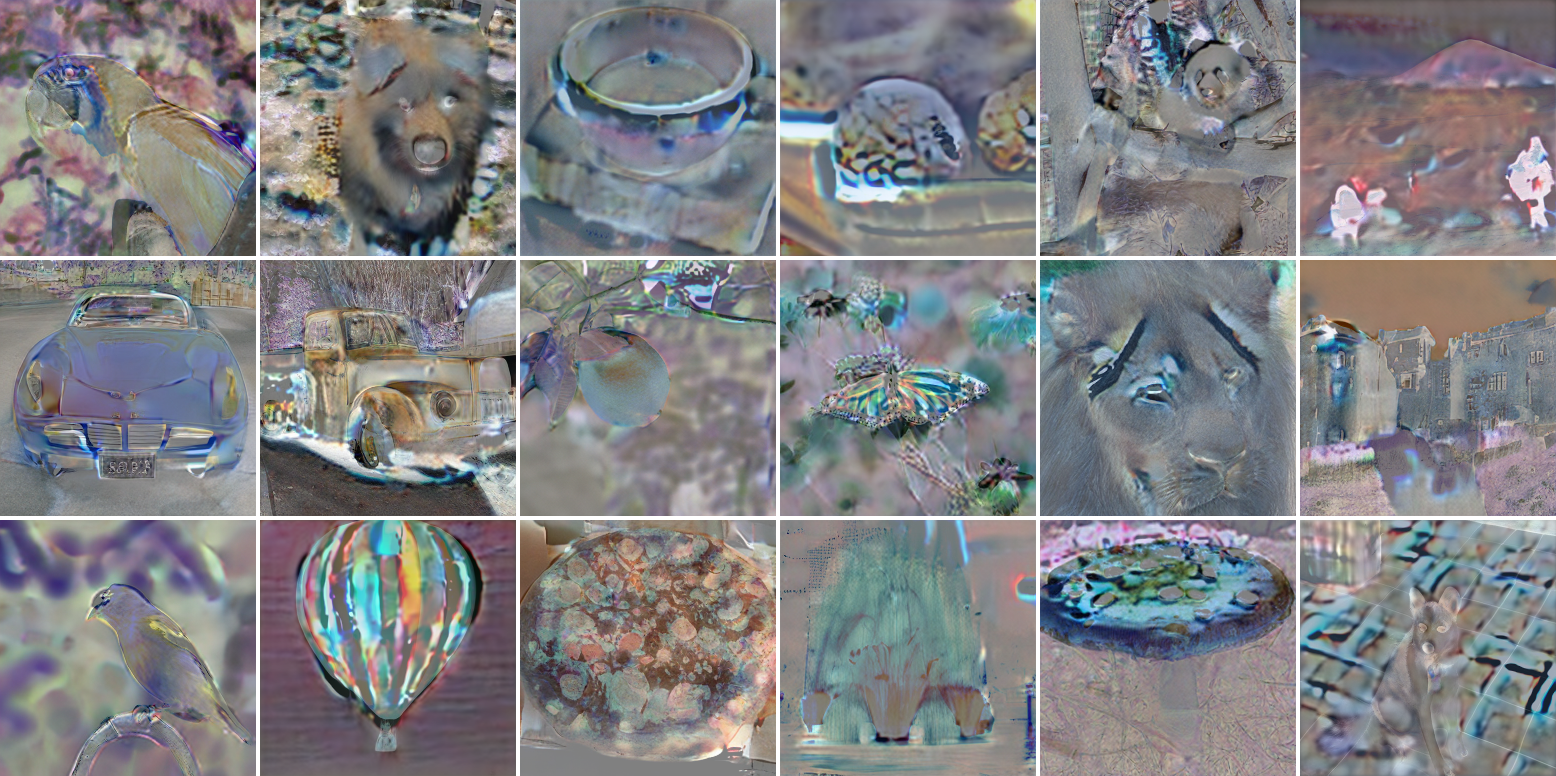}\hfill
        \includegraphics[width=0.32\textwidth]{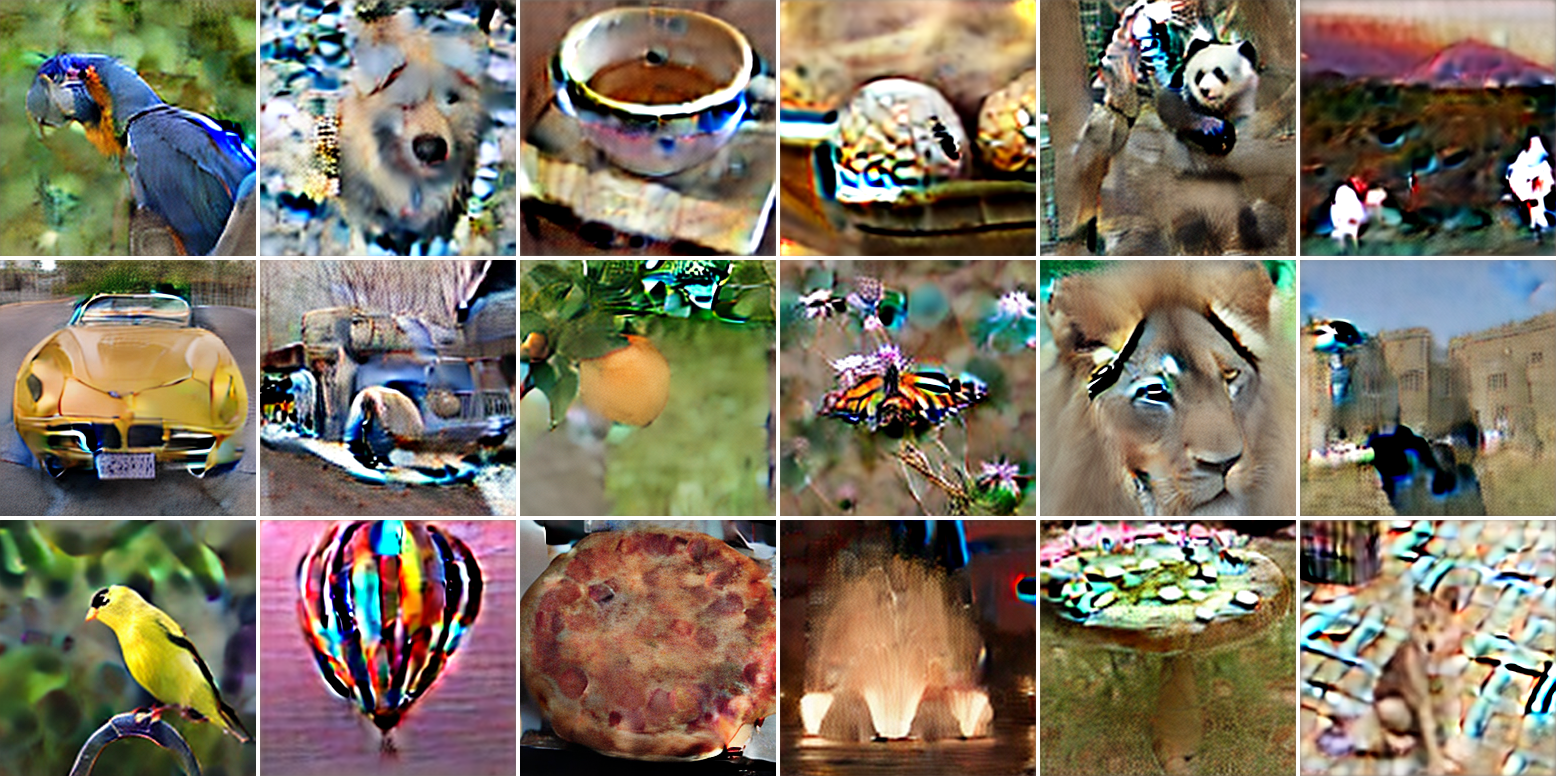}
        \caption*{$\alpha = 0.4$}
    \end{subfigure}

    \vspace{0.3em}

    \begin{subfigure}[b]{\textwidth}
        \centering
        \includegraphics[width=0.32\textwidth]{figures/per_alpha_grids/alpha_0.1/grid_standard.png}\hfill
        \includegraphics[width=0.32\textwidth]{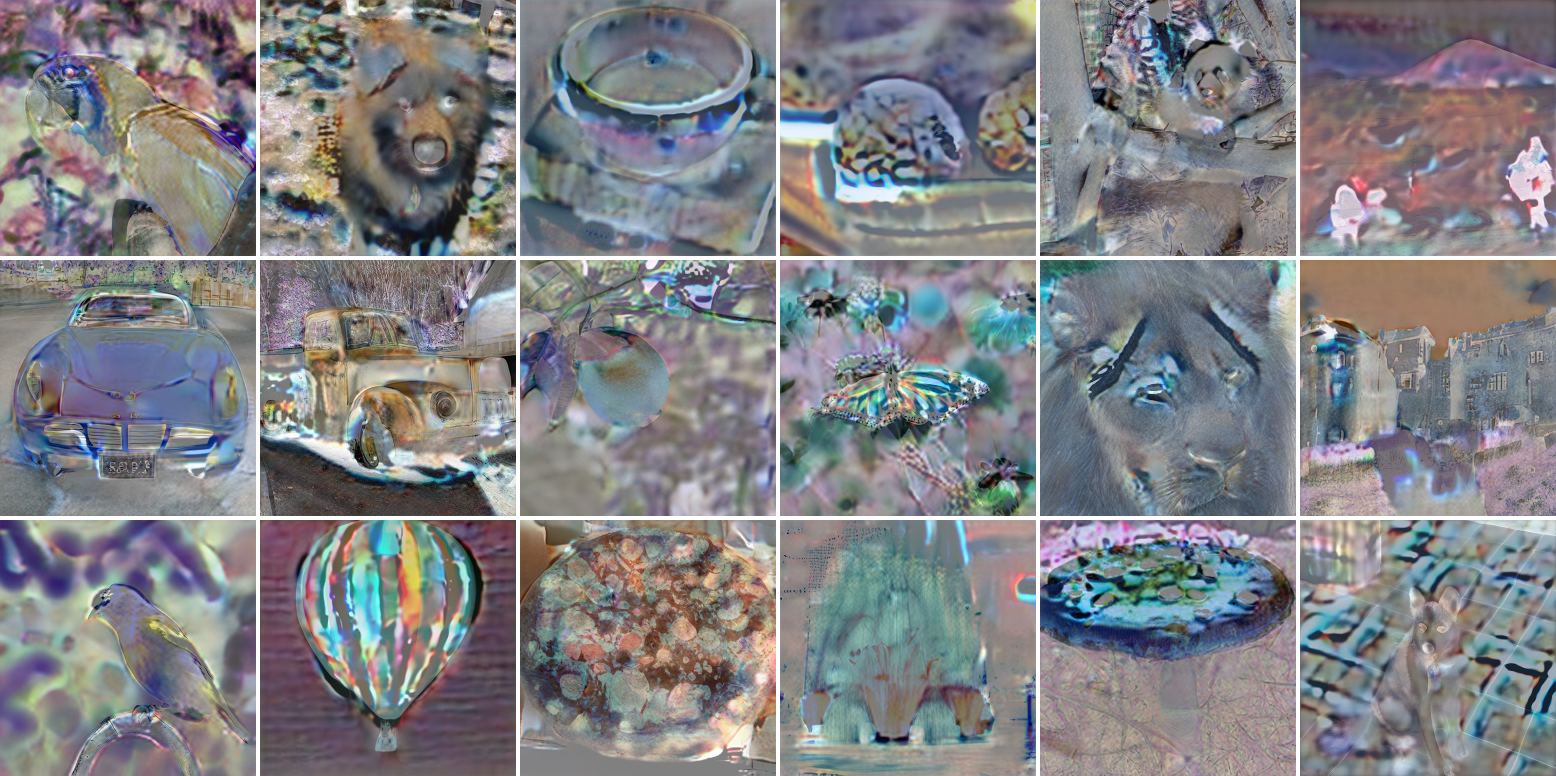}\hfill
        \includegraphics[width=0.32\textwidth]{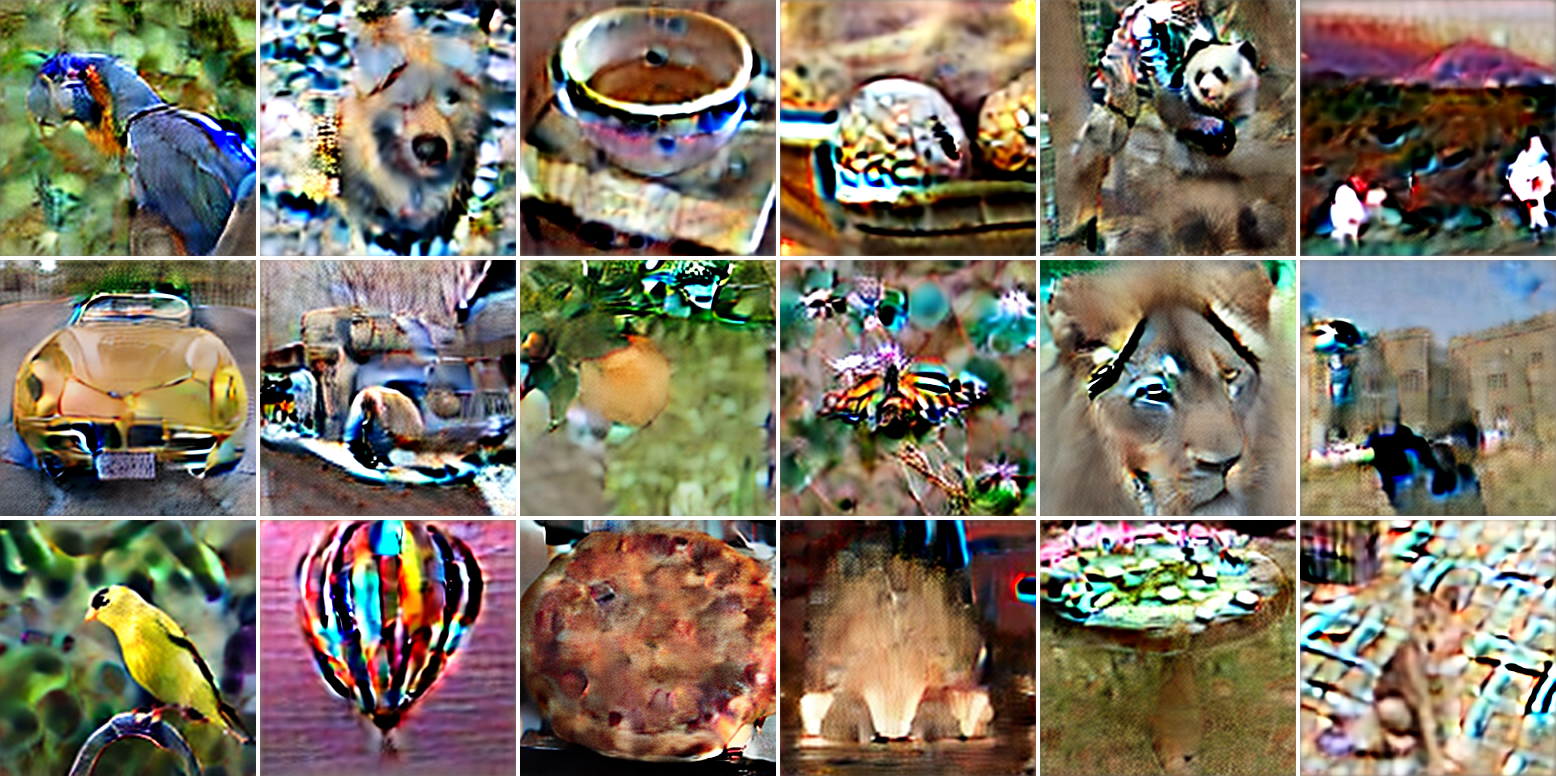}
        \caption*{$\alpha = 0.5$}
    \end{subfigure}

    \caption{
        Additional examples of those shown in Figure~\ref{fig:model_outputs}
        at different GMO~$\alpha$ levels.
    }
    \label{fig:model_outputs_more}
\end{figure}

% --- PART 1 (First 3 Rows) ---
\begin{figure}[p] % [p] puts it on its own dedicated page
    \centering
    \begin{tabular}{@{}cc@{}}
    \textbf{\large Standard Outputs} & \textbf{\large GMO Outputs} \\[0.5em]
    
    \includegraphics[width=0.46\textwidth]{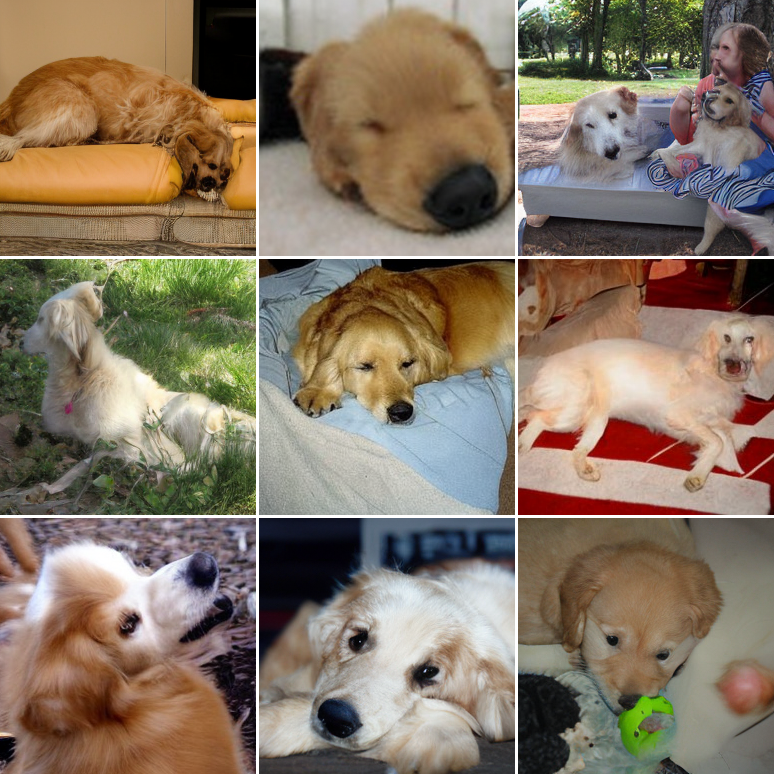} &
    \includegraphics[width=0.46\textwidth]{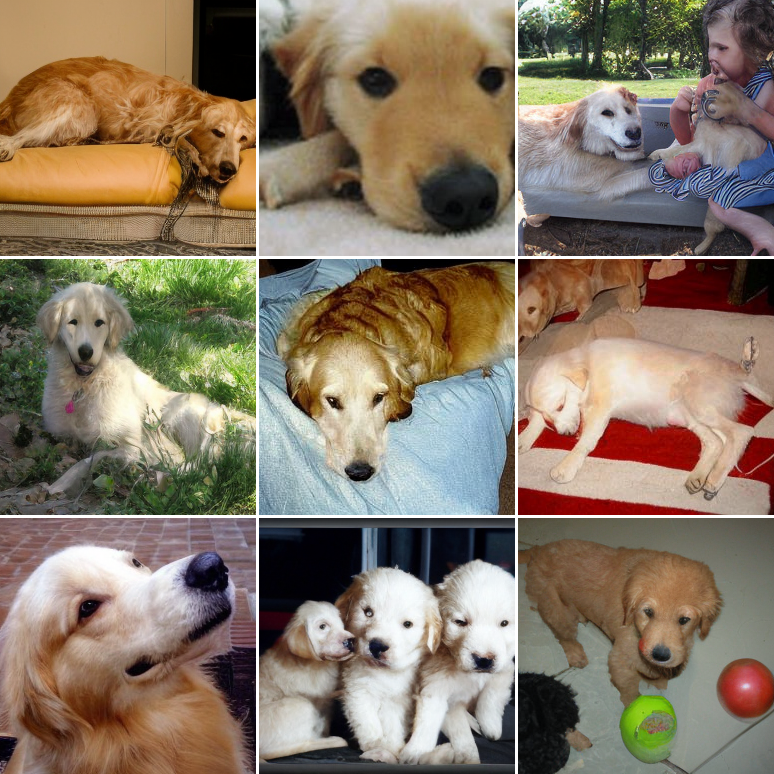} \\[-0.2em]
    \multicolumn{2}{c}{\small Golden Retriever} \\[1.0em]

    \includegraphics[width=0.46\textwidth]{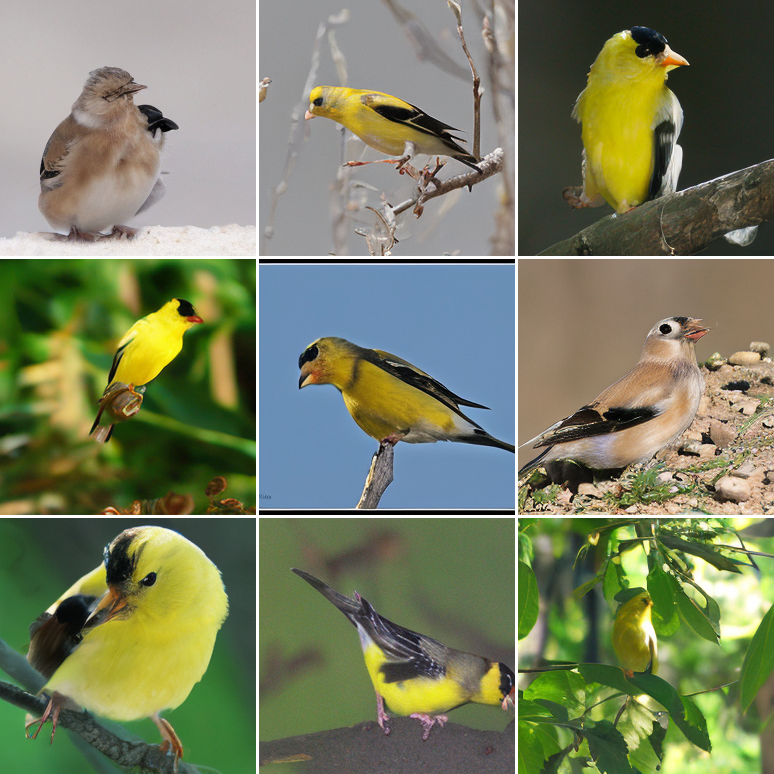} &
    \includegraphics[width=0.46\textwidth]{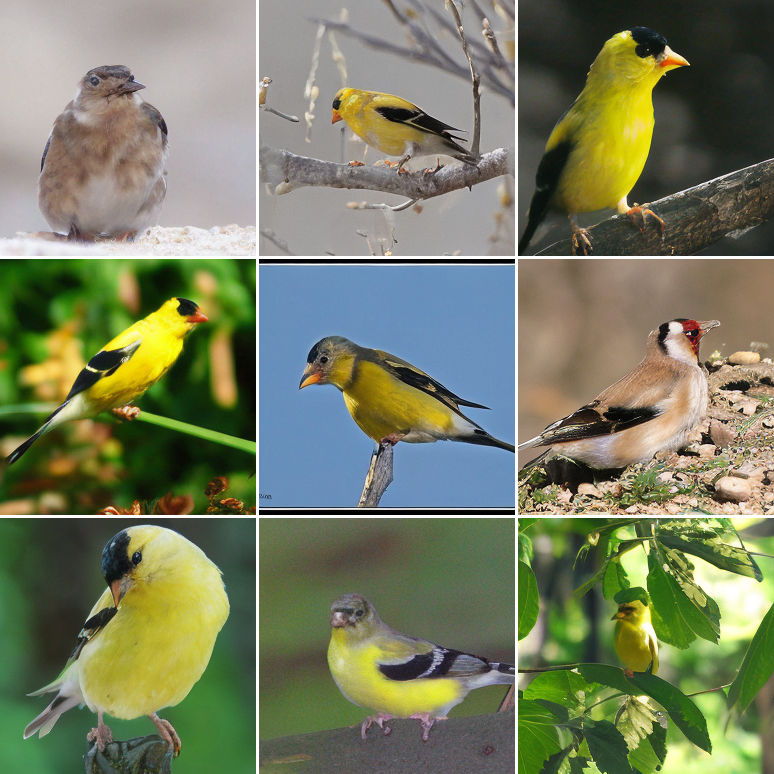} \\[-0.2em]
    \multicolumn{2}{c}{\small Goldfinch} \\[1.0em]

    \includegraphics[width=0.46\textwidth]{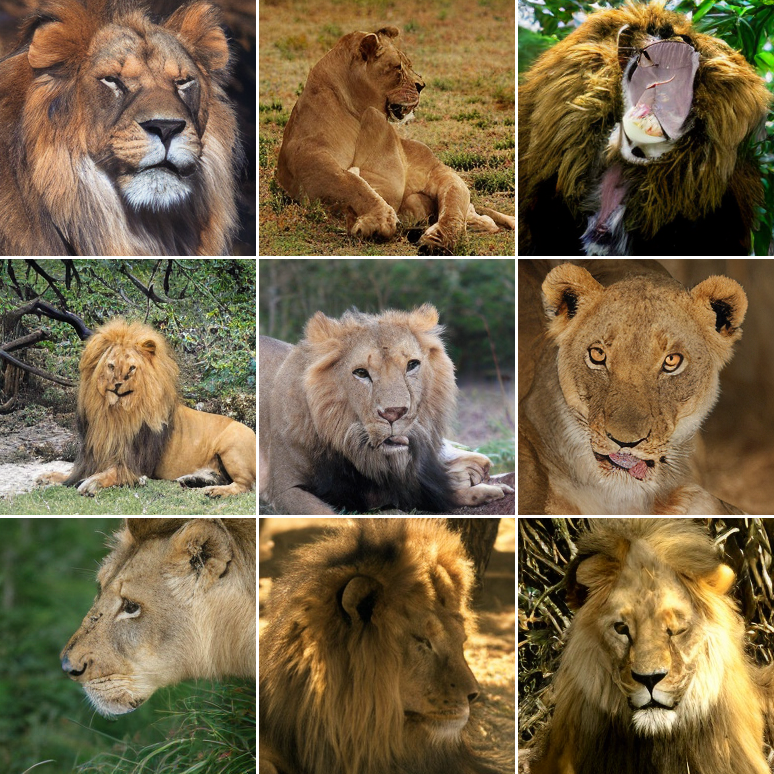} &
    \includegraphics[width=0.46\textwidth]{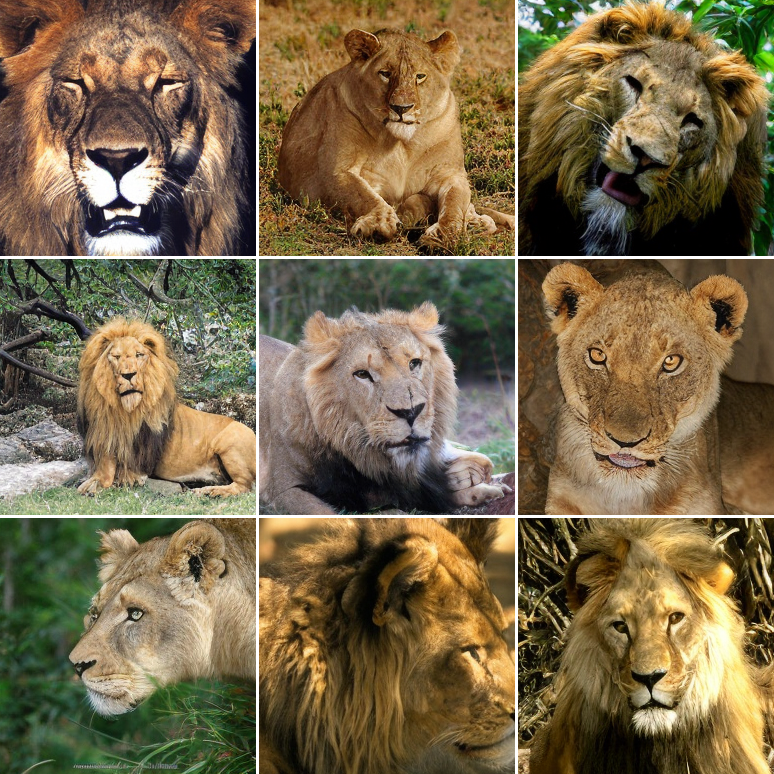} \\[-0.2em]
    \multicolumn{2}{c}{\small Lion}
    \end{tabular}
    \caption{Qualitative comparison on IMM ImageNet-256 (Part I).}
\end{figure}

\clearpage % Forces the next figure onto the next page

% --- PART 2 (Remaining 2 Rows) ---
\begin{figure}[t!]
    \centering
    \addtocounter{figure}{-1} % Keeps the figure number as 12
    \begin{tabular}{@{}cc@{}}
    \textbf{\large Standard Outputs} & \textbf{\large GMO Outputs} \\[0.5em]
    
    \includegraphics[width=0.48\textwidth]{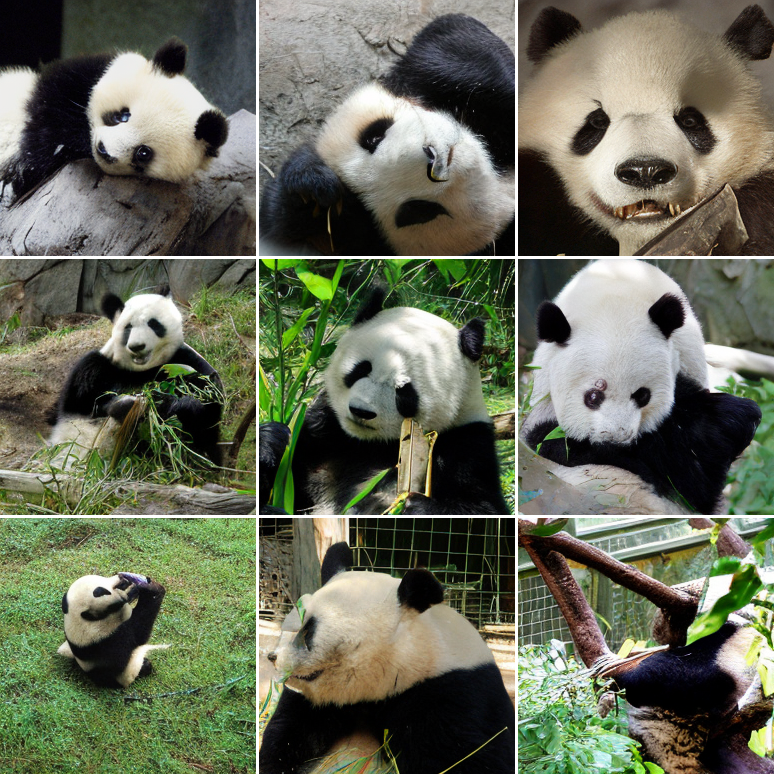} &
    \includegraphics[width=0.48\textwidth]{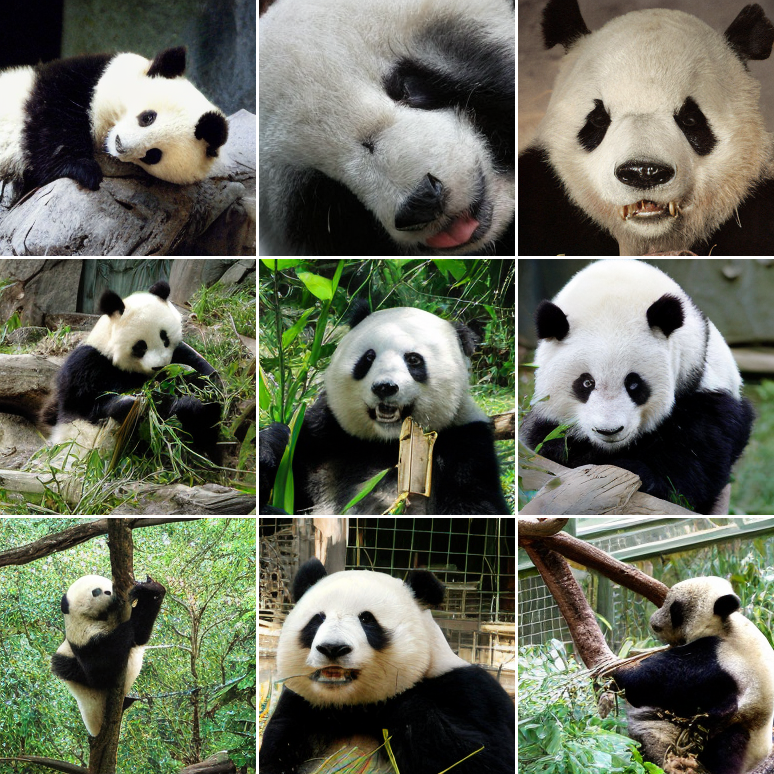} \\[-0.2em]
    \multicolumn{2}{c}{\small Giant Panda} \\[1.0em]

    \includegraphics[width=0.48\textwidth]{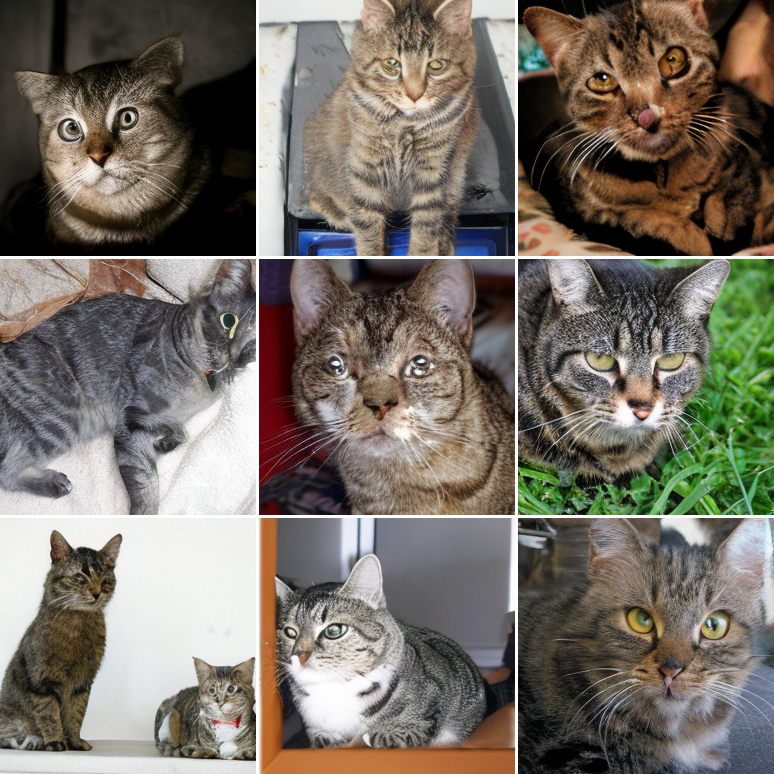} &
    \includegraphics[width=0.48\textwidth]{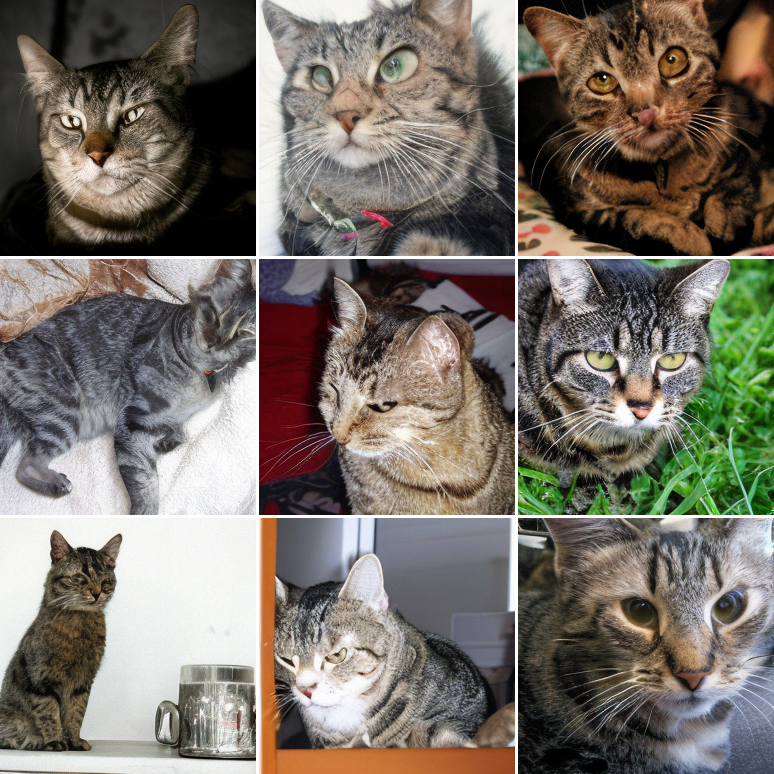} \\[-0.2em]
    \multicolumn{2}{c}{\small Tabby Cat}
    \end{tabular}
    \caption{Qualitative comparison on IMM ImageNet-256 (Part II).}
\end{figure}

% \begin{figure}[ht]
%   \centering
%   \begin{subfigure}[b]{0.32\textwidth}
%     \centering
%     \includegraphics[width=\textwidth]{figures/grid_standard_seed1.png}
%   \end{subfigure}
%   \hfill
%   \begin{subfigure}[b]{0.32\textwidth}
%     \centering
%     \includegraphics[width=\textwidth]{figures/grid_perturbation_a0.1_seed1.png}
%   \end{subfigure}
%   \hfill
%   \begin{subfigure}[b]{0.32\textwidth}
%     \centering
%     \includegraphics[width=\textwidth]{figures/grid_gmo_a0.1_seed1.png}
%   \end{subfigure}
  
%   \begin{subfigure}[b]{0.32\textwidth}
%     \centering
%     \includegraphics[width=\textwidth]{figures/grid_standard_seed2.png}
%     \caption*{Standard Outputs}
%   \end{subfigure}
%   \hfill
%   \begin{subfigure}[b]{0.32\textwidth}
%     \centering
%     \includegraphics[width=\textwidth]{figures/grid_perturbation_a0.1_seed2.png}
%     \caption*{Perturbation}
%   \end{subfigure}
%   \hfill
%   \begin{subfigure}[b]{0.32\textwidth}
%     \centering
%     \includegraphics[width=\textwidth]{figures/grid_gmo_a0.1_seed2.png}
%     \caption*{Geometrically Modified Outputs}
%   \end{subfigure}
%   \caption{
%   Additional examples of those shown in \Cref{fig:model_outputs}.
%   }
%   \label{fig:model_outputs_more}
% \end{figure}

\begin{figure}[ht]
    \centering
    \begin{subfigure}[b]{0.18\textwidth}
        \centering
        \includegraphics[width=\textwidth]{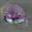}
    \end{subfigure}
    \hfill
    \begin{subfigure}[b]{0.18\textwidth}
        \centering
        \includegraphics[width=\textwidth]{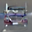}
    \end{subfigure}
    \hfill
    \begin{subfigure}[b]{0.18\textwidth}
        \centering
        \caption*{\small Exact}
        \includegraphics[width=\textwidth]{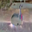}
    \end{subfigure}
    \hfill
    \begin{subfigure}[b]{0.18\textwidth}
        \centering
        \includegraphics[width=\textwidth]{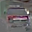}
    \end{subfigure}
    \hfill
    \begin{subfigure}[b]{0.18\textwidth}
        \centering
        \includegraphics[width=\textwidth]{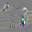}
    \end{subfigure}

    \begin{subfigure}[b]{0.18\textwidth}
        \centering
        \includegraphics[width=\textwidth]{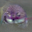}
    \end{subfigure}
    \hfill
    \begin{subfigure}[b]{0.18\textwidth}
        \centering
        \includegraphics[width=\textwidth]{figures/exact_u1_latent1.png}
    \end{subfigure}
    \hfill
    \begin{subfigure}[b]{0.18\textwidth}
        \centering
        \caption*{\small Approximate}
        \includegraphics[width=\textwidth]{figures/exact_u1_latent2.png}
    \end{subfigure}
    \hfill
    \begin{subfigure}[b]{0.18\textwidth}
        \centering
        \includegraphics[width=\textwidth]{figures/exact_u1_latent3.png}
    \end{subfigure}
    \hfill
    \begin{subfigure}[b]{0.18\textwidth}
        \centering
        \includegraphics[width=\textwidth]{figures/exact_u1_latent4.png}
    \end{subfigure}
    \caption{
    In \Cref{fig:computational_approximation}, we quantitatively demonstrate that power iteration can be used to accurately approximate the singular vectors of the input-output Jacobians of generative models.
    Here, we qualitatively support this by showing the exact (top) and approximate top left singular vectors of these input-output Jacobians.
    }
    \label{fig:approximate_singular_vectors}
\end{figure}

\begin{figure}[ht]
    \centering
    \begin{subfigure}[b]{0.18\textwidth}
        \centering
        \includegraphics[width=\textwidth]{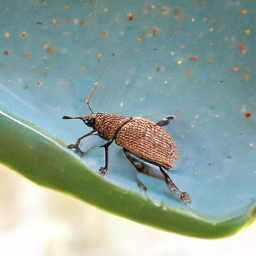}
    \end{subfigure}
    \hfill
    \begin{subfigure}[b]{0.18\textwidth}
        \centering
        \includegraphics[width=\textwidth]{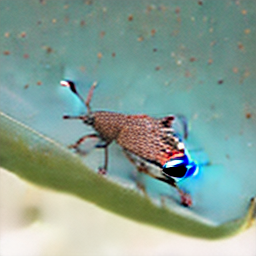}
    \end{subfigure}
    \hfill
    \begin{subfigure}[b]{0.18\textwidth}
        \centering
        \includegraphics[width=\textwidth]{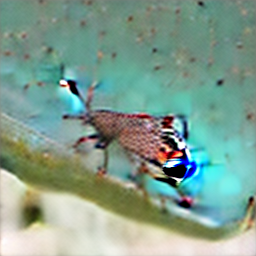}
    \end{subfigure}
    \hfill
    \begin{subfigure}[b]{0.18\textwidth}
        \centering
        \includegraphics[width=\textwidth]{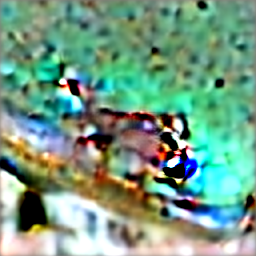}
    \end{subfigure}
    \hfill
    \begin{subfigure}[b]{0.18\textwidth}
        \centering
        \includegraphics[width=\textwidth]{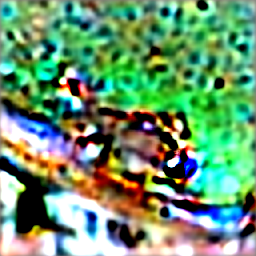}
    \end{subfigure}

    \begin{subfigure}[b]{0.18\textwidth}
        \centering
        \includegraphics[width=\textwidth]{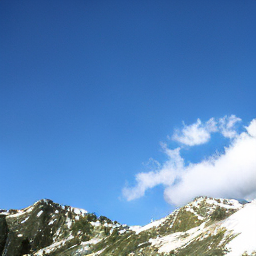}
        \caption*{$\alpha=0.0$}
    \end{subfigure}
    \hfill
    \begin{subfigure}[b]{0.18\textwidth}
        \centering
        \includegraphics[width=\textwidth]{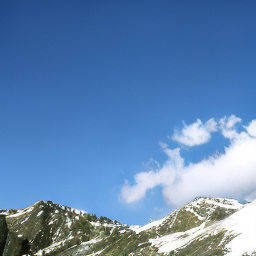}
        \caption*{$\alpha=0.1$}
    \end{subfigure}
    \hfill
    \begin{subfigure}[b]{0.18\textwidth}
        \centering
        \includegraphics[width=\textwidth]{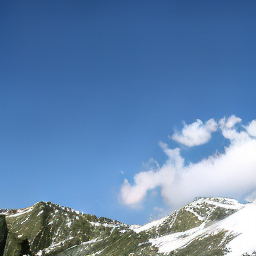}
        \caption*{$\alpha=0.2$}
    \end{subfigure}
    \hfill
    \begin{subfigure}[b]{0.18\textwidth}
        \centering
        \includegraphics[width=\textwidth]{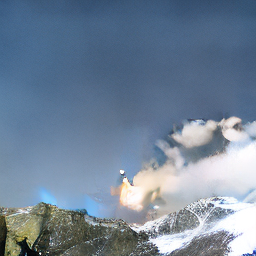}
        \caption*{$\alpha=0.6$}
    \end{subfigure}
    \hfill
    \begin{subfigure}[b]{0.18\textwidth}
        \centering
        \includegraphics[width=\textwidth]{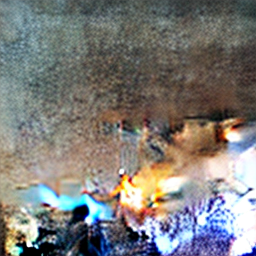}
        \caption*{$\alpha=1.0$}
    \end{subfigure}
    \caption{
    Here we provide additional examples to that illustrated in the bottom row of \Cref{fig:mode_seeking}.
    }
    \label{fig:mode_seeking2}
\end{figure}

\begin{figure}[ht]
    \centering
    \begin{subfigure}[b]{0.32\textwidth}
        \centering
        \includegraphics[width=\textwidth]{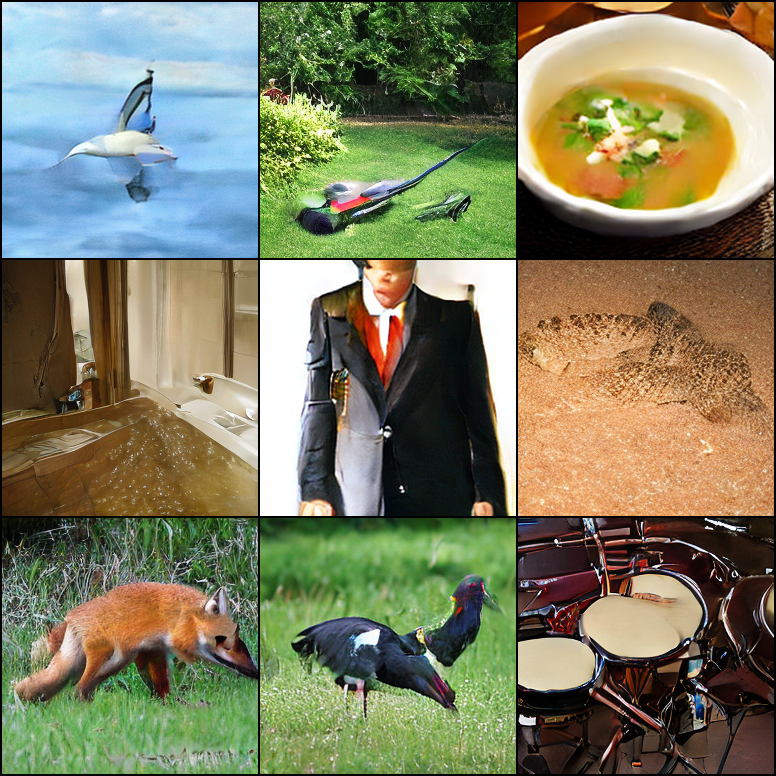}
        \caption*{$\alpha=0.0$}
    \end{subfigure}
    \hfill
    \begin{subfigure}[b]{0.32\textwidth}
        \centering
        \includegraphics[width=\textwidth]{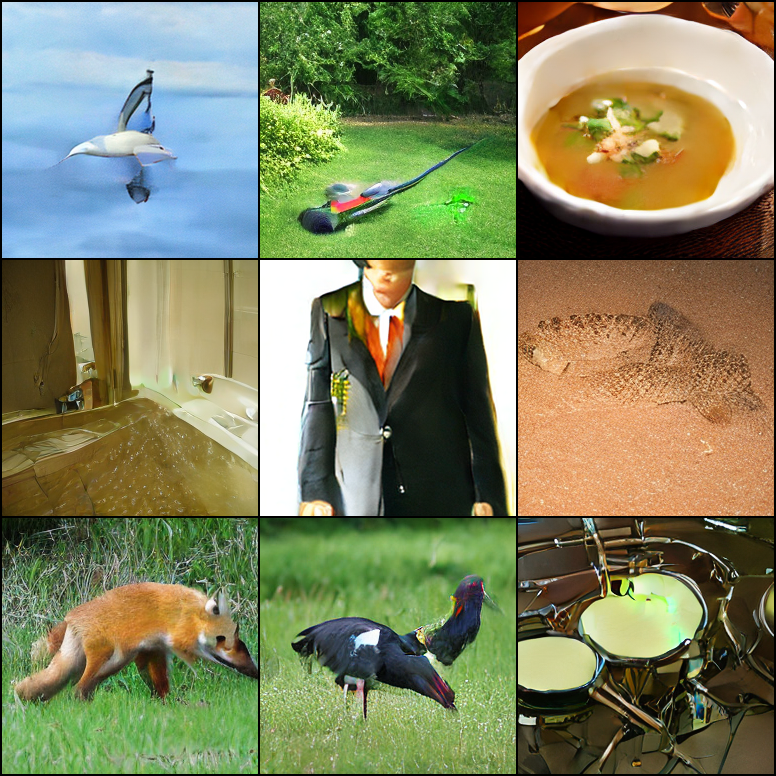}
        \caption*{$\alpha=0.05$}
    \end{subfigure}
    \hfill
    \begin{subfigure}[b]{0.32\textwidth}
        \centering
        \includegraphics[width=\textwidth]{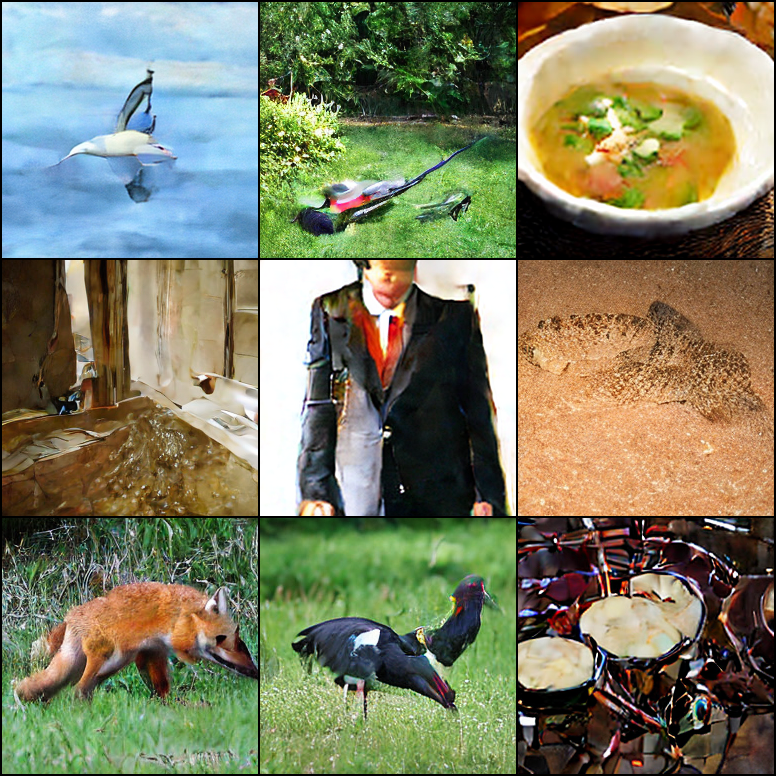}
        \caption*{$\gamma=0.05$}
    \end{subfigure}
    \caption{
    Here we provide additional examples to complement the third panel of \Cref{fig:b2_imagenet256}.
    }
    \label{fig:gaussian_examples}
\end{figure}

% NeurIPS checklist omitted in the ICLR conversion.

\end{document}